\documentclass[11pt]{article} 
\usepackage{etoolbox}
\newtoggle{coltformat}
\togglefalse{coltformat}
\newcommand{\colt}[1]{\iftoggle{coltformat}{#1}{}}
\newcommand{\arxiv}[1]{\iftoggle{coltformat}{}{#1}}

\arxiv{%
 \usepackage[letterpaper, left=1in, right=1in, top=1in, bottom=1in]{geometry}

\usepackage[colorlinks=true, linkcolor=blue!70!black, citecolor=blue!70!black]{hyperref}
\usepackage{microtype}

\usepackage{natbib}
\bibliographystyle{plainnat}
\bibpunct{(}{)}{;}{a}{,}{,}

\usepackage{amsthm}
\usepackage{mathtools}
\usepackage{amsmath}
\usepackage{bbm}
\usepackage{amsfonts}
\usepackage{amssymb}

\usepackage{xpatch}
\usepackage{pifont}
\usepackage{array}
\usepackage{booktabs}
\usepackage{floatrow}
\newfloatcommand{capbtabbox}{table}[][\FBwidth]
\usepackage{blindtext}
\usepackage{caption}

\theoremstyle{definition}  %
 \newtheorem{observation}{Observation}

\theoremstyle{theorem}
\newtheorem{lemma}{Lemma}
\newtheorem{theorem}{Theorem}

\newtheorem{definition}{Definition}

\xpatchcmd{\proof}{\itshape}{\normalfont\proofnameformat}{}{}
\newcommand{\proofnameformat}{\bfseries}

\usepackage{prettyref}
\newcommand{\pref}[1]{\prettyref{#1}}
\newcommand{\pfref}[1]{Proof of \prettyref{#1}}
\newcommand{\savehyperref}[2]{\texorpdfstring{\hyperref[#1]{#2}}{#2}}
\newrefformat{eq}{\savehyperref{#1}{\textup{(\ref*{#1})}}}
\newrefformat{ineq}{\savehyperref{#1}{\textup{(\ref*{#1})}}}
\newrefformat{eqn}{\savehyperref{#1}{Equation~\ref*{#1}}}
\newrefformat{obs}{\savehyperref{#1}{Observation~\ref*{#1}}}
\newrefformat{lem}{\savehyperref{#1}{Lemma~\ref*{#1}}}
\newrefformat{def}{\savehyperref{#1}{Definition~\ref*{#1}}}
\newrefformat{thm}{\savehyperref{#1}{Theorem~\ref*{#1}}}
\newrefformat{fact}{\savehyperref{#1}{Fact~\ref*{#1}}}
\newrefformat{tab}{\savehyperref{#1}{Table~\ref*{#1}}}
\newrefformat{table}{\savehyperref{#1}{Table~\ref*{#1}}}
\newrefformat{line}{\savehyperref{#1}{Line~\ref*{#1}}}
\newrefformat{corr}{\savehyperref{#1}{Corollary~\ref*{#1}}}
\newrefformat{cor}{\savehyperref{#1}{Corollary~\ref*{#1}}}
\newrefformat{sec}{\savehyperref{#1}{Section~\ref*{#1}}}
\newrefformat{app}{\savehyperref{#1}{Appendix~\ref*{#1}}}
\newrefformat{ass}{\savehyperref{#1}{Assumption~\ref*{#1}}}
\newrefformat{ex}{\savehyperref{#1}{Example~\ref*{#1}}}
\newrefformat{fig}{\savehyperref{#1}{Figure~\ref*{#1}}}
\newrefformat{alg}{\savehyperref{#1}{Algorithm~\ref*{#1}}}
\newrefformat{rem}{\savehyperref{#1}{Remark~\ref*{#1}}}
\newrefformat{conj}{\savehyperref{#1}{Conjecture~\ref*{#1}}}
\newrefformat{prop}{\savehyperref{#1}{Proposition~\ref*{#1}}}
\newrefformat{proto}{\savehyperref{#1}{Protocol~\ref*{#1}}}
\newrefformat{prob}{\savehyperref{#1}{Problem~\ref*{#1}}}
\newrefformat{claim}{\savehyperref{#1}{Claim~\ref*{#1}}}
\newrefformat{tbl}{\savehyperref{#1}{Table~\ref*{#1}}} %
}
\colt{%
\usepackage{microtype}

\usepackage{mathtools}
\usepackage{amsmath}
\usepackage{bbm}
\usepackage{amsfonts}

\usepackage{amssymb}

\usepackage{xpatch}

\usepackage{pifont}

\usepackage{array}
\usepackage{booktabs}

\usepackage{floatrow}
\newfloatcommand{capbtabbox}{table}[][\FBwidth]

\usepackage{blindtext}
\usepackage{caption}

\newtheorem{observation}{Observation}

\xpatchcmd{\proof}{\itshape}{\normalfont\proofnameformat}{}{}
\newcommand{\proofnameformat}{\bfseries}

\usepackage{prettyref}
\newcommand{\pref}[1]{\prettyref{#1}}
\newcommand{\pfref}[1]{Proof of \prettyref{#1}}
\newcommand{\savehyperref}[2]{\texorpdfstring{\hyperref[#1]{#2}}{#2}}
\newrefformat{eq}{\savehyperref{#1}{\textup{(\ref*{#1})}}}
\newrefformat{ineq}{\savehyperref{#1}{\textup{(\ref*{#1})}}}
\newrefformat{eqn}{\savehyperref{#1}{Equation~\ref*{#1}}}
\newrefformat{obs}{\savehyperref{#1}{Observation~\ref*{#1}}}
\newrefformat{lem}{\savehyperref{#1}{Lemma~\ref*{#1}}}
\newrefformat{def}{\savehyperref{#1}{Definition~\ref*{#1}}}
\newrefformat{thm}{\savehyperref{#1}{Theorem~\ref*{#1}}}
\newrefformat{fact}{\savehyperref{#1}{Fact~\ref*{#1}}}
\newrefformat{tab}{\savehyperref{#1}{Table~\ref*{#1}}}
\newrefformat{table}{\savehyperref{#1}{Table~\ref*{#1}}}
\newrefformat{line}{\savehyperref{#1}{Line~\ref*{#1}}}
\newrefformat{corr}{\savehyperref{#1}{Corollary~\ref*{#1}}}
\newrefformat{cor}{\savehyperref{#1}{Corollary~\ref*{#1}}}
\newrefformat{sec}{\savehyperref{#1}{Section~\ref*{#1}}}
\newrefformat{app}{\savehyperref{#1}{Appendix~\ref*{#1}}}
\newrefformat{ass}{\savehyperref{#1}{Assumption~\ref*{#1}}}
\newrefformat{ex}{\savehyperref{#1}{Example~\ref*{#1}}}
\newrefformat{fig}{\savehyperref{#1}{Figure~\ref*{#1}}}
\newrefformat{alg}{\savehyperref{#1}{Algorithm~\ref*{#1}}}
\newrefformat{rem}{\savehyperref{#1}{Remark~\ref*{#1}}}
\newrefformat{conj}{\savehyperref{#1}{Conjecture~\ref*{#1}}}
\newrefformat{prop}{\savehyperref{#1}{Proposition~\ref*{#1}}}
\newrefformat{proto}{\savehyperref{#1}{Protocol~\ref*{#1}}}
\newrefformat{prob}{\savehyperref{#1}{Problem~\ref*{#1}}}
\newrefformat{claim}{\savehyperref{#1}{Claim~\ref*{#1}}}
\newrefformat{tbl}{\savehyperref{#1}{Table~\ref*{#1}}}

}

\newcommand\numberthis{\addtocounter{equation}{1}\tag{\theequation}}
\allowdisplaybreaks

\DeclarePairedDelimiter{\abs}{\lvert}{\rvert} 
\DeclarePairedDelimiter{\brk}{[}{]}
\DeclarePairedDelimiter{\crl}{\{}{\}}
\DeclarePairedDelimiter{\prn}{(}{)}
\DeclarePairedDelimiter{\nrm}{\|}{\|}
\DeclarePairedDelimiter{\norm}{\|}{\|}
\DeclarePairedDelimiter{\tri}{\langle}{\rangle}

\DeclarePairedDelimiter{\ceil}{\lceil}{\rceil}
\DeclarePairedDelimiter{\floor}{\lfloor}{\rfloor}

\newcommand{\lmid}{~\Big|~}

\let\Pr\undefined

\DeclareMathOperator{\En}{\mathbb{E}}

\DeclareMathOperator{\Pr}{Pr}

\DeclareMathOperator*{\argmin}{arg\,min} %

\newcommand{\ls}{\ell}

\newcommand{\ind}[1]{\mathbbm{1}\crl*{#1}}    %

\newcommand{\eps}{\epsilon}
\newcommand{\veps}{\varepsilon}

\newcommand{\ldef}{\vcentcolon=}

\newcommand{\wt}[1]{\widetilde{#1}}
\newcommand{\wh}[1]{\widehat{#1}}
\newcommand{\mb}[1]{\boldsymbol{#1}}
\newcommand{\proman}[1]{\prn*{\romannumeral #1}}

\def\ddefloop#1{\ifx\ddefloop#1\else\ddef{#1}\expandafter\ddefloop\fi}
\def\ddef#1{\expandafter\def\csname bb#1\endcsname{\ensuremath{\mathbb{#1}}}}
\ddefloop ABCDEFGHIJKLMNOPQRSTUVWXYZ\ddefloop
\def\ddefloop#1{\ifx\ddefloop#1\else\ddef{#1}\expandafter\ddefloop\fi}
\def\ddef#1{\expandafter\def\csname b#1\endcsname{\ensuremath{\mathbf{#1}}}}
\ddefloop ABCDEFGHIJKLMNOPQRSTUVWXYZ\ddefloop
\def\ddef#1{\expandafter\def\csname c#1\endcsname{\ensuremath{\mathcal{#1}}}}
\ddefloop ABCDEFGHIJKLMNOPQRSTUVWXYZ\ddefloop
\def\ddef#1{\expandafter\def\csname h#1\endcsname{\ensuremath{\widehat{#1}}}}
\ddefloop ABCDEFGHIJKLMNOPQRSTUVWXYZabcdefghijklmnopqrsuvwxyz\ddefloop    %
\def\ddef#1{\expandafter\def\csname hc#1\endcsname{\ensuremath{\widehat{\mathcal{#1}}}}}
\ddefloop ABCDEFGHIJKLMNOPQRSTUVWXYZ\ddefloop
\def\ddef#1{\expandafter\def\csname t#1\endcsname{\ensuremath{\widetilde{#1}}}}
\ddefloop ABCDEFGHIJKLMNOPQRSTUVWXYZ\ddefloop
\def\ddef#1{\expandafter\def\csname tc#1\endcsname{\ensuremath{\widetilde{\mathcal{#1}}}}}
\ddefloop ABCDEFGHIJKLMNOPQRSTUVWXYZ\ddefloop

\newcommand{\ball}{\mathbb{B}}

\DeclareSymbolFont{extraup}{U}{zavm}{m}{n}
\DeclareMathSymbol{\varheart}{\mathalpha}{extraup}{86}
\DeclareMathSymbol{\vardiamond}{\mathalpha}{extraup}{87}
\newcommand{\grad}{\nabla}
\newcommand{\hess}{\nabla^2}

\newcommand{\overleq}[1]{\overset{ #1}{\leq{}}}
\newcommand{\overgeq}[1]{\overset{#1}{\geq{}}}
\newcommand{\overeq}[1]{\overset{#1}{=}}

\usepackage{tikz}

 \newcommand{\markedterm}[1]{(\text{\bf #1})}      
\newcommand{\xhat}{\wh{x}}

\newcommand{\hessn}[1]{m_{#1}}

\newcommand{\bern}{\mathrm{Bernoulli}}

\newcommand{\numbabysteps}{K}
\newcommand{\bias}{b}
\newcommand{\newestparams}{{\epsilon}} 
\newcommand{\Q}{Q}
\newcommand{\coin}{C} 
\newcommand{\derF}[2]{\grad^{#2}F (#1)}
\newcommand{\ojaF}{\mathsf{Oja}}
\newcommand{\oja}{\ojaF}
\newcommand{\Eoja}{\mathsf{E}^{\ojaF}}

\newcommand{\del}{\partial}
\newcommand{\barzDelta}{\tilde{\Delta}_0}
\newcommand{\Phibar}{\Lambda}
\newcommand{\barF}{G_T}
\newcommand{\barFunscaled}{G_T}
\newcommand{\barFscaled}{G^\star_T}
\newcommand{\lambdamin}{\lambda_{\min}}
\newcommand{\bt}{{(t)}} 
\newcommand{\zi}{z^{(i)}}
\newcommand{\progf}[2]{\prog_{\frac{#1}{#2}}}

\newcommand{\pchain}{\rho}
\newcommand{\stocOhigh}{\mathsf{O}^{p}}

\newcommand{\HOSF}{\cF_{p} \prn*{\Delta, \Lipp}} %

\newcommand{\pest}[1]{\widehat{\grad^{#1} F}}
\newcommand{\pestunscaled}[1]{\widehat{\grad^{#1} F_T}}
\newcommand{\pestscaled}[1]{\widehat{\grad^{#1} F_T^\star}}

\newcommand{\pestunscaledBar}[1]{\widehat{\grad^{#1} G_T}}
\newcommand{\pestscaledBar}[1]{\widehat{\grad^{#1} 
G_T^\star}}

\newcommand{\Fclasstilp}{\cF_p(\Delta, \Lip{1:p})}

\newcommand{\stocpOracle}[1][p]{\cO_{#1}\prn{F,\sigma_{1:#1} }}

\newcommand{\oracle}{\mathsf{O}}
\newcommand{\estimator}[1][p]{\oracle^{\,#1}_F}

\newcommand{\stocder}[2]{\widehat{\grad^{#2} F}(#1)}
\newcommand{\stocgrad}[1]{\widehat{\grad F}\prn*{#1}}
\newcommand{\stochess}[1]{\widehat{\hess F}\prn*{#1}}
\newcommand{\stocgradNA}{\widehat{\grad F}}
\newcommand{\stochessNA}{\widehat{\hess F}}

\newcommand{\xprev}{x_{\mathrm{prev}}}

\newcommand{\func}[1]{F\prn{#1}}
\newcommand{\gradestc}[1]{\gradest[#1]}

\newcommand{\tepsilon}{\tilde{\epsilon}}

\newcommand{\stocpOracleas}{\overline{\cO}_2(F, \sigma_1, 
\bar{\sigma}_2)}
 
\newcommand{\bsigma}{\bar{\sigma}}

\makeatletter
\newcommand{\hessest}[1][\@nil]{%
  \def\tmp{#1}%
   \ifx\tmp\@nnil
	H 
    \else
     H^{\prn*{#1}}  
    \fi}
\makeatother

\makeatletter
\newcommand{\gradest}[1][\@nil]{%
  \def\tmp{#1}%
   \ifx\tmp\@nnil
	g 
    \else
     g^{\prn*{#1}}  
    \fi}
\makeatother

\newcommand{\outerbatchsize}{n}

\newcommand{\indicator}[1]{\mb{1}\crl*{#1}}
\newcommand{\indic}{\mb{1}} 
\newcommand{\gprev}{g_{\mathrm{prev}}}

\newcommand{\getgradientfunc}{\textsf{HVP-RVR-Gradient-Estimator}}
\newcommand{\mainalg}{\textsf{HVP-RVR}\xspace}

\newcommand{\R}{\mathbb{R}}
\newcommand{\E}{\mathbb{E}}
\newcommand{\N}{\mathbb{N}}

\newcommand{\defeq}{\coloneqq}

\newcommand{\Lip}[1]{L_{#1}}

\newcommand{\LipGradBar}{\bar{L}}

\newcommand{\Lipp}{\Lip{p}}

\newcommand{\lip}[1]{\ell_{#1}}
\newcommand{\lipBar}[1]{\tilde{\ell}_{#1}}

\newcommand{\Pz}{P_z}

\newcommand{\op}{\mathrm{op}}
\newcommand{\opnorm}[1]{\norm{#1}_{\rm op}}

\newcommand{\alg}{\mathsf{A}}

\newcommand{\AlgZR}{\mathcal{A}_{\textnormal{\textsf{zr}}}}

\newcommand{\unscaled}[1]{#1}

\newcommand{\Funscaled}{\unscaled{F}_T}

\newcommand{\Fscaled}{F^{\star}_T}

\newcommand{\Otilde}{\wt{O}}

\newcommand{\eigmin}{\lambda_{\mathrm{min}}}

\renewcommand{\ind}[1]{^{(#1)}}

\DeclareMathOperator*{\support}{\mathrm{support}}

\newcommand{\prog}{\mathrm{prog}}

\newcommand{\iid}{\mathrm{i.i.d.}}
\newcommand{\iidsim}{\overset{\iid}{\sim}}

\newcommand{\symd}{\bbS^{d}}

\newcommand{\tth}{$t^\text{th}$\xspace}
\newcommand{\pth}{$p$th\xspace}

\newcommand{\tensT}{T}

\definecolor{innerboxcolor}{rgb}{.9,.95,1}
\definecolor{outerlinecolor}{rgb}{.6,0,.2}

\usepackage{xcolor}
\arxiv{
 \usepackage[suppress]{color-edits}
\addauthor{as}{red}
\addauthor{df}{green!50!black}
\addauthor{ks}{blue}
\addauthor{yc}{cyan!75!black}
\addauthor{ya}{red!50!black}

}
\colt{
 \usepackage[suppress]{color-edits}
\addauthor{as}{red}
\addauthor{df}{green!50!black}
\addauthor{ks}{blue}
\addauthor{yc}{cyan!75!black}
\addauthor{ya}{red!50!black}

}

\colt{
 
\setcitestyle{numbers,square,comma} %
\usepackage{times}
}

\usepackage{longtable}
\usepackage{import}
\usepackage{enumitem}
\usepackage{algorithm}
\usepackage[noend]{algpseudocode} 

\algblockdefx[class]{Class}{EndClass}[1]{\textbf{object} \textsc{#1}:}{\textbf{end object}}
\makeatletter 
\ifthenelse{\equal{\ALG@noend}{t}}%
  {\algtext*{EndClass}}
  {}%
\makeatother
\usepackage{amsmath}
\usepackage{makecell}
\usepackage{enumitem}
\usepackage{xspace}
\usepackage[scaled=.9]{helvet}
\usepackage{booktabs}
\usepackage[export]{adjustbox}
\arxiv{
\newcommand{\algcomment}[1]{\textcolor{blue!70!black}{\footnotesize{\texttt{\textbf{//
          #1}}}}}

}
\colt{
\newcommand{\algcomment}[1]{\textcolor{blue}{\footnotesize{\texttt{\textbf{//
          #1}}}}}

}

\newcommand{\sigmss}{\sigma_{\mathrm{mss}}}
\renewcommand{\veps}{\eps}
\allowdisplaybreaks

\colt{
\title[Second-Order Information in Non-Convex Stochastic 
Optimization]{Second-Order Information in Non-Convex 
Stochastic Optimization: Power and Limitations} 
}
\arxiv{
\title{Second-Order Information in Non-Convex Stochastic Optimization: Power and Limitations} 
}

\colt{\usepackage{times}}

\arxiv{
\author{ Yossi Arjevani\\
	New York University\\
        {\small\texttt{yossia@nyu.edu}}\\ 
	  \and
	  Yair Carmon\\ 
	  Stanford University\\
	  {\small\texttt{yairc@stanford.edu}}\\
	  \and
	  John C.\ Duchi\\
	  Stanford University\\
	  {\small\texttt{jduchi@stanford.edu}}\\
          \and
	~ Dylan J.\ Foster\\
        MIT\\
        {\small\texttt{dylanf@mit.edu}}\\ 
	  \and
	Ayush Sekhari\\
	Cornell University\\
                {\small\texttt{as3663@cornell.edu}}\\
       \and
        Karthik Sridharan\\
        Cornell University\\
        {\small\texttt{ks999@cornell.edu}}\\
}
}

\colt{
\coltauthor{%
 \Name{Yossi Arjevani}\footnotemark[1] \Email{yossia@nyu.edu}
 \AND
 \Name{Yair Carmon}\footnotemark[2] \Email{yairc@stanford.edu}
 \AND
  \Name{John C.\ Duchi}\footnotemark[2] \Email{jduchi@stanford.edu}
 \AND
  \Name{Dylan J.\ Foster}\footnotemark[3] \Email{dylanf@mit.edu}
 \AND
  \Name{Ayush Sekhari}\footnotemark[4] \Email{as3663@cornell.edu}
 \AND
  \Name{Karthik Sridharan}\footnotemark[4] \Email{ks999@cornell.edu}\\
 \addr \footnotemark[1]New York University, \footnotemark[2]Stanford University,
 \footnotemark[3]Massachusetts Institute of Technology, \footnotemark[4]Cornell University
}
}

\date{} 

\begin{document}

\maketitle
 \begin{abstract}
  We design an algorithm which finds an $\epsilon$-approximate stationary 
  point (with $\|\nabla F(x)\|\le \epsilon$) using $O(\epsilon^{-3})$ 
  stochastic gradient and Hessian-vector products, matching guarantees 
  that were previously available only under a stronger assumption of access 
  to multiple queries with the same random seed. 
 We prove a lower bound which establishes that this rate
is optimal and---surprisingly---that it cannot be improved using stochastic $p$th order methods
for any $p\ge 2$, even when the first $p$ derivatives of the objective 
are Lipschitz. Together, these results characterize the complexity of
non-convex stochastic optimization with second-order methods and beyond.
Expanding our scope to the oracle complexity of finding 
$(\epsilon,\gamma)$-approximate second-order stationary points, we 
establish nearly matching upper and lower bounds for stochastic 
second-order methods. Our lower bounds here are novel even in the noiseless 
case.
\end{abstract}

\colt{
}

\section{Introduction}\label{sec:intro}
Let $F:\R^d\to\R$ have Lipschitz continuous gradient and Hessian, and 
consider the task of finding an $(\epsilon, \gamma)$-second-order 
stationary point (SOSP), that is, $x\in\R^d$ such that
\begin{equation}\label{eq:sosp}
\norm{\grad F(x)} \le \eps
\quad\mbox{and}\quad
\hess F(x) \succeq - \gamma I.
\end{equation}
This task plays a central role in the study of non-convex optimization: for 
functions satisfying a weak strict saddle condition \citep{ge2015escaping}, exact SOSPs 
(with 
$\epsilon=\gamma=0$) are 
local minima, and therefore the condition \eqref{eq:sosp} serves
as a proxy for 
approximate local optimality.\footnote{
	 However, it is NP-Hard to decide whether a SOSP 
	 is a local minimum or a high-order saddle 
	 point~\citep{murty1987some}.}
Moreover, for a growing set of non-convex optimization problems arising 
in 
machine learning, SOSPs are in fact \emph{global minima} \citep{ge2015escaping, ge2016matrix, sun2018geometric, 
mc2019implicit}. 
Consequently, there has been intense recent interest in the design of efficient 
algorithms 
for finding approximate SOSPs~\citep{jin2017how,allen2018how,carmon2018accelerated,fang2018spider,tripuraneni2018stochastic,xu2018first,fang2019sharp}.

In stochastic approximation tasks---particularly those motivated by
machine learning---access 
to the objective function is often restricted to
stochastic estimates of its gradient; for each query point $x\in\R^d$ we observe 
$\stocgrad{x,z}$, where $z\sim P_z$ is a random variable such that
\begin{equation}\label{eq:stoc-grad} 
\E \bigl[\stocgrad{x,z}\bigr] = \grad F(x)
~~\mbox{and}~~
\E \, \norm{\stocgrad{x,z}  - \grad F(x)}^2 \le \sigma_1^2.
\end{equation}
This restriction typically arises due to computational considerations (when  
$\stocgrad{\cdot{}, z}$ is much cheaper to compute than $\grad F(\cdot{})$, as in 
empirical risk minimization or Monte Carlo simulation), or due to
fundamental online nature of the 
 problem at hand (e.g., when $x$ represents a routing scheme 
and $z$ represents traffic on a given day). However, for many problems with additional structure, we have
access to extra information. For example, we
often have access to stochastic second-order 
information in the form of a Hessian estimator $\stochess{x,z}$ satisfying
\begin{equation}\label{eq:stoc-hess}
\E \bigl[\stochess{x,z}\bigr] = \hess F(x)
~~\mbox{and}~~
\E \, \opnorm{\stochess{x,z}  - \hess F(x)}^2 \le \sigma_2^2.
\end{equation} 

In this paper, we characterize the extent to which the stochastic Hessian 
information~\pref{eq:stoc-hess}, as well as higher-order information, contributes to the efficiency of
finding first- and second-order stationary points. We approach
this question from the 
perspective of \emph{oracle complexity}~\citep{nemirovski1983problem}, 
which measures  
efficiency by the number of queries to estimators of the 
form~\pref{eq:stoc-grad}---and 
possibly~\pref{eq:stoc-hess}---required to satisfy the 
condition~\pref{eq:sosp}. 

\subsection{Our Contributions}
\arxiv{\begin{figure}[t]
      \centering
      		\includegraphics[width=0.6\textwidth]{{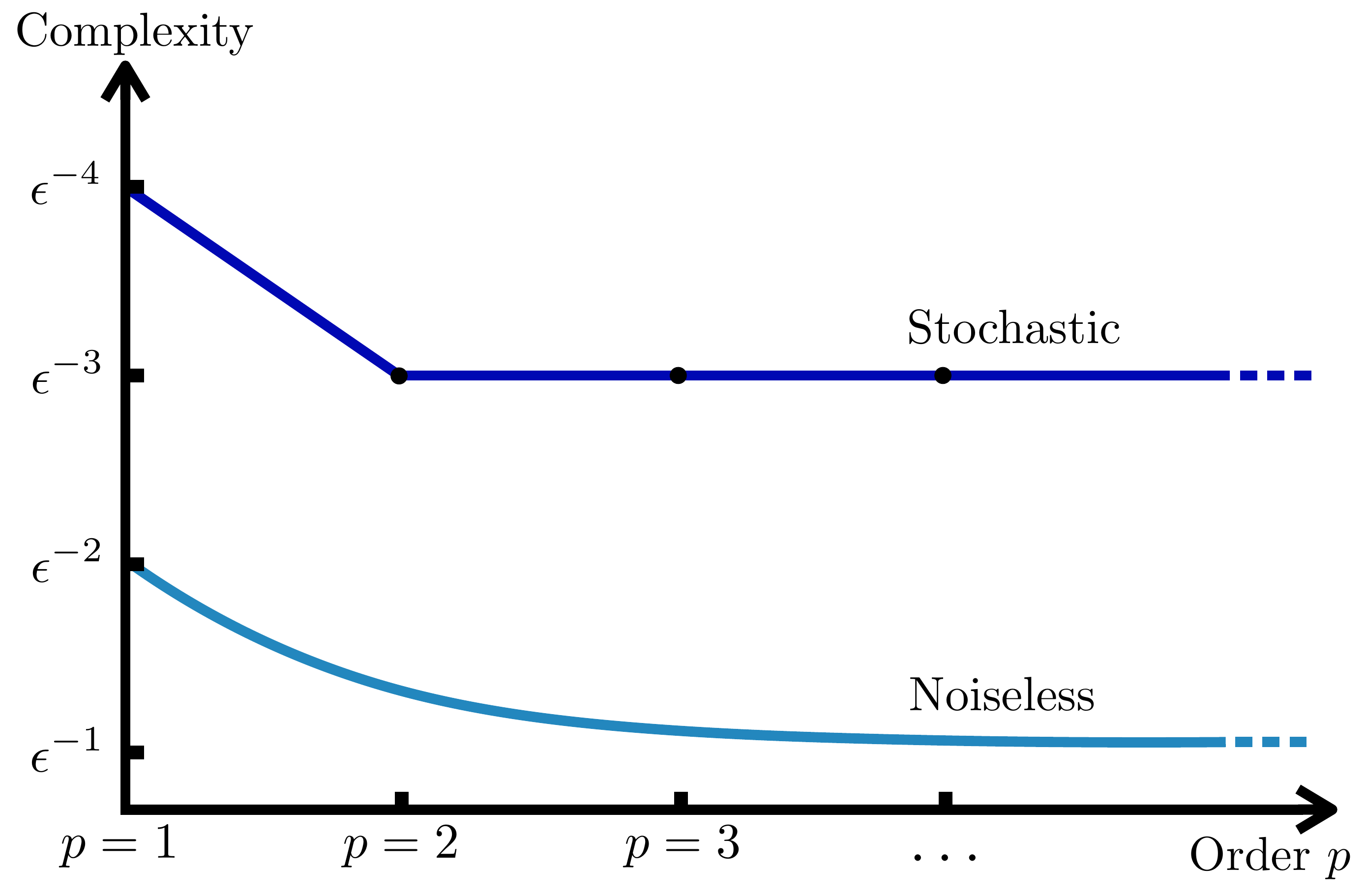}}
      \vspace{.85cm} 
      		\caption{The \emph{elbow effect}: For stochastic
      			oracles, the optimal complexity sharply improves from $\veps^{-4}$ 
      			for
      			$p=1$ to $\veps^{-3}$ for $p=2$,
      			but there is no further improvement for
                        $p>2$. For noiseless oracles, the optimal
                        complexity begins at $\eps^{-2}$ for $p=1$ and smoothly 
      			approaches 
      			$\eps^{-1}$ as the derivative order 
      			$p\to\infty$.}\label{fig:elbow}%
\end{figure}}

\colt{%
\captionsetup[figure]{font=scriptsize}
 \begin{figure}[t]
      \centering
      \begin{floatrow}
      	\ffigbox[0.43\textwidth]{%
      		      		\centering
      		\resizebox{0.45\textwidth}{!}{%
      		\includegraphics[width=0.5\textwidth]{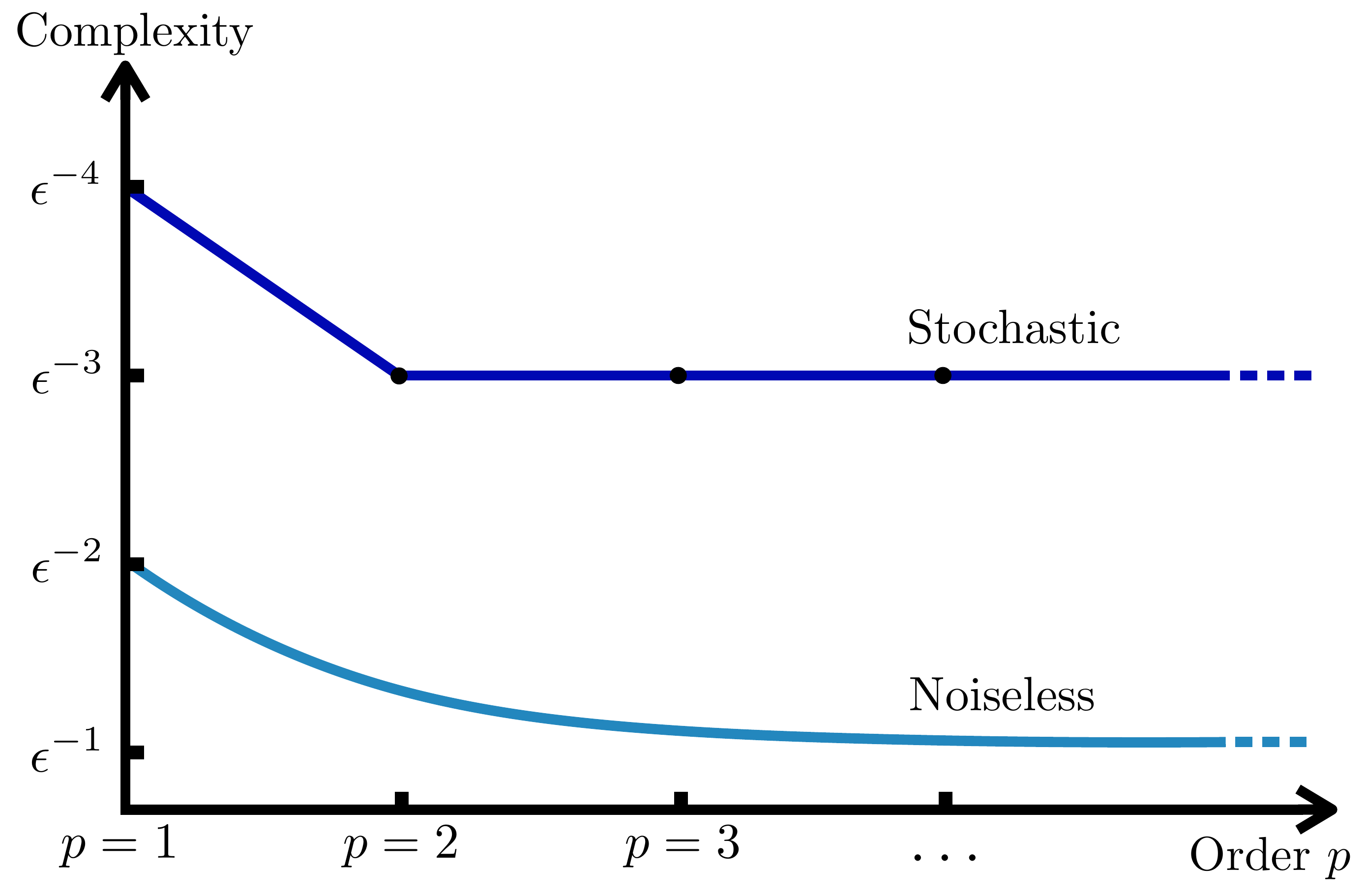}
              }
      \vspace{.85cm}
    }{%
      		\caption{The ``elbow effect:'' for stochastic
      			oracles, optimal complexity sharply improves from $\veps^{-4}$ 
      			for
      			$p=1$ to $\veps^{-3}$ for $p=2$,
      			but has no further improvement for $p>2$. Noiseless
      			complexity begins at $\eps^{-2}$ for $p=1$ and smoothly 
      			approaches 
      			$\eps^{-1}$ as the derivative order 
      			$p\to\infty$.}\label{fig:elbow}%
      	}
        ~
      	\capbtabbox[0.5\textwidth]{%
      		\resizebox{0.5\textwidth}{!}{
      		\begin{tabular}{p{3.8cm} p{1.4cm} p{1.6cm} p{2.5cm}}
      			\toprule
      			\makecell[lt]{Method} & \makecell[lt]{Requires \\
      				$\stochessNA$?} & \makecell[lt]{Complexity\\ bound} & 
      			\makecell[lt]{Additional\\ assumptions} \\
      			\midrule
      			\makecell[lt]{SGD~\citep{ghadimi2013stochastic}\\~} & No 
      			& 
      			$O(\epsilon^{-4})$ &\\
      			\makecell[lt]{Restarted SGD~\citep{fang2019sharp}} & No & 
      			$O(\eps^{-3.5})$ & \makecell[tl]{$\stocgradNA$ Lipschitz\\
      				almost surely}
      			\\
      			\makecell[lt]{Subsampled regularized\\
      				Newton~\citep{tripuraneni2018stochastic}} & Yes & 
      				$O(\eps^{-3.5})$ & \\
      			\makecell[lt]{Recursive variance\\ 
      			reduction~\citep[e.g.,][]{fang2018spider}} & No &
      			$O(\epsilon^{-3})$ 
      			&\makecell[tl]{Mean-squared
      				\\smoothness,\\Sim.\ queries \\
      				(see \pref{app:oracle})}
      			\\
      			\Xhline{4\arrayrulewidth}
      			\addlinespace[3pt]
      			Hessian-vector recursive VR ({\pref{alg:epsilon_hvp}}) & Yes & 
      			${O(\eps^{-3})}$& None\\
      			Subsampled Newton with VR ({\pref{alg:epsilon_cubic}}) & Yes & 
      			${O(	\eps^{-3})}$& None\\
      			\bottomrule 
      		\end{tabular}}
      	}{%
      		 \caption{Comparison of guarantees for finding $\eps$-stationary 
      			points 
      			(i.e., $\E \norm{\grad F(x)}\le \eps$) for a function $F$ with
      			Lipschitz gradient and Hessian. See
      			\pref{table:summary_detailed} for detailed comparison.}
      		\label{table:summary}%
      	}
    \end{floatrow} 
\end{figure}
\captionsetup[figure]{font=normalsize}

}

We provide new upper and lower bounds on the 
stochastic oracle complexity of finding $\eps$-stationary points and ($\epsilon,\gamma)$-SOSPs. In
brief, our main results are as follows.
\arxiv{
  \begin{itemize}
  }
  \colt{
  \begin{itemize}[leftmargin=\parindent]
    }
\item \textbf{Finding $\veps$-stationary points: The elbow effect.} 
	We propose a new algorithm that finds an $\veps$-stationary point 
	($\gamma=\infty$) with  $O(\veps^{-3})$ stochastic 
	gradients and stochastic Hessian-vector products. We furthermore show  
	that this guarantee is not improvable via a complementary
  $\Omega(\eps^{-3})$ lower bound. All previous algorithms achieving
  $O(\veps^{-3})$ complexity require ``multi-point'' queries, in
  which the algorithm can query stochastic gradients at multiple
  points for the same random seed. Moreover, we show that
  $\Omega(\eps^{-3})$ remains a lower bound for stochastic \pth-order
  methods for all $p\geq{}2$ and hence---in contrast to the deterministic 
  setting---the optimal rates for higher-order methods exhibit an ``elbow
  effect''; see \pref{fig:elbow}.
  
\item \textbf{$(\veps,\gamma)$-stationary points: Improved algorithm and
  nearly matching lower bound.} We extend our algorithm to find $(\eps,\gamma)$-stationary points using $O(\veps^{-3}+
  \veps^{-2}\gamma^{-2} + \gamma^{-5})$ stochastic gradient and
  Hessian-vector products, and prove a nearly matching   $\Omega(\veps^{-3}+ \gamma^{-5})$ lower bound.
\end{itemize}

In the remainder of this section we overview our results in greater
detail. Unless otherwise stated, we assume $F$ has both Lipschitz
gradient and Hessian. To simplify the overview, we focus on dependence
on $\eps^{-1}$ and $\gamma^{-1}$ while keeping the other
parameters---namely the initial optimality 
gap $F(x\ind{0})-\inf_{x\in\R^d} F(x)$, the Lipschitz constants of $\grad F$ 
and $\hess
F$,  and the variances of their estimators---held fixed. Our main theorems 
give
explicit dependence on these parameters.

\subsubsection{First-order stationary points ($\gamma=\infty$)}
We first describe our developments for the task of finding 
$\eps$-approximate first-order 
stationary points (satisfying~\pref{eq:sosp} with $\gamma=\infty$), and 
subsequently extend our results to general $\gamma$. The reader may also 
refer to \pref{table:summary}  for a succinct comparison of upper bounds.

\paragraph{Variance reduction via Hessian-vector products: A new gradient estimator.}
Using stochastic gradients and stochastic Hessian-vector products as primitives, we 
design a new variance-reduced gradient estimator. Plugging it into standard stochastic gradient descent (SGD), we obtain an algorithm that returns a 
point 
$\hx$ satisfying $\E \,\norm{\grad F(\hx)}\le \eps$
and requires 
$O(\eps^{-3})$ stochastic gradient and HVP queries in expectation. In comparison, vanilla SGD requires 
$O(\eps^{-4})$ queries~\citep{ghadimi2013stochastic}, and the previously 
best known rate under our assumptions
was $O(\eps^{-3.5})$, by both cubic-regularized Newton's method and a
restarted variant of SGD \citep{tripuraneni2018stochastic,fang2019sharp}.

Our approach builds on a line of work by~\citet{fang2018spider,zhou2018stochastic,
	wang2018spiderboost,cutkosky2019momentum} 
      that also develop algorithms with complexity $O(\eps^{-3})$, but
      require a  ``multi-point'' oracle
      in which algorithm can query the stochastic gradient at
      multiple points for the same random seed. Specifically, in the
      $n$-point variant of this model, the algorithm can query at the set of points $(x_1, \ldots, x_n)$ and receive 
\begin{equation}
			\stocder{x_1, z}{}, \ldots, \stocder{x_n, z}{}, \quad \text{where} 
			\quad z \overset{\iid}{\sim} P_z, \label{eq:n_point_query_model}
\end{equation}
and where the estimator $\stocder{x,z}{}$ is unbiased and has bounded
variance in the sense of \pref{eq:stoc-grad}. The aforementioned works achieve
$O(\eps^{-3})$ complexity using $n=2$ simultaneous queries, while our 
new algorithm
achieves the same rate using $n=1$ (i.e., $z$ is drawn afresh at each
query), but using stochastic Hessian-vector products in addition to
stochastic gradients. However, we show in
\pref{app:oracle} that under the statistical assumptions made in these
works, the two-point stochastic gradient oracle model is
\emph{strictly} stronger than the single-point stochastic
gradient/Hessian-vector product oracle we consider here. On the other hand, unlike our
algorithm, these works do not require Lipschitz Hessian.

The algorithms that achieve complexity $O(\eps^{-3})$ using
two-point queries
work by estimating gradient differences of the form $\grad 
F(x) - \grad F(x')$ using 
$\stocgrad{x,z} - \stocgrad{x',z}$ and applying recursive variance 
reduction~\citep{nguyen2017sarah}. Our primary algorithmic
contribution is a second-order stochastic estimator 
for $\grad F(x) - \grad F(x')$ which avoids simultaneous queries while 
maintaining comparable error guarantees. To derive our estimator, we note 
that $\grad F(x) - \grad F(x')  = \int_0^1 \hess F(xt + x'(1-t))
(x-x') dt$, and 
use $K$ queries to the stochastic 
Hessian
estimator~\pref{eq:stoc-hess} to numerically approximate this
integral.\footnote{%
	More precisely, our estimator~\pref{eq:key-estimator} only requires 
	stochastic Hessian-vector 
	products, whose computation is often roughly as expensive as that of a 
	stochastic gradient~\citep{pearlmutter1994fast}.}
      Specifically, our estimator takes the form
\begin{equation}\label{eq:key-estimator}
\frac{1}{K}\sum_{k=0}^{K-1} \wh{\grad^{2}F}\bigl(x\cdot (1-\tfrac{k}{K}) + x' \cdot 
\tfrac{k}{K}, 
z\ind{i} \bigr)
(x-x'),
\end{equation}
where $z\ind{i}\iidsim{}P_z$. Unlike the usual estimator $\stocgrad{x,z} - \stocgrad{x',z}$, the 
estimator~\pref{eq:key-estimator} is biased.
Nevertheless, we show that 
\emph{choosing $K$ dynamically} according to $K\propto \norm{x-x'}^2$ 
provides adequate control over both bias and variance while maintaining 
the desired 
query complexity. Combining the
integral estimator~\pref{eq:key-estimator} with recursive variance 
reduction, we attain $O(\eps^{-3})$ complexity.

\paragraph{Demonstrating the power of second-order information.}
For functions with Lipschitz gradient and Hessian, we prove an
$\Omega(\eps^{-3.5})$ lower bound on the minimax oracle complexity 
of algorithms 
for finding stationary points using \emph{only} stochastic
gradients~\pref{eq:stoc-grad}.\footnote{We formally prove our results for the structured class 
 of \emph{zero-respecting algorithms}~\citep{carmon2019lower_i}; the
 lower bounds
 extend to general randomized algorithms via similar arguments 
 to~\citet
 {arjevani2019lower}.} This lower 
bound is an extension of the results 
of~\citet{arjevani2019lower}, who showed that for functions with Lipschitz
  gradient but \emph{not} Lipschitz Hessian, the optimal rate is
  $\Theta(\eps^{-4})$ using \textit{only} stochastic
  gradients~\pref{eq:stoc-grad}. Together with our new $O(\eps^{-3})$ upper bound, this lower 
bound reveals 
that stochastic Hessian-vector products offer an $\Omega(\eps^{-0.5})$ 
improvement in the oracle complexity for finding stationary points in
the single-point query model. 
This contrasts the noiseless optimization setting, where 
finite 
gradient differences can approximate Hessian-vector products arbitrarily well,
meaning these oracle models are equivalent. 
\arxiv{
\begin{table}
	\small
  \renewcommand{\arraystretch}{1.1}
      		\begin{tabular}{p{5.5cm} p{1.6cm} p{2cm} p{5.8cm}}
      			\toprule
      			\makecell[lt]{Method} & \makecell[lt]{Requires \\
      				$\stochessNA$?} & \makecell[lt]{Complexity\\ bound} & 
      			\makecell[lt]{Additional\\ assumptions} \\
      			\midrule
      			\makecell[lt]{SGD \citep{ghadimi2013stochastic}} & No 
      			& 
      			$O(\epsilon^{-4})$ &\\[9pt]
      	      			\makecell[lt]{Restarted SGD \citep{fang2019sharp}} & No & 
      			$O(\eps^{-3.5})$ & \makecell[tl]{$\stocgradNA$ Lipschitz\\
      				almost surely}
                  \\[9pt]
      			\makecell[lt]{Subsampled regularized 
      				Newton\\\citep{tripuraneni2018stochastic}} & Yes$^{*}$ & 
      				$O(\eps^{-3.5})$ & \\[9pt]
      			\makecell[lt]{Recursive variance 
      			reduction\\\citep[e.g.,][]{fang2018spider}} & No &
      			$O(\epsilon^{-3})$ 
      			&\makecell[tl]{Mean-squared
      				smoothness, \\simultaneous queries 
      				(see \pref{app:oracle})}
      			\\
      			\Xhline{4\arrayrulewidth}
      			\addlinespace[3pt]
                  \makecell[lt]{SGD with \mainalg\\({\pref{alg:epsilon_hvp}})} & Yes$^{*}$ & 
      			${O(\eps^{-3})}$& None\\[9pt]
      			Subsampled Newton \\w/ \mainalg ({\pref{alg:epsilon_cubic}}) & 
      			Yes & 
      			${O(	\eps^{-3})}$& None\\
      			\bottomrule 
      		\end{tabular}
      		 \caption{Comparison of guarantees for finding $\eps$-stationary 
      			points 
      			(i.e., $\E \norm{\grad F(x)}\le \eps$) for a function $F$ with
      			Lipschitz gradient and Hessian. See
      			\pref{table:summary_detailed} for explicit
                        dependence on problem parameters. Algorithms
                        marked with $^{*}$ require only stochastic Hessian-vector products.}
      		\label{table:summary}%
		\end{table} 
}

\paragraph{Demonstrating the limitations of higher-order 
information ($p > 2$).}
For algorithms that can query both stochastic gradients and stochastic 
Hessians, we prove a lower bound of $\Omega(\eps^{-3})$ on the oracle 
complexity of finding an expected $\epsilon$-stationary point. 
This proves that our $O(\eps^{-3})$ upper bound is optimal in the leading order term in 
$\eps$, despite using only stochastic Hessian-vector products rather than 
full stochastic Hessian queries.

Notably, our $\Omega(\eps^{-3})$ lower 
bound extends to 
settings where stochastic higher-order oracles are available, i.e, 
when the first $p$ derivatives are Lipschitz and we have bounded-variance estimators $\{\stocder{\cdot, 
\cdot}{q}\}_{q \le p}$. The lower bound holds for any finite $p$, and thus, as a function of the 
oracle order $p$, the minimax complexity has an elbow
(\pref{fig:elbow}): 
for $p=1$ the complexity is $\Theta(\eps^{-4})$~\citep{arjevani2019lower} 
while for all $p\ge 2$ it is $\Theta(\eps^{-3})$. This means that 
smoothness and stochastic 
derivatives beyond the second-order cannot improve the leading term in rates of 
convergence to stationarity, establishing a fundamental limitation of 
stochastic high-order information. This highlights another contrast with the noiseless setting, where $p$th order methods enjoy
improved complexity for every $p$ \citep{carmon2019lower_i}.

As we discuss in 
\pref{app:oracle}, for multi-point stochastic oracles
\pref{eq:n_point_query_model}, the 
rate $O(\eps^{-3})$ is attainable even without stochastic Hessian
access. Moreover, our $\Omega(\eps^{-3})$ lower bound for stochastic
$p$th order oracles holds even when multi-point queries are
allowed. Consequently, when viewed through the lens of  worst-case oracle complexity, our lower bounds show that even stochastic Hessian information is not 
helpful in the multi-point setting.

\subsubsection{Second-order stationary points}

\paragraph{Upper bounds for general $\gamma$.}
We incorporate our recursive variance-reduced Hessian-vector product-based gradient 
estimator into an algorithm that combines SGD with negative curvature
search. Under the slightly stronger (relative to \pref{eq:stoc-hess})
assumption that the stochastic Hessians have almost surely bounded error,
we prove that---with constant probability---the algorithm returns an $(\epsilon, \gamma)$-SOSP after performing $O(\eps^{-3} + 
\eps^{-2}\gamma^{-2} + \gamma^{-5})$  stochastic gradient and
Hessian-vector product queries. 
\paragraph{A lower bound for finding second-order stationary points.}
We prove a minimax lower bound which establishes that the stochastic second-order oracle complexity of finding 
$(\eps, \gamma)$-SOSPs is $\Omega(\eps^{-3} + \gamma^{-5})$. 
Consequently, the algorithms we develop have optimal worst-case 
complexity in the regimes $\gamma=O(\epsilon^{2/3})$ and $\gamma = 
\Omega(\eps^{0.5})$. Compared to  
our lower bounds for finding $\veps$-stationary points, 
proving the $\Omega(\gamma^{-5})$ lower bound requires a more 
substantial modification of the constructions 
of~\cite{carmon2019lower_i} and \cite{arjevani2019lower}. In fact, our lower bound is new 
even in the noiseless regime (i.e., $\sigma_1=\sigma_2=0$), where it 
becomes $\Omega(\eps^{-1.5} + \gamma^{-3})$; this matches the 
guarantee of the cubic-regularized Newton's 
method~\citep{nesterov2006cubic} and consequently characterizes the 
optimal rate for finding approximate SOSPs using noiseless second-order 
methods.
\subsection{Further related work} We briefly survey additional upper and lower complexity bounds 
related to our work and 
place our results within their context.
The works of \citet{monteiro2013accelerated,arjevani2019oracle,agarwal2017lower} delineate the second-order
oracle complexity 
of \emph{convex} optimization in the noiseless setting;
\cite{arjevani2017oracle} treat the finite-sum 
setting.

For 
functions with Lipschitz gradient and Hessian, oracle access to the Hessian 
significantly accelerates convergence to $\varepsilon$-approximate global 
minima, reducing the complexity from $\Theta(\varepsilon^{-0.5})$ to 
$\Theta(\varepsilon^{-2/7})$. 
However, since the hard instances for first-order 
convex optimization are  
quadratic~\citep{nemirovski1983problem,arjevani2016iteration,simchowitz2018randomized},
 assuming Lipschitz continuity of the Hessian does not improve the 
complexity if one only has access to a first-order oracle. This contrasts the case for finding $\eps$-approximate stationary points of 
\emph{non-convex} 
functions with noiseless oracles. There, Lipschitz continuity of the 
Hessian 
improves the first-order oracle complexity from $\Theta(\eps^{-2})$ to 
$O(\eps^{-1.75})$, with a lower bound of $\Omega(\eps^{-12/7})$ for 
deterministic algorithms~\citep{carmon2017convex,carmon2019lower_ii}. 
Additional access to full Hessian further improves this complexity to 
$\Theta(\eps^{-1.5})$, and for \pth-order oracles with Lipschitz \pth
derivative, the complexity further improves to
$\Theta(\eps^{-(1+\frac{1}{p})})$ \citep{carmon2019lower_i}; see
\pref{fig:elbow}.

\subsection{Paper organization}
We formally introduce our notation
and oracle model in \pref{sec:prelim}. 
\pref{sec:epsilon} contains our results concerning the complexity of finding 
$\epsilon$-first-order stationary points: algorithmic upper bounds 
(\pref{sec:epsilon_upper}) and algorithm-independent lower bounds 
(\pref{sec:epsilon_lower}). Following a similar outline, 
\pref{sec:epsilon_gamma} describes our upper and lower bounds for finding 
$(\eps, \gamma)$-SOSPs. \arxiv{We conclude the paper in \pref{sec:conclusion} 
with a discussion of directions for further research. }\colt{In
\pref{app:conclusion}, we discussion directions for future research. }Additional
technical comparison with related work is given in \pref{app:comparison}
and \ref{app:oracle}, and proofs are given in \pref{app:estimators}
through \pref{app:lower}.

\newcommand{\xmark}{\ding{55}}
\newcommand{\cmark}{\ding{51}}

\paragraph{Notation.} 
We let $\cC^p$ denote the class of $p$-times 
differentiable real-valued functions, and let $\grad^q F$ 
denote the $q$th derivative of a  
given function $F\in \cC^p$ for $q\in\{1,\dots,p\}$. 
Given a function $F\in\cC^1$, we let $
\grad_i F(x) \defeq \brk*{\grad 
F(x)}_i = \frac{\partial}{\partial x_i} F(x)$.  When $F\in\cC^2$ 
is twice 
differentiable, we 
define, $\nabla^2_{ij} f(x) 
\ldef \brk*{\grad^2 f(x)}_{ij} = 
\frac{\partial^2}{\partial x_i \partial x_j} f(x)$, 
and similarly 
define $\brk*{\grad^p f(x)}_{i_1,i_2,\ldots,i_p}=\frac{\partial^p}{\partial 
x_{i_1}\cdots \partial x_{i_p}} f(x)$ for 
\pth-order 
derivatives. 
For a vector $x\in\bbR^{d}$, $\nrm*{x}$ denotes the Euclidean 
norm and 
$\nrm*{x}_\infty$ denotes the $\ls_\infty$ norm. For matrices
$A\in\bbR^{d\times{}d}$, $\nrm*{A}_{\op}$ denotes the operator
norm. More generally, for symmetric \pth
order tensors $T$, we define the operator norm via
$\nrm*{T}_{\op}=\sup_{\nrm*{v}=1}\abs*{\tri*{T,v^{\otimes{}p}}}$, and
we let $T[v\ind{1},\ldots,v\ind{p}]=\tri*{T,v\ind{1}\otimes\cdots\otimes{}v\ind{p}}$.
Note that for a vector $x\in\bbR^{d}$ the operator norm $\nrm*{x}_\op$
coincides with the Euclidean norm $\nrm{x}$. We let $\symd$ denote the
space of symmetric matrices in $\bbR^{d\times{}d}$. We let $\bbB_r(x)$ denote
the Euclidean ball of radius $r$ centered at $x\in\bbR^{d}$ (with
dimension clear from context). We adopt non-asymptotic
big-O notation, where $f=O(g)$ for $f,g:\cX\to\bbR_{+}$ if 
$f(x)\leq{}Cg(x)$ for some constant $C>0$.

\section{Setup} \label{sec:prelim} 

We
study the problem of finding $\eps$-stationary and
$(\eps,\gamma)$-second order stationary points in the standard
 oracle complexity framework \citep{nemirovski1983problem}, which we
briefly review here.

\paragraph{Function classes.}
We consider $p$-times differentiable functions satisfying
standard regularity conditions, and define
\begin{equation*}
  \Fclasstilp = \left\{ F:\R^{d}\to\R \left| 
    \begin{array}{l}
      F \in \cC^p, \quad F(0) - \inf_x F(x) \leq \Delta, \\ 
      \nrm*{\grad^q{}F(x)-\grad{}^qF(y)}_\op \leq{}L_q\nrm*{x-y}
      ~\mbox{for~all~} x, y \in \R^d,~q\in[p] 
    \end{array}
    \right. \!\! \right\},
\end{equation*}
so that $\Lip{1:p}\defeq(L_1,\dots,L_p)$ specifies the Lipschitz constants 
of the $q$th order derivatives $\nabla^q F$ with respect to the operator norm.	We make no restriction on the ambient dimension $d$.
	\paragraph{Oracles.} 
        For a given function $F\in\Fclasstilp$, we consider a class of stochastic 
        \pth order oracles defined by a distribution $P_z$ over a measurable 
        set $\cZ$ and an estimator
	\begin{equation} \label{eq:oracle_answer}
  		\estimator(x,z) \defeq \prn*{\wh{F}(x,z), \stocder{x, 
  		z}{}, \stocder{x, z}{2}, \ldots, \stocder{x, z}{p}}, %
	\end{equation} 
        where 
	$\crl{\stocder{\cdot, z}{q}}_{q=0}^{p}$ are unbiased estimators of the respective 
	derivatives. That is, for all $x$, {$\En_{z\sim{}P_z}[\wh{F}(x,z)]=F(x)$} and $\En_{z 
	\sim P_z}[\stocder{x, z}{q}] = \grad^q 
	F(x)$ for all $q\in[p]$.\footnote{For $p\ge2$ we assume without 
	loss of generality that $\stocder{x,z}{p}$ is a symmetric tensor.}

	Given variance parameters $\sigma_{1:p} = (\sigma_1, 
	\ldots, \sigma_{p})$, we define the \textit{oracle class} 
	$\stocpOracle$ to be the set of all stochastic 
	\pth-order oracles for which the variance of the derivative estimators 
	satisfies
	\begin{equation} \label{eq:oracle_variance}
	\En_{z\sim P_z}\,{\nrm*{ 
	 \stocder{x, z}{q} - \nabla^q F(x) }_\op^2} \leq 
	 \sigma_q^2,~~q\in [p].
       \end{equation}
       The upper bounds in this paper hold even when 
       $\sigma_0^{2}\ldef{}\max_{x\in\bbR^{d}}\mbox{Var}(\wh{F}(x,z))$ is infinite,
       while our lower bounds hold when $\sigma_0=0$, so
       to reduce notation, we leave dependence on this parameter
       tacit.
	\paragraph{Optimization protocol.}
	We consider stochastic \pth-order optimization algorithms that access an 
	unknown 
	function $F\in\Fclasstilp$ through multiple rounds of queries to a 
	stochastic \pth-order oracle $(\oracle_F^{p},P_z) \in \stocpOracle$. 
	When queried at $x\ind{t}$ in round $t$, the oracle performs an 
        independent draw of $z^{(t)} \sim P_z$ and answers with 
        $\oracle_F^{p}(x\ind{t},z\ind{t})$. Algorithm queries depend on $F$ 
        only through the oracle answers; see e.g.\ \citet[Section 
        2]{arjevani2019lower} for a more formal treatment.

\section{Complexity of finding first-order stationary points} 
\label{sec:epsilon} 

In this section we focus on the task of finding $\epsilon$-approximate 
stationary points (satisfying $\norm{\grad F(x)}\le \epsilon$). As prior work 
observes~\citep[cf.][]{carmon2017convex,allen2018how}, stationary point 
search is a useful primitive for achieving the end goal of finding 
second-order stationary points~\eqref{eq:sosp}. We begin with describing 
algorithmic upper bounds on the complexity of finding stationary points 
with stochastic second-order oracles, and then proceed to match their 
leading terms with general $p$th order lower bounds.

\subsection{Upper bounds} \label{sec:epsilon_upper}
Our algorithms rely on \emph{recursive variance 
reduction}~\citep{nguyen2017sarah}: 
we sequentially estimate the 
gradient at the points $\{x\ind{t}\}_{t\ge0}$ by accumulating cheap 
estimators 
 of $\grad F(x\ind{\tau})-\grad 
F(x\ind{\tau-1})$ for $\tau=t_0+1,\ldots, t$, where at iteration $t_0$ we 
reset the gradient estimator by computing a high-accuracy approximation 
of $\grad F(x\ind{t_0})$ with many oracle 
queries. Our implementation of recursive variance reduction, 
\pref{alg:gradient_estimator}, differs from previous approaches~\citep{fang2018spider,zhou2018stochastic,wang2018spiderboost} in 
three aspects.
  \colt{

 First, in \pref{line:hvp_update_step} we estimate differences of
	the form $\grad F(x\ind{\tau})-\grad F(x\ind{\tau-1})$ by averaging 
	stochastic Hessian-vector products. This allows us to do away
        with multi-point queries and operate under 
	weaker assumptions than prior work (see \pref{app:oracle}), but it also 
	introduces bias to our estimator, which makes its analysis more involved. 
	This is the key novelty in our algorithm.
         Second, rather than resetting the gradient estimator 
	every fixed number of steps, we reset with a 
	user-defined probability $\bias$ (\pref{line:initialization_est}); 
	this makes the estimator stateless and greatly simplifies its analysis, 
	especially in our algorithms for finding SOSPs, where we use a
        varying value of $\bias$. Finally, we dynamically select 
	the batch size $\numbabysteps$ for estimating gradient differences 
	based on the distance between iterates (\pref{line:rvr-params}), while 
	prior work uses a constant batch size. Our dynamic batch size scheme is 
	crucial for controlling the bias in our  estimator,
        while still allowing for large step sizes as in~\citet{wang2018spiderboost}.
    }
\arxiv{
  \begin{enumerate}
	\item In \pref{line:hvp_update_step} we estimate differences of
	the form $\grad F(x\ind{\tau})-\grad F(x\ind{\tau-1})$ by averaging 
	stochastic Hessian-vector products. This allows us to do away
        with multi-point queries and operate under 
	weaker assumptions than prior work (see \pref{app:oracle}), but it also 
	introduces bias to our estimator, which makes its analysis more involved. 
	This is the key novelty in our algorithm.
	\item Rather than resetting the gradient estimator 
	every fixed number of steps, we reset with a 
	user-defined probability $\bias$ (\pref{line:initialization_est}); 
	this makes the estimator stateless and greatly simplifies its analysis, 
	especially when we use a varying value of $\bias$ to find second-order 
	stationary points.
	\item We dynamically select 
	the batch size $\numbabysteps$ for estimating gradient differences 
	based on the distance between iterates (\pref{line:rvr-params}), while 
	prior work uses a constant batch size. Our dynamic batch size scheme is 
	crucial for controlling the bias in our gradient estimator,
        while still allowing for large step sizes as in~\citet{wang2018spiderboost}.
      \end{enumerate}
      }

The core of our analysis is the following lemma, which bounds the
gradient estimation error and expected oracle complexity. To
state the lemma, we let $\{x\ind{t}\}_{t\ge0}$ be sequence of queries to
  \pref{alg:gradient_estimator}, and let ${g\ind{t} =
  \getgradientfunc_{\epsilon,b}(x\ind{t}, x\ind{t-1}, g\ind{t-1})}$
be the sequence of estimates it returns.
\begin{lemma}\label{lem:base-var-bound}
	For any oracle in  $\cO_{2}(F, \sigma_{1:2})$ and 
  $F \in \cF_2\prn{\Delta, L_{1:2}}$, \pref{alg:gradient_estimator} guarantees that
	\begin{equation*}
	\E\, \norm{g\ind{t} - \grad F(x\ind{t})}^2 \le \epsilon^2
      \end{equation*}
       for all $t\geq{}1$. 
Furthermore, conditional on $x\ind{t-1}$, $x\ind{t}$ and $g\ind{t-1}$, the 
$t^\textup{th}$ 
 execution of \pref{alg:gradient_estimator}  with reset probability $\bias$ 
uses at most
\[
O\Big(1 + \bias \frac{\sigma_1^2}{\epsilon^2} + {\nrm[\big]{x\ind{t} - 
x\ind{t-1}}^2 \cdot \frac{  
{\sigma_2^2 + \epsilon L_2}}{\bias\epsilon^2}}\Big)
\]
stochastic gradient and Hessian-vector product queries in expectation. 
\end{lemma}

We prove the lemma in \pref{app:estimators} by bounding the 
per-step variance using the HVP oracle's variance 
bound~\pref{eq:oracle_variance}, and by bounding the per-step bias 
relative to $\grad 
F(x\ind{t})-\grad F(x\ind{t-1})$ using the Lipschitz continuity of the
Hessian.

\colt{
\newlength{\textfloatsepsave}
\setlength{\textfloatsepsave}{\floatsep} 
\setlength{\floatsep}{6pt} 
\setlength{\textfloatsep}{0pt} 
}

\begin{algorithm}[t] 
	\caption{Recursive variance reduction with stochastic
		Hessian-vector products (\mainalg)} 
	\label{alg:gradient_estimator} 
	\begin{algorithmic}[1]
		\Statex{} \algcomment{Gradient estimator for $F \in \cF_2(\Delta, L_{1:2})$ given stochastic oracle in  $\cO_{2}(F, \sigma_{1:2})$.} 
		\Function{\getgradientfunc$_{\newestparams,\bias}$}{$x$, $\xprev$, 
			$\gprev$}: 
		\State  Set $\numbabysteps  = \ceil*{\frac{5 \prn*{ \sigma_2^2 + 
					L_2\epsilon}}{\bias \epsilon^2}\cdot \norm{x-\xprev}^2}$ and 
					$n = 
		\ceil*{\frac{5 
				\sigma_1^2}{\epsilon^2}}$.\label{line:rvr-params}
		\State Sample $\coin \sim \bern(\bias)$.\label{line:coin-toss} 
		\If {$\coin$ is  $1$ or $\gprev$ is $\bot$}{} 
		\label{line:initialization_est}
		\State\label{line:fresh_start} Query the oracle $\outerbatchsize $ 
		times at $x$ and set $g \leftarrow \frac{1}{n} 
		\sum_{j=1}^{\outerbatchsize} \stocder{x, z\ind{j}}{}$, where 
		$z\ind{j}\overset{\mathrm{i.i.d.}}{\sim}P_z.$
		\Else
		\State Define $x\ind{k} \ldef{} { \frac{k}{K}  x +
			\prn*{1 - \frac{k}{K}} \xprev}$ for $k \in \crl{0,\ldots,K}$. 
		\State \label{line:hvp_update_step} Query the oracle at the set of 
		points $\prn*{x\ind{k}}_{k=0}^{K-1}$ to compute 	
		\item[]  {\begin{center}
				$g \leftarrow \gprev +  \sum_{k=1}^K \stocder{x\ind{k-1},  
				z\ind{k}}{2} \prn*{x\ind{k} - x\ind{k-1}},\quad\text{where } 
				z\ind{k}\overset{\mathrm{i.i.d.}}{\sim}P_z.$ 
		\end{center}} 
		\EndIf 
		\State \Return $g$. 
		\EndFunction  
	\end{algorithmic}
\end{algorithm} 

\begin{algorithm}[t]
	\caption{Stochastic gradient descent with \mainalg} 
	\label{alg:epsilon_hvp} 
	\begin{algorithmic}[1]
		\Require  
		Oracle $(\estimator[2], P_z) \in 
		\stocpOracle[2]$ for $F \in \cF_2\prn*{\Delta, L_1, L_2}$.  
		Precision parameter $\epsilon$. 	
		\State Set $\eta = \frac{1}{2 \sqrt{ L_1^2 + \sigma_2^2 + \epsilon 
				L_2}}$, $T = \ceil*{\frac{2\Delta}{\eta \epsilon^2}}$, $\bias=\min \crl*{1, 
				\frac{\eta \epsilon \sqrt{\sigma_2^2 + \epsilon 
				L_2}}{\sigma_1}}$. \label{line:parameter_setting}
		\State Initialize $x\ind{0}, x\ind{1} \leftarrow 0$, $\gradest[0] 
		\leftarrow \bot$. 
		\For {$t = 1 ~\text{to}~ T$}
		\State $\gradest[t] \leftarrow$ $\getgradientfunc_{\newestparams, 
		\bias}(x 
		\ind{t}, x\ind{t-1}, \gradest[t-1])$. 
		\label{line:gradient_estimate_epsilon_hvp}
		\State $x \ind{t+1} \leftarrow x \ind{t} - \eta \gradest[t]$.  
		\label{line:hvp_update_rule}
		\EndFor 
		\State \textbf{return} $\hx$ chosen uniformly at random from 
		$\crl*{x\ind{t}}_{t=1}^{T}$.  
	\end{algorithmic} 
\end{algorithm}

Our first algorithm for finding $\eps$-stationary points, \pref{alg:epsilon_hvp}, is simply stochastic gradient descent using the 
\mainalg gradient estimator (\pref{alg:gradient_estimator}); we bound its 
complexity by
$O(\eps^{-3})$. Before stating the result formally, we briefly sketch the 
analysis here\colt{ (see \pref{app:epsilon_hvp} for details)}. Standard 
analysis of SGD with step size 
$\eta\le \frac{1}{2L_1}$ shows that its iterates satisfy $\E \norm{\grad 
F(x\ind{t})}^2 
\le \frac{1}{\eta}\E [F(x\ind{t+1}) - F(x\ind{t})] + O(1) \cdot \E\, 
\norm{g\ind{t} 
- \grad F(x\ind{t})}^2$. Telescoping over $T$ steps, using 
\pref{lem:base-var-bound} and substituting in the 
initial 
suboptimality bound $\Delta$, this implies that
\begin{equation} \label{eq:sgd-variance-bound}
\frac{1}{T} \sum_{t=0}^{T-1} \En{} \norm{\grad F(x\ind{t})}^2 \le
\frac{\Delta}{\eta T} + O(\eps^2).
\end{equation}
Taking  $T=\Omega(\frac{\Delta}{\eta \eps^2})$, we are guaranteed that
a uniformly selected iterate has expected norm $O(\veps)$.

\colt{\setlength{\textfloatsep}{\textfloatsepsave}}

To account for oracle complexity, we observe from
\pref{lem:base-var-bound} that $T$ calls to 
\pref{alg:gradient_estimator} require at most $T(\frac{\sigma_1^2 
\bias}{\epsilon^2} + 1) +   
\sum_{t=1}^{T} \En {\norm{x\ind{t}-x\ind{t-1}}^2} \cdot 
\big(\frac{\sigma_2^2 + L_2 \epsilon}{\bias\epsilon^2} \big)$ oracle 
queries in expectation. Using $x\ind{t}-x\ind{t-1}=\eta g\ind{t-1}$, 
\pref{lem:base-var-bound} and~\pref{eq:sgd-variance-bound} imply that 
$\sum_{t=1}^{T} \En {\norm{x\ind{t}-x\ind{t-1}}^2} \le O(T\epsilon^2)$. 
We 
then choose 
$\bias$ to out the  terms $T\big({\frac{\sigma_1^2 
\bias}{\epsilon^2}}\big)$ and $T\big({\frac{\sigma_2^2 + L_2\epsilon}{b}}\big)$.%
\arxiv{
This gives the 
following 
complexity guarantee, 
which we prove in~\pref{app:epsilon_hvp}.
}
\colt{
  This gives the 
following 
 guarantee.
}

\begin{theorem}\label{thm:epsilon_hvp}
 For any function $F \in \cF_2\prn*{\Delta, L_1, L_2}$, stochastic 
 second-order oracle in $\cO_{2}(F, \sigma_1, \sigma_2)$, and $\epsilon < 
\min \crl*{\sigma_1, \sqrt{\Delta{}L_1}}$, with probability at least 
$\frac{3}{4}$, \pref{alg:epsilon_hvp}  returns a point $\hx$ such that $
 \nrm*{\grad F(\hx)} \leq \epsilon$ and performs at most
 \begin{equation*}
 O\Big( 
 \frac{\Delta \sigma_1 \sigma_2}{\epsilon^3} + \frac{\Delta 
 L_2^{0.5}\sigma_1 }{\epsilon^{2.5}} 
 + \frac{\Delta L_1}{\epsilon^2} 
\Big)
 \end{equation*}
stochastic gradient and Hessian-vector product queries.
\end{theorem}

The oracle complexity of \pref{alg:epsilon_hvp} depends on the Lipschitz 
parameters of $F$ only through lower-order terms in $\epsilon$, with the 
leading term scaling only with the variance of the gradient and Hessian 
estimators. In the low noise regime 
where $\sigma_1 < \epsilon$ and $\sigma_2 < \max\{L_1,\sqrt{L_2\eps}\}$, the 
complexity becomes $O(\Delta L_1 \epsilon^{-2} + 
\Delta L_2^{0.5}\epsilon^{-1.5})$ which is simply the maximum of the noiseless guarantees for gradient descent and Newton's method. We remark,
however, that in the noiseless regime $\sigma_1=\sigma_2=0$, a slightly 
better 
guarantee $O(\Delta L_1^{0.5}L_2^{0.25} \epsilon^{-1.75} + 
\Delta L_2^{0.5}\epsilon^{-1.5})$ is achievable~\citep{carmon2017convex}.

In the noiseless setting, any algorithm that
uses only first-order and Hessian-vector product queries must have
complexity scaling with $L_1$, but full Hessian access can remove this 
dependence \citep{carmon2019lower_ii}. We show that
the same holds true in the stochastic setting:
\pref{alg:epsilon_cubic}, a subsampled cubic regularized trust-region 
method using \pref{alg:gradient_estimator} for gradient estimation, 
enjoys a complexity bound independent of $L_1$. \arxiv{We defer the
analysis to \pref{app:upper_bounds_e_L2} and state the guarantee as 
follows.}\colt{We defer the
full description and analysis to \pref{app:upper_bounds_e_L2} and state the guarantee here.}
\begin{theorem}\label{thm:epsilon_cubic}
	For any function $F \in \cF_2\prn*{\Delta, \infty, L_2}$, stochastic 
	second 
	order oracle in $\cO_{2}(F, \sigma_1, \sigma_2)$, and $\epsilon < 
	\sigma_1$, with probability at least $\frac{3}{4}$, \pref{alg:epsilon_cubic} returns a point $\hx$ such that 
	$ \nrm*{\grad F(\hx)} \leq \epsilon$ and performs at most
	\begin{equation*}
	O\Big( 
		\frac{\Delta \sigma_1 \sigma_2 }{\epsilon^3}\cdot{}\log^{0.5}d + \frac{\Delta 
			L_2^{0.5}\sigma_1 }{\epsilon^{2.5}} 
	\Big)
	\end{equation*}
	stochastic gradient and Hessian queries.
\end{theorem} 

The guarantee of \pref{thm:epsilon_cubic} constitutes an improvement in 
query complexity over
\pref{thm:epsilon_hvp} in the regime $L_1 \gtrsim
(1+\frac{\sigma_1}{\eps})(\sigma_2 + \sqrt{L_2 \epsilon})$. However, 
depending on the problem, full stochastic Hessians can be up to $d$ times 
more expensive to compute than stochastic Hessian-vector products. 

\arxiv{%
\begin{algorithm}[t]
   \caption{Subsampled cubic-regularized trust-region method with 
     \mainalg} 
	\label{alg:epsilon_cubic}
	\begin{algorithmic}[1] 
		\Require 
		 Oracle $(\estimator[2], P_z) \in 
\stocpOracle[2]$ for $F \in \cF_2\prn*{\Delta, \infty, L_2}$.
		 Precision parameter $\epsilon$.
		\State Set $M = 5 \max\crl*{L_2, \frac{\epsilon\sigma_2^2\log (d)}{\sigma_1^2}}$, $\eta =  25 \sqrt{\frac{\epsilon}{M}}$, $T = 
		\ceil*{\frac{5 \Delta}{3 \eta \epsilon}}$ and $n_H = \ceil*{\frac{22 \sigma_2^2 \eta^2 \log(d)}{\epsilon^2}}$.
		\State Set $\bias=\min\crl*{1, \frac{\eta \sqrt{\sigma_2^2 + \epsilon L_2}}{25 \sigma_1}}$. 
		\State Initialize $x\ind{0}, x\ind{1} \leftarrow 0$,~ $\gradest[0] \leftarrow \bot$.  
		\For {$t = 1 ~\text{to}~ T$}
		\State \label{line:Hessian_estimator}Query the oracle $n_H$ times at 
		$x \ind{t}$ and compute 
		\begin{equation*}
		\hessest[t] \leftarrow \frac{1}{n_H} \sum_{j=1}^{n_H} 
		\stocder{x\ind{t}, z\ind{t, j}}{2},\quad\text{where}\quad{}z\ind{t, 
		j}\overset{\mathrm{i.i.d.}}{\sim}P_z.
		\end{equation*}
		\State $\gradest[t] \leftarrow 
		\getgradientfunc_{\newestparams, \bias}\prn*{x\ind{t}, x\ind{t-1}, g\ind{t-1}}$.\label{line:gradient_estimation_cubic}  
		\State\label{line:constrained_cubic_step}Set the next point $x\ind{t+1}$ as 
	\begin{equation*}
			x\ind{t+1} \gets \argmin_{y:\norm{y-x\ind{t}}\le\eta} 
		\tri[\big]{\gradest[t], y - x\ind{t}} + \frac{1}{2}\tri[\big]{y-x\ind{t}, 
		\hessest[t](y - x\ind{t})} + \frac{M}{6} 
		\nrm[\big]{y-x\ind{t}}^{3}.
	\end{equation*}
		\EndFor
		\State \textbf{return} $\hx$ chosen uniformly at random from 
		$\crl*{x\ind{t}}_{t=2}^{T + 1}.$ 
\end{algorithmic} 
\end{algorithm}

}

\subsection{Lower bounds} \label{sec:epsilon_lower}

	Having presented stochastic second-order methods with
	$O(\eps^{-3})$-complexity bound for finding $\eps$-stationary 
	points, our we next show that this rates cannot be 
	improved. In fact, we show that this rate is optimal even when one is 
	given access to stochastic higher derivatives of \emph{any} order.
	We prove our lower bounds for the class of \emph{zero-respecting} 
	algorithms, which subsumes the majority of existing 
	optimization methods; see \pref{app:epsilon_lower} for a
        formal definition. We believe that existing 
	techniques~\citep{carmon2019lower_i,arjevani2019lower} can 
	strengthen our lower bounds to apply to general randomized 
	algorithms; for brevity, we do not pursue it here.  
	
	The lower bounds in this section closely follow a recent construction 
	by~\citet[Section 3]{arjevani2019lower}, who prove lower bounds for 
	stochastic first-order methods. To establish complexity bounds for 
	$p$th-order methods, we extend the `probabilistic zero-chain' gradient 
	estimator introduced in~\citet{arjevani2019lower} to high-order derivative 
	estimators.The most technically demanding part of our proof is a careful scaling of 
	the basic construction to simultaneously meet multiple Lipschitz continuity 
	and variance constraints. Deferring the proof details 
	to~\pref{app:epsilon_lower}, our 
	lower bound is as follows.

	\begin{theorem}
		\label{thm:eps_so_zero_respecting}
			For all $p\in\N$, $\Delta, \Lip{1:p}, \sigma_{1:p}>0$ and $\eps \le 
			O(\sigma_1)$, there exists $F \in \cF_p\prn*{\Delta, L_{1:p}}$ 
			and 
			$(\estimator[p], P_z) \in \stocpOracle[p]$, such that for any 
			\pth-order zero-respecting algorithm, the number of 
			 queries required to obtain an 
			$\eps$-stationary point with constant
                        probability is bounded from below by 
			\begin{align}\label{eq:minimax_eps_p}
			\Omega(1)\cdot \frac{\Delta\sigma_1^2}{\epsilon^3}
			\min\crl*{ 
				\min_{q\in\{2,\ldots,p\}}\prn*{\frac{\sigma_q}{\sigma_1}}^{\frac{1}{q-1}},
				\min_{q'\in\{1,\ldots,p\}}\prn*{\frac{L_{q'}}{\eps }}^{1/q'}}.
			\end{align}
			A construction of dimension 
			$\Theta\Bigl(\frac{\Delta 
			}{\epsilon}		
			\min\Bigl\{
			\min_{q\in\{2,\ldots,p\}}\prn*{\frac{\sigma_q}{\sigma_1}}^{\frac{1}{q-1}},
			\min_{q'\in\{1,\ldots,p\}}\prn*{\frac{L_{q'}}{\eps }}^{1/q'}\Bigr\}\Bigr)$ realizes 
			this lower bound.
	\end{theorem}
			For second-order methods (with $p=2$), 
		\pref{thm:eps_so_zero_respecting} 
		specializes to the oracle complexity lower bound 
		\begin{align}\label{eq:minimax_eps_2}
		\Omega(1)\cdot 
		\min\crl*{ 				
			\frac{\Delta\sigma_1\sigma_2}{\epsilon^3}, 
				\frac{\Delta L_2^{0.5}\sigma_1}{\epsilon^{3.5}},
				\frac{\Delta L_1 \sigma_1^2}{\epsilon^{4}}},
		\end{align}
		which is tight in that it matches (up to 
		numerical constants) the convergence rate of~\pref{alg:epsilon_hvp} 
		in the regime where 
		$\Delta 
		\sigma_1\sigma_2\eps^{-3}$ dominates both the upper bound 
		in~\pref{thm:epsilon_hvp} and expression~\pref{eq:minimax_eps_2}. The lower 
		bound~\pref{eq:minimax_eps_2} is also tight when the second-order 
		information is not available or reliable
                ($\sigma_2$ is infinite or very 
		large, respectively): Standard SGD matches the $\eps^{-4}$ 
		term~\citep{ghadimi2013stochastic}, while more
                sophisticated variants based on
                restarting~\citep{fang2019sharp} and normalized
                updates with momentum \citep{cutkosky2020momentum} match the 
		$\eps^{-3.5}$ term 
		(the former up to logarithmic factors)---neither of 
		these algorithms requires stochastic second derivative estimation. 
				
		\pref{thm:eps_so_zero_respecting} implies that while  higher-order 
		methods (with $p>2$) might achieve better dependence on the 
		variance parameters than the upper bounds 
		for~\pref{alg:epsilon_hvp} or \pref{alg:epsilon_cubic}, they cannot improve the
		$\eps^{-3}$ scaling. This highlights a fundamental limitation for higher-order methods in 
		stochastic non-convex optimization which does not exist in the 
		noiseless case. Indeed, without noise the optimal rate for finding 
		$\epsilon$-stationary point with a $p$th order method is 
		$\Theta(\eps^{-1+\frac{1}{p}})$~\cite{carmon2019lower_i}; we 
		illustrate this contrast in \pref{fig:elbow}. 
                
       Altogether, the results presented in this section
                fully characterize (with respect to dependence on
                $\eps$) the complexity of finding $\eps$-stationary
                points with stochastic second-order methods and beyond in the
                single-point query model. We briefly remark that lower
                bound in \pref{eq:minimax_eps_p} immediately extends
                to multi-point queries, which shows that even
                second-order methods offer little benefit once two or
                more simultaneous queries are allowed.

\section{Complexity of finding second-order stationary points} 
Having established rates of convergence for finding $\eps$-stationary
points, we now turn our attention to $(\eps,\gamma)$-second order
stationary points, which have the additional 
requirement that $\lambda_{\min}(\hess F(x)) \ge -\gamma$, i.e.\ that $F$ 
is $\gamma$-weakly convex around $x$. This section follows the general 
organization of the prequel: we first design and analyze an algorithm 
with improved upper bounds, and then develop nearly-matching lower bounds that apply to a 
broad class of algorithms.

\label{sec:epsilon_gamma}
\arxiv{\subsection{Upper bounds}}
\colt{\subsection{Upper bounds}}
\label{sec:epsilon_gamma_upper}
Our first contribution for this section is an algorithm that enjoys improved
complexity for finding $(\eps,\gamma)$-second-order stationary points,
and that achieves this using only stochastic gradient and Hessian-vector product queries.  
To guarantee second-order stationarity, we follow the established technique 
 of interleaving an algorithm for finding a first-order stationary point with 
 negative curvature descent~\citep{carmon2017convex, allen2018how}. 
 However, we employ a randomized variant of this approach. Specifically, at every iteration we 
 flip a biased coin to determine whether to perform a stochastic 
 gradient step or a stochastic negative curvature descent step.
 
Our algorithm estimates stochastic gradients using the \mainalg scheme 
(\pref{alg:gradient_estimator}), where the value of the restart probability 
$\bias$ depends on the type of the previous step (gradient or negative 
curvature). To implement negative curvature 
descent, 
we apply Oja's 
 method~\citep{oja1982simplified,allen2017follow} which detects
 directions of negative curvature using only stochastic Hessian-vector
 product queries. For technical reasons pertaining to the analysis of Oja's method, we require 
 the stochastic 
 Hessians to be bounded almost surely, i.e.,  
 $\opnorm{\stochess{x,z}-\hess 
 	F(x)}\le\bar{\sigma}_2$ a.s.; we let $\stocpOracleas$ denote 
 	the  class of such bounded noise oracles. Under this assumption, 
 \pref{alg:epsilon_gamma_hvp}---whose description is deferred to the \pref{app:upper_eg}---enjoys the following convergence 
 guarantee.\footnote{The notation $\Otilde(\cdot)$ hides lower-order terms and logarithmic dependence on the dimension $d$. See the proof in \pref{app:upper_eg} for the complete description of the algorithm and the full
  complexity bound, including lower order terms.}    

\newcommand{\sigt}{\tilde{\sigma}_2}
\begin{theorem}
\label{thm:epsilon_gamma_hvp} For any function $F \in \cF_2(\Delta, L_{1:2})$, stochastic Hessian-vector product oracle in $\stocpOracleas$, 
$\epsilon \leq \min\crl*{\sigma_1, \sqrt{\Delta L_1}}$, and $\gamma \leq \min\crl*{\bar{\sigma}_2, L_1, 
\sqrt{\epsilon{}L_2}}$, with probability at least $\frac{5}{8}$ \pref{alg:epsilon_gamma_hvp} returns a point $\xhat$ such that  
\[
  \nrm*{\grad \func{\xhat}} \leq \epsilon \quad\text{ 
   and} \quad\lambda_{\min{}} \prn*{\grad^2 \func{\xhat}} \geq -\gamma,
\] and performs at most
 \begin{align*}
\widetilde{O} \prn*{\frac{  \Delta \sigma_1 \bsigma_2}{\epsilon^3}  +  
\frac{\Delta L_2 \sigma_1\bsigma_2}{\gamma^2 \epsilon^2} +  
\frac{\Delta L^2_2 \prn*{\bsigma_2 + L_1}^2}{\gamma^5} + \frac{\Delta 
L_1}{\epsilon^2}}
\end{align*}
 stochastic gradient and Hessian-vector product queries.
\end{theorem}

Similar to the case for finding $\eps$-stationary points (see 
discussion preceding \pref{thm:epsilon_cubic}), using full stochastic 
Hessian information allows us to design an algorithm (\pref{alg:eg_cubic})
which removes the dependence on $L_1$ from the theorem above. Moreover, estimating negative curvature directly from empirical Hessian 
estimates saves us the need to use Oja's method, which means that we do 
not need the additional boundedness assumption on the stochastic 
Hessian used by \pref{alg:epsilon_gamma_hvp}. \arxiv{We defer the complete description and
analysis for \pref{alg:eg_cubic} to 
\pref{app:eg_cubic_upper}, and state its complexity guarantee
below.}\colt{We defer the complete description, complexity guarantee,
and for analysis for \pref{alg:eg_cubic} to 
\pref{app:eg_cubic_upper}.}

\arxiv{
\begin{theorem}
\label{thm:eg_cubic_upper} For any function $F \in \cF_2(\Delta, \infty, 
L_2)$, stochastic second order oracle in $\cO_{2}(F, \sigma_1, \sigma_2)$, 
$\epsilon \leq \sigma_1$, and $\gamma \leq \min\crl[\big]{\sigma_2, 
\sqrt{\epsilon{}L_2}, \Delta^{\frac{1}{3}} L_2^\frac{2}{3}}$, with probability at least $\frac{3}{5}$ \pref{alg:eg_cubic} returns a point $\xhat$ such that  
   \[
   \lambda_{\min{}} \prn*{\grad^2 \func{\xhat}} \geq -\gamma, \quad\text{ 
   and }\quad\nrm*{\grad \func{\xhat}} \leq \epsilon,
\]
 and performs at most
 \begin{align*}
\widetilde{O} \prn*{\frac{  \Delta \sigma_1 \sigma_2}{\epsilon^3}  +  
\frac{\Delta L_2 \sigma_1\sigma_2}{\gamma^2 \epsilon^2} +  \frac{\Delta 
L^2_2 \sigma_2^2}{\gamma^5}}
\end{align*}
 stochastic gradient and Hessian queries.
\end{theorem} 
}

\arxiv{\subsection{Lower bounds}}
\colt{\subsection{Lower bounds}}
\label{sec:epsilon_gamma_lower}

	\newcommand{\dsunscaled}{\Sigma_{2T}}
        We now develop lower complexity bounds for the 
        task of finding $(\eps,\gamma)$-stationary points. To do so,
        we prove new lower bounds for the simpler sub-problem of finding a \emph{$\gamma$-weakly convex} point, i.e., a point $x$ such 
	that $\lambdamin(\grad^2 F(x))\ge-\gamma$ (with no restriction on 
	$\norm{\grad F(x)}$). Lower bounds for finding $(\eps,\gamma)$-SOSPs 
	 follow as the maximum (or, equivalently, the sum) of lower 
	bounds we develop here and the lower bounds for finding $\veps$-stationary 
	points given in~\pref{thm:gamma_so_zero_respecting}. To see why this is 
	so, let 
	$F_\eps$ and $F_\gamma$ be hard instances for finding 
	$\epsilon$-stationary and $\gamma$-weakly-convex points 
	respectively, and consider the ``direct sum''
	$F_{\epsilon,\gamma}(x)\ldef\frac{1}{2}F_\epsilon(x_1, \ldots, x_d) + 
	\frac{1}{2}F_\gamma(x_{d+1},\ldots,x_{2d})$; this is a hard instance for 
	finding $(\epsilon,\gamma)$-SOSPs that inherits all the regularity 
	properties of its constituent functions.

	The basic construction we use here is a modification of the zero-chain 
	introduced in~\citet{carmon2019lower_i} (see
        \pref{eq:hard_function_eps_LB} in \pref{app:lower}) 
	in 	which large 
	$\lambdamin(\grad^2 F(x))$ 
	is possible only when essentially none of 
	the entries of $x$ is zero. Given $T > 0$, we define the hard function
	\begin{equation}
	\barF(x) \defeq \Psi(1)\Phibar(x_1) + \sum_{i=2}^T ~ \brk*{ 
	\Psi(-x_{i-1}) 
        \Phibar(-x_i) + \Psi(x_{i-1}) \Phibar(x_i)) },
      	\label{eq:hard_function_gamma_LB}
      \end{equation}
	where $\Psi(x)\ldef\exp(1-\frac{1}{(2x-1)^2})\indicator{x>\frac{1}{2}}$ 
	(as in~\citet{carmon2019lower_i}) and $\Phibar(x)\defeq  8\prn{ 
	e^{\frac{-x^2}{2}}-1}$.

	Our design for the function $\Phibar$ guarantees that any query whose last coordinate 
	is zero has significant negative curvature, while maintaining the original 
	chain structure which guarantees that zero-respecting algorithms require 
	many queries before ``discovering'' the last coordinate. We complete the 
	construction by specifying a collection of stochastic derivative estimators similar to 
	those in~\pref{sec:epsilon_gamma_lower}, except for that we
        choose the stochastic gradient estimator $\wh{\grad 
	\barFunscaled}$ to be exactly equal to $\grad \barFunscaled$, so
      that the lower bound holds even for $\sigma_1=0$; 
	Appropriately scaling $\barFunscaled$ allows us to tune the 
	Lipschitz constants of its derivatives and the variance of the
	estimators, thereby	establishing the following complexity bounds (see 
	\pref{app:gamma_lower} for a full derivation).
		\begin{theorem} 
			\label{thm:gamma_so_zero_respecting}
		Let $p\ge 2$ and $\Delta, \Lip{1:p}, \sigma_{1:p}>0$ be fixed. 
		 If $\gamma \le 
		O(\min\{\sigma_2,L_1\})$, then there exists ${F \in \cF_p\prn*{\Delta, 
		L_{1:p}}}$ 
		and 
		$(\estimator[p], P_z) \in \stocpOracle[p]$ such that for any 
		stochastic \pth-order zero-respecting algorithm, the number of 
		queries to $\estimator[p]$required to obtain a 
		$\gamma$-weakly convex point with constant probability
                is at least
		\begin{align}\label{eq:minimax_gamma_p}
			\Omega(1)\cdot
			\begin{cases}
			\frac{\Delta \sigma_2^2 L_2^2}{\gamma^5}, & p=2, \\
			\frac{\Delta \sigma_2^2
			}{ \gamma^3 }
			\min\crl*{
				\min_{q\in\{3,\ldots,p\}}\prn*{\frac{\sigma_q}{\sigma_2}}^{\frac{2}{q-2}},
				\min_{q'\in\{2,\ldots,p\}}\prn*{\frac{\Lip{q'}}{\gamma}}^{\frac{2}{q'-1}}},
			& p>2.\\
			\end{cases}
		\end{align}
		\arxiv{A construction of dimension 
		$\Theta\Bigl(\frac{\Delta}{\gamma}  
				\min\Bigl\{
					\min_{q\in\{3,\ldots,p\}}\prn*{\frac{\sigma_q}{\sigma_2}}^{\frac{2}{q-2}},
					\min_{q'\in\{2,\ldots,p\}}\prn*{\frac{\Lip{q'}}{\gamma}}^{\frac{2}{q'-1}}\Bigr\}\Bigr)$
			 realizes the lower bound.}
	\end{theorem}
        \arxiv{
        \pref{thm:gamma_so_zero_respecting} is new even in the
        noiseless case  (in which $\sigma_{1} = \dots = 	
        \sigma_p=0$), where it specializes to
        \begin{equation}\label{eq:minimax_gamma_p_det}
          \Omega(1)\cdot\frac{\Delta 
          }{ \gamma }
          \min_{q\in\{2,\ldots,p\}}\prn*{\frac{\Lip{q}}{\gamma}}^{\frac{2}{q-1}}.
        \end{equation}
For the class $\cF_p(\Delta,L_p)$, the lower bound 
\pref{eq:minimax_gamma_p_det} further simplifies to
$\Delta{}L_p^{\frac{2}{p-1}}\gamma^{-\frac{p+1}{p-1}}$, which is attained by 
the \pth-order regularization method given in \citet[Theorem 
3.6]{cartis2017improved}.
}
\colt{
        \pref{thm:gamma_so_zero_respecting} is new even in the
        noiseless case  (in which $\sigma_{1} = \dots = 	
        \sigma_p=0$), where it specializes to
        $\frac{\Delta 
          }{ \gamma }
          \min_{q\in\{2,\ldots,p\}}\prn*{\frac{\Lip{q}}{\gamma}}^{\frac{2}{q-1}}$. For
          the class $\cF_p(\Delta,L_p)$, this further simplifies to
$\Delta{}L_p^{\frac{2}{p-1}}\gamma^{-\frac{p+1}{p-1}}$, which is attained by 
the \pth-order regularization method given in \citet[Theorem 
3.6]{cartis2017improved}.
}
Together, these results characterize the 
deterministic complexity of finding $\gamma$-weakly convex points with 
noiseless \pth-order methods.
	
Returning to the stochastic setting, the bound in \pref{thm:gamma_so_zero_respecting}, when combined with \pref{thm:eps_so_zero_respecting}, implies the 
following oracle complexity lower bound bound for finding $(\veps,\gamma)$-SOSPs with 
	zero-respecting stochastic second-order methods ($p=2$):
	\begin{align}
		\Omega(1)\cdot\prn*{ 
	\min\crl*{ 				
		\frac{\Delta\sigma_1\sigma_2}{\epsilon^3}, 
		\frac{\Delta L_2^{0.5}\sigma_1}{\epsilon^{3.5}},
		\frac{\Delta L_1 \sigma_1^2}{\epsilon^{4}}
	} + \frac{\Delta \sigma_2^2\Lip{2}^2	}{\gamma^5}}.
	\end{align}
	Our lower bound matches the $\epsilon^{-3} + \gamma^{-5}$ terms in 
	the upper bound given by~\pref{thm:epsilon_gamma_hvp}, but does not 
	match the mixed term $\eps^{-2}\gamma^{-2}$ appearing in the
        upper bound.\footnote{Young's 
	inequality only gives $\epsilon^{-3}+\gamma^{-5} \ge 
	\Omega(\epsilon^{-9/5}\gamma^{-2})$.} Overall, the rates match whenever
	$\gamma = \Omega(\epsilon^{0.5})$ or $\gamma =O(\epsilon^{2/3})$.

	\pref{thm:gamma_so_zero_respecting} is suggestive of another ``elbow'' phenomenon:
	In the stochastic regime, the rate does not improve beyond
        $\gamma^{-3}$ for $p\ge3$, while the optimal rate in the
        noiseless regime, $\gamma^{-\frac{p+1}{p-1}}$, continues
        improving for all $p$.\footnote{Indeed, when 
	high-order noise moments are assumed finite, the term	
	$\min_{q\in\{3,\ldots,p\}}\prn*{{\sigma_q}/{\sigma_2}}^{\frac{2}{q-2}}$ 
	can longer be disregarded. This, in turn, implies that for sufficiently 
	small $\gamma$, one cannot improve over $\gamma^{-3}$-scaling, as seen 
	by~\pref{eq:minimax_gamma_p}.} However, we are not yet aware of an
        algorithm using stochastic third-order information or higher
        that can achieve the $\gamma^{-3}$ complexity bound.

\arxiv{
\section{Conclusion} \label{sec:conclusion}
This paper provides a fairly complete picture of the worst-case oracle 
complexity of finding stationary points with a stochastic second-order 
oracle: for $\epsilon$-stationary points we characterize the leading term in 
$\eps^{-1}$ exactly and for ($\epsilon,\gamma$)-SOSPs we characterize the 
leading term in $\gamma^{-1}$ for a wide range of parameters. 
Nevertheless, our results point to a number of open questions.

\paragraph{Benefits of higher-order information for $\gamma$-weakly convex 
points.} 

Our upper and lower bounds (in \pref{thm:eg_cubic_upper} and 
\pref{thm:gamma_so_zero_respecting}) resolve the optimal rate to find an 
$(\epsilon, \gamma)$-stationary point for $p=2$, i.e., when $F$ is 
second-order smooth and the algorithm can query stochastic gradient and 
Hessian information. Furthermore, \pref{thm:eps_so_zero_respecting} 
shows that higher order information ($p \geq 3$) cannot improve the 
dependence of the rate on the first-order stationarity parameter $\epsilon$. 
However, our lower bound for dependence on $\gamma$ scales as 
$\gamma^{-5}$ for $p= 2$, but scales as $\gamma^{-3}$ for $p \geq
3$. The weaker lower bound for $p \geq 3$ leaves 
open the possibility of a stronger upper bound using third-order 
information or higher. 

\paragraph{Global methods.} For statistical learning and sample average 
approximation problems, it is natural to consider problem instances of the 
form $F(x)=\En\brk[\big]{\wh{F}(x,z)}$. For this setting, a more powerful 
oracle 
model is the \emph{global oracle}, in which samples 
$z\ind{1},\ldots,z\ind{n}$ are 
drawn i.i.d. and the learner observes the entire function 
$\wh{F}(\cdot,z\ind{t})$ 
for 
each $t\in[n]$. Global oracles are more powerful than stochastic \pth order 
oracles for every $p$, and lead to improved rates in the
convex setting \citep{foster2019complexity}. Is it possible to beat the $\eps^{-3}$ elbow for such 
oracles, or do our lower bounds extend to this setting?

\paragraph{Adaptivity and instance-dependent complexity.} 
Our lower bounds show that stochastic higher-order methods cannot 
improve the $\eps^{-3}$ oracle complexity attained with stochastic 
gradients and Hessian-vector products. Furthermore, in the multi-point
query model, stochastic second-order information does not even lead to
improved rates over stochastic first-order information. However, these 
conclusions could be artifacts of our worst-case point of view---are there 
natural families of problem instances for which higher-order methods can 
adapt to additional problem structure and obtain stronger 
instance-dependent convergence guarantees? Developing a theory of 
instance-dependent complexity that can distinguish adaptive algorithms 
stands out as an exciting research prospect.

 }

\colt{
\subsection*{Discussion}
Due to space constraints, we defer conclusions and discussion to \pref{app:conclusion}.}

\newpage
\subsection*{Acknowledgements} 
We thank Blake Woodworth and Nati Srebo for helpful discussions. YA 
acknowledges partial support from the Sloan Foundation and Samsung 
Research. JCD acknowledges support from the NSF CAREER award 
CCF-1553086, ONR YIP N00014-19-2288, Sloan Foundation, NSF HDR 
1934578 (Stanford Data Science Collaboratory), and the DAWN Consortium.
DF acknowledges the support of TRIPODS award 1740751.  KS 
acknowledges support from NSF CAREER Award 1750575 and a Sloan 
Research Fellowship.

\setlength{\bibsep}{6pt}

\newpage 

\appendix 

\renewcommand{\contentsname}{Contents of Appendix}
\tableofcontents
\addtocontents{toc}{\protect\setcounter{tocdepth}{3}}
\clearpage

\colt{\section{Further discussion} \label{app:conclusion}
This paper provides a fairly complete picture of the worst-case oracle 
complexity of finding stationary points with a stochastic second-order 
oracle: for $\epsilon$-stationary points we characterize the leading term in 
$\eps^{-1}$ exactly and for ($\epsilon,\gamma$)-SOSPs we characterize the 
leading term in $\gamma^{-1}$ for a wide range of parameters. 
Nevertheless, our results point to a number of open questions.

\paragraph{Benefits of higher-order information for $\gamma$-weakly convex 
points.} 

Our upper and lower bounds (in \pref{thm:eg_cubic_upper} and 
\pref{thm:gamma_so_zero_respecting}) resolve the optimal rate to find an 
$(\epsilon, \gamma)$-stationary point for $p=2$, i.e., when $F$ is 
second-order smooth and the algorithm can query stochastic gradient and 
Hessian information. Furthermore, \pref{thm:eps_so_zero_respecting} 
shows that higher order information ($p \geq 3$) cannot improve the 
dependence of the rate on the first-order stationarity parameter $\epsilon$. 
However, our lower bound for dependence on $\gamma$ scales as 
$\gamma^{-5}$ for $p= 2$, but scales as $\gamma^{-3}$ for $p \geq
3$. The weaker lower bound for $p \geq 3$ leaves 
open the possibility of a stronger upper bound using third-order 
information or higher. 

\paragraph{Global methods.} For statistical learning and sample average 
approximation problems, it is natural to consider problem instances of the 
form $F(x)=\En\brk[\big]{\wh{F}(x,z)}$. For this setting, a more powerful 
oracle 
model is the \emph{global oracle}, in which samples 
$z\ind{1},\ldots,z\ind{n}$ are 
drawn i.i.d. and the learner observes the entire function 
$\wh{F}(\cdot,z\ind{t})$ 
for 
each $t\in[n]$. Global oracles are more powerful than stochastic \pth order 
oracles for every $p$, and lead to improved rates in the
convex setting \citep{foster2019complexity}. Is it possible to beat the $\eps^{-3}$ elbow for such 
oracles, or do our lower bounds extend to this setting?

\paragraph{Adaptivity and instance-dependent complexity.} 
Our lower bounds show that stochastic higher-order methods cannot 
improve the $\eps^{-3}$ oracle complexity attained with stochastic 
gradients and Hessian-vector products. Furthermore, in the multi-point
query model, stochastic second-order information does not even lead to
improved rates over stochastic first-order information. However, these 
conclusions could be artifacts of our worst-case point of view---are there 
natural families of problem instances for which higher-order methods can 
adapt to additional problem structure and obtain stronger 
instance-dependent convergence guarantees? Developing a theory of 
instance-dependent complexity that can distinguish adaptive algorithms 
stands out as an exciting research prospect.

}

\section{Detailed comparison with existing rates}
\label{app:comparison}
\colt{\pref{table:summary_detailed} provides a detailed comparison between our 
upper bounds on the complexity of finding $\epsilon$-stationary points
and those of prior work.}

\begin{table}[h!]
	\begin{center}
          \renewcommand{\arraystretch}{1.3}
         \begin{tabular}{p{5.7cm} p{.9cm} p{4.5cm} p{3.5cm}}
		\toprule
                  \makecell[l]{Method} & \makecell[l]{Uses \\
		$\stochessNA$?} & \makecell[l]{Complexity bound} & 
                                                    \makecell[l]{Additional\\ assumptions} \\
		\midrule
                  \makecell[lt]{SGD~\\\citep{ghadimi2013stochastic}~} & No 
		& 
		$O(\Delta L_1 \sigma_1^2 \epsilon^{-4})$ &\\
                  \makecell[lt]{Restarted SGD~\\\citep{fang2019sharp}} & No & 
                  $O(\Delta L_2^{0.5} 
		\sigma_1^2 \eps^{-3.5})^{\dag}$ & \makecell[tl]{$\stocgradNA$ Lipschitz\\
                  almost surely} \\
                  \makecell[lt]{Normalized 
                  SGD~\\\citep{cutkosky2020momentum}} 
                  & No & 
$O(\Delta L_2^{0.5} 
                                                                                                  \sigma_1^2 \eps^{-3.5})^{\dag}$                                                                                                  &
                  \\
		          \makecell[lt]{Subsampled regularized\\
		Newton~\citep{tripuraneni2018stochastic}} & Yes$^*$ & $O(\Delta 
		L_2^{0.5} 
		\sigma_1^2 \eps^{-3.5})^{\dag}$ & \\
                  \makecell[lt]{Recursive variance 
                  \\reduction~\citep[e.g.,][]{fang2018spider}} & No &
		$O(\Delta \sigma_1 \sigmss \epsilon^{-3}+ \Delta L_1 
		\eps^{-2})$ 
                                                                 &\makecell[tl]{Mean-squared
                  \\smoothness  $\sigmss\leq\sigma_2$,\\simultaneous queries \\
		(\pref{app:oracle})}
		\\ 
		\midrule
		\makecell[lt]{SGD with \mainalg\\({\pref{alg:epsilon_hvp}})} & Yes$^*$ & 
		${O(\Delta \sigma_1 \sigma_2 
		\eps^{-3} + \Delta L_2^{0.5} \sigma_1 \eps^{-2.5}+\Delta L_1 
		\eps^{-2})}$& \\
           \makecell[lt]{Subsampled Newton with\\ \mainalg ({\pref{alg:epsilon_cubic}})} & Yes & 
		${O(\Delta \sigma_1 \sigma_2 
		\eps^{-3} + \Delta L_2^{0.5} \sigma_1 \eps^{-2.5}+\Delta \sigma_2  
		\eps^{-2})}$& \\
		\bottomrule 
	\end{tabular}
\end{center}
	\caption{Detailed comparison of guarantees for finding $\eps$-stationary points 
	(satisfying $\E \norm{\grad F(x)}\le \eps$) for a function $F$
        with $L_1$-Lipschitz gradients and $L_2$-Lipschitz Hessian. Here $\Delta$ is the initial 
	optimality gap, and $\sigma_p$ is the variance of $\widehat{\grad^p F}$. Algorithms
        marked with $^{*}$ require only stochastic Hessian-vector products. Complexity 
	bounds marked with $^\dag$  only 
	show  
	leading order term in $\epsilon$.}\label{table:summary_detailed}
\end{table}

\section{\mbox{Comparison: multi-point queries and mean-squared smoothness}}\label{app:oracle}
Stochastic first-order methods that utilize variance
reduction \citep{lei2017non,fang2018spider,zhou2018stochastic}
 employ the following \emph{mean-squared smoothness} (MSS) 
assumption on the stochastic gradient estimator:
\begin{equation*}\label{eq:mss}
\E \, \norm{ \stocgrad{x,z} - \stocgrad{y,z} }^2 
\le 
\LipGradBar^2\norm{x-y}^2~~\mbox{for all}~~x,y\in\R^d.
\end{equation*}
Since $\E \brk{\stocgrad{x,z}} = \grad F(x)$, this is equivalent to assuming 
\begin{equation}\label{eq:mss-var}
\E \, \norm{ \stocgrad{x,z} - \stocgrad{y,z} - \prn{\grad F(x) - \grad 
F(y)}}^2 
\le 
\sigmss^2\norm{x-y}^2~~\mbox{for all}~~x,y\in\R^d,
\end{equation}
for some $\sigmss < \LipGradBar$. In fact, while it always holds that
$\LipGradBar^2\leq{}L_1^2+\sigmss^2$, inspection of the results 
of~\citet{fang2018spider,wang2018spiderboost} shows one can replace 
$\LipGradBar$ with $\sigmss$ in the leading terms of their complexity 
bounds without any change to the algorithms. 

Algorithms that take advantage of the MSS structure rely on the following 
additional \emph{simultaneous query} assumption (which is a special
case of \pref{eq:n_point_query_model} for $n=2$):
\begin{equation}\label{eq:2-query}
\mbox{\emph{We may query $x,y\in\R^d$ and observe  $\estimator[1](x,z)$ 
and $\estimator[1](y,z)$ for the same draw of $z\sim \Pz$.}}
\end{equation}
In empirical risk minimization problems, $z$ represents the datapoint 
index and possibly data augmentation parameters, and the value of $z$ is 
typically part of the query, which means that assumption~\pref{eq:2-query} 
indeed holds. In certain online learning settings, however, the assumption can fail. 
For example, the variable $z$ could represent the instantaneous power 
demands in an electric grid, and testing two grid configurations for the same 
grid state might be impractical.

We observe that assuming access to both an MSS gradient estimator and 
simultaneous two-point queries is stronger than assuming a bounded 
variance stochastic Hessian-vector product estimator. This holds because the former 
allows us to simulate the latter with finite differencing. Formally, we have 
the following.
\begin{observation}\label{obs:mss-stronger}
	Let $F$ have $L_2$-Lipschitz Hessian, let $\stocgradNA$ 
	satisfy~\pref{eq:mss-var}, and assume we have access to a
        two-point query oracle as in \pref{eq:2-query}. Then, for 
	any $\delta > 0$ and every unit-norm vector $u$, the Hessian-vector 
	product estimator 
	\begin{equation}\label{eq:hess-from-mss}
	\stochessNA_\delta(x,z) u \defeq \frac{1}{\delta} \brk*{
		\stocgrad{x+\delta \cdot u, z} - \stocgrad{x , z}}
	\end{equation}
	satisfies
	\begin{equation*}
	\nrm*{ \E \brk{\stochessNA_\delta(x,z) u} - \hess F(x)u } \le \frac{L_2 
	\delta}{2}
	~~\mbox{and}~~
	\E \, \nrm*{ \stochessNA_\delta(x,z) u -\hess F(x)u}^2 \le \sigmss^2 
	+ 
	\frac{L_2^2 \delta^2}{4}.
	\end{equation*}
\end{observation}
\begin{proof}
	We have $\E \brk{\stochessNA_\delta(x,z) u} = \frac{1}{\delta}\brk{\grad 
	F(x+\delta\cdot u) - \grad F(x)}$, and by Lipschitz continuity of $\hess F$,
	\begin{equation*}
	\norm{ \grad F(x+\delta\cdot u) - \grad F(x) - \hess F(x) [\delta u]} \le 
	\frac{L_2}{2}\delta^2 \norm{u}^2 = \frac{L_2}{2}\delta^2,
	\end{equation*}
which implies the bound on the bias. To bound the variance, we note that
	\begin{flalign*}
	&\E \, \nrm*{  \stochessNA_\delta(x,z) u - \E\brk{ \stochessNA_\delta(x,z) 
	u}}^2\\& \quad
	\le \frac{1}{\delta^2}  \E\,  \nrm*{  
	\stocgrad{x+\delta u, z} - 
	\stocgrad{x,z} - \brk{\grad F(x+\delta u) - \grad F(x)}}^2 \le 
	\frac{1}{\delta^2} \cdot \sigma_2^2\norm{\delta u}^2 =\sigmss^2,
	\end{flalign*}
	by the MSS property~\pref{eq:mss-var}. 
\end{proof}
We conclude from \pref{obs:mss-stronger} that
\pref{alg:epsilon_hvp}, which only requires stochastic Hessian-vector products,
attains $O(\eps^{-3})$ complexity under assumptions no stronger than
previous algorithms.
In fact, we show now that our assumptions are in fact strictly weaker
than prior work. That is, while an MSS gradient estimator implies a bounded variance Hessian 
estimator, the opposite is not true in general. This is simply due to the fact 
that in our oracle model, $\stocgradNA$ and $\stochessNA$ can be completely unrelated. 
Consider for example the case where $\Pz$ is uniform on $\{-1,1\}$ and
\begin{equation*}
\stocgrad{x,z} = \begin{cases}
\grad F(x) + \frac{x}{\norm{x}} z & x\ne 0\\
\grad F(x)  & x=0,\\
\end{cases}
~~\mbox{while}~~\stochess{x,z}=\hess F(x).
\end{equation*}
Clearly $\stocgradNA$ is not MSS, even though $\stochessNA$ has zero 
variance. 

There is, however, an important setting where bounded variance for 
$\stochessNA$ \emph{does} imply that $\stocgradNA$ is MSS. Suppose that
the derivative of $\stocgradNA(x,z)$ exists, and has the form
\begin{equation}\label{eq:weak-stat-learning}
\grad [\stocgrad{x,z}] = \stochess{x,z}.
\end{equation}
That is, the Hessian estimator is the Jacobian of the gradient 
estimator. In this case, bounded variance for the Hessian estimator
implies mean-squared smoothness.
\begin{observation}\label{obs:mss-equiv}
	Let $F$ have gradient and Hessian estimators $\stocgradNA$ and 
	$\stochessNA$ satisfying~\pref{eq:stoc-hess} 
	and~\pref{eq:weak-stat-learning}. Then $\stocgradNA$ has the MSS 
	property~\pref{eq:mss-var} with $\sigmss\leq\sigma_2$.
\end{observation} 
\begin{proof} 
	Under the property~\pref{eq:weak-stat-learning}, we have
	\begin{flalign*}
	&\stocgrad{x,z} - \stocgrad{y,z} - \brk{\grad F(x) - \grad 
		F(y)} \\&\quad\quad\quad
		= \int_{0}^1 \prn*{\stochess{x t + y(1-t), z} - \hess F (x t + y(1-t))}(x-y)
		dt.
	\end{flalign*}
	Taking the squared norm, applying Jensen's inequality, and substituting 
	the variance bound~\pref{eq:stoc-hess} gives the MSS 
	property~\pref{eq:mss-var}.
\end{proof}
 The property~\pref{eq:weak-stat-learning} holds for empirical risk 
 minimization, where we have the more general relation $\stocder{x,z}{p} = 
 \grad^p \wh{F}(x,z)$ for any $p$; That is, all the stochastic derivative
 estimators are themselves the derivatives of a single stochastic function. 
 Therefore, by \pref{obs:mss-stronger} and \pref{obs:mss-equiv}, in empirical risk 
 minimization 
 settings, mean-square smoothness is essentially equivalent to bounded 
 variance of the stochastic Hessian estimator.

\section{Variance-reduced gradient estimator (\mainalg)} \label{app:estimators}
In this section we prove \pref{lem:base-var-bound}. First, we formally 
describe the protocol in which our optimization algorithms query the 
gradient estimator $\getgradientfunc$ described in 
\pref{alg:gradient_estimator}, and define some additional notation. 

Given a function $F \in \cF_2(\Delta, L_1, L_2)$ and a stochastic 
second-order oracle in $\stocpOracle[2]$, the optimization algorithm 
interacts with $\getgradientfunc$ by sequentially querying points 
$\crl*{x\ind{t}}_{t=1}^\infty$ with reset probabilities 
$\crl*{\bias\ind{t}}_{t=1}^\infty$, to obtain estimates $\gradest[t]$ for 
$\grad \func{x\ind{t}}$ for each time $t$; that is,
\begin{flalign}
&x\ind{t}=\mathsf{A}\ind{t}(g\ind{0}, g\ind{1},\ldots,g\ind{t-1};r\ind{t-1}),
~\bias\ind{t}=\mathsf{B}\ind{t}({r}\ind{t-1}),
~~\mbox{and}~~\nonumber \\&
g\ind{t}=\getgradientfunc_{\epsilon, 
\bias\ind{t}}(x\ind{t},x\ind{t-1},g\ind{t-1}),
\label{eq:lemma1_protocol}
\end{flalign}
where $\mathsf{A}\ind{t}, \mathsf{B}\ind{t}$ are measurable mappings
modeling the optimization algorithm and  
$\{r\ind{t}\}$ is an independent 
sequence of random seeds.\footnote{This level of formalism is not used
  within the proof, but we include it here for clarity.} That is, 
\pref{lem:base-var-bound} holds for any sequence of queries where 
$x\ind{t}$, and $\bias\ind{t}$ are adapted to the filtration
\[
\cG\ind{t} = 
\sigma\prn*{\{g\ind{j}, r\ind{j}\}_{j<t}},
\]
but $\bias\ind{t}$ is independent of $\cG\ind{t-1}$ and $\gradest[t-1]$.

\pref{lem:base-var-bound} is an immediate consequence of  
\pref{lem:gradient_estimator_general_lemma} and 
\pref{lem:gradient_estimator_oracle_complexity_meta_lemma}, proven 
below, which respectively establish the estimator's error and complexity 
bounds. 

\begin{lemma} \label{lem:gradient_estimator_general_lemma} Given a 
function  $F \in \cF_2(\Delta, \infty, L_2)$, a stochastic oracle in 
$\stocpOracle[2]$, and initial points $x\ind{0}$ and $g\ind{0} 
= \bot$, let $\crl{g\ind{t}}_{t\geq0}$ denote the sequence of gradient 
estimates at $\crl{ x\ind{t} }_{ t \geq 0}$ respectively, returned by 
$\textnormal{\getgradientfunc}$ under the protocol \pref{eq:lemma1_protocol}. Then, for all $t \geq 1$, 
\begin{align*} 
\En{}{ \nrm[\big]{ \gradestc{t} - \derF{x\ind{t}}{}}^2} \leq \epsilon^2.   
\end{align*}
\end{lemma}  
\begin{proof} 
	We prove that 
	\begin{equation*}
	\En{}{ \nrm[\big]{ \gradestc{t} - \derF{x\ind{t}}{}}^2} \leq 
	\prn*{1-\frac{\E\brk{\bias\ind{t}}}{2}}\En{}{ \nrm[\big]{ \gradestc{t-1} - 
	\derF{x\ind{t-1}}{}}^2}  + \frac{\E\brk{\bias\ind{t}}}{2} \epsilon^2,
	\end{equation*}
	whence the result follows by a simple induction whose basis is
	\begin{equation*}
	\En{}{ \nrm[\big]{ \gradestc{1} - \derF{x\ind{1}}{}}^2}
	\le \frac{\sigma_1^2}{n} \le \epsilon^2.
	\end{equation*}
	
	\newcommand{\err}{\mathfrak{e}}
	
	Let $\coin\ind{t}$ denote the value of the coin toss in the $t^\text{th}$ call to  
	\pref{alg:gradient_estimator} (\pref{line:coin-toss}), recalling that 
	$\coin\ind{t}\sim\bern(\bias\ind{t})$. Writing $\err\ind{t} = \gradestc{t} 
	- \derF{x\ind{t}}{}$ for brevity, we have
	\begin{equation}
	\En{}\brk*{ \nrm[\big]{ \err\ind{t} }^2 ~\Big\vert~ 
	\bias\ind{t}} = \bias\ind{t} \En{}\brk*{ \nrm[\big]{ \err\ind{t}}^2 
	~\Big\vert~ \coin\ind{t}=1} +
	 (1-\bias\ind{t}) 
	\En{}\brk*{ \nrm[\big]{ \err\ind{t}}^2 ~\Big\vert~
	\coin\ind{t}=0}.\label{eq:coin-errors}
	\end{equation}
	Clearly,
	\begin{equation}
	\En{}\brk*{ \nrm[\big]{ \err\ind{t}}^2 
		~\Big\vert~ \coin\ind{t}=1} \le \frac{\sigma_1^2}{n} = 
		\frac{\epsilon^2}{5}.\label{eq:coin1-error}
	\end{equation}
	Moreover, conditional on $\coin\ind{t}=0$, we have from the definition 
	of 
	the gradient estimator that
	\begin{equation*}
	\err\ind{t} = \err\ind{t-1} + \psi\ind{t},
	\end{equation*}
	where
	\begin{equation*}
	\psi\ind{t}\defeq 
	\sum_{k=1}^{\numbabysteps\ind{t}} 
	\stocder{x\ind{t,k-1}, z\ind{t, k}}{2} \prn*{ x\ind{t,k} - 
		x\ind{t,k-1}} - \grad F(x\ind{t}) + \grad F(x\ind{t-1}),
	\end{equation*}
	and
	\begin{equation}\label{eq:key-scaling}
	\numbabysteps\ind{t} = \ceil*{\frac{5 \prn*{\sigma_2^2 + 
			L_2\epsilon}}{\bias\ind{t} \epsilon^2}\cdot 
			\norm{x\ind{t}-x\ind{t-1}}^2},
                    \end{equation}
                    where $x\ind{t,k}$ and $x\ind{t,k}$ respectively
                    denote the values of $x\ind{k}$ and $z\ind{k}$
                    (defined on \pref{line:hvp_update_step}) during the \tth
                    call to \pref{alg:gradient_estimator}.
	
	\newcommand{\expectedErr}[1]{\E \brk[\big]{ \psi\ind{#1} ~\big\vert~ 
			\cG\ind{#1}}}
	
	We may therefore decompose the error conditional on $\coin\ind{t}=0$ 
	as
	\begin{flalign*}
	&\En{}\brk*{ \nrm[\big]{ \err\ind{t}}^2 
		~\Big\vert~ \coin\ind{t}=0} 
	\overeq{\proman{1}}
	\En{}\nrm[\big]{ \err\ind{t-1} + \expectedErr{t} }^2 
	+ \En{} \nrm[\big]{ \psi\ind{t} - \expectedErr{t} }^2
	\\& \quad
	\overleq{\proman{2}}
	\En{}\brk[\bigg]{ \prn[\bigg]{1+\frac{\bias\ind{t}}{2}}\nrm[\big]{ \err\ind{t-1} }^2 }
	+ \En{}\brk*{ \prn*{1+\frac{2}{\bias\ind{t}}}\nrm[\big]{ \expectedErr{t} 
	}^2 }
	+ \En{} \nrm[\big]{ \psi\ind{t} - \expectedErr{t} }^2,
	\label{eq:coin0-error}
	\numberthis
	\end{flalign*}
	where $\proman{1}$ is due to $\err\ind{t-1}\in\cG\ind{t}$ and 
	$\proman{2}$ is due to Young's inequality.

	The facts that $z\ind{t,k}$ is independent from $\cG\ind{t}$,  that 
	$\grad F(x\ind{t})-\grad F(x\ind{t-1}) \in \cG\ind{t}$, and that 
	$\stocder{\cdot}{2}$ is unbiased give
	\begin{equation*}
	\E \brk*{ \psi\ind{t} ~\Big\vert~ \cG\ind{t}} =  
	\sum_{k=1}^{\numbabysteps\ind{t}} 
	\hess F (x\ind{t,k-1})\prn*{ x\ind{t,k} - 	x\ind{t,k-1}} - \grad F(x\ind{t}) 
	+ \grad F(x\ind{t-1}) 
	\end{equation*}
	for every $t$. Consequently, the scaling \pref{eq:key-scaling} and  
	Hessian 
	estimator variance bound imply
	\begin{flalign}
	&\E \brk*{ \norm[\Big]{ \psi\ind{t} - \E \brk*{ \psi\ind{t} ~\big\vert~ 
	\cG\ind{t}}}^2 ~\Big\vert~\cG\ind{t}
	}
	\nonumber \\& \qquad
	\overeq{(\star)} \frac{1}{(\numbabysteps\ind{t})^2} 
	\sum_{k=1}^{\numbabysteps\ind{t}} 
	\En \brk*{\nrm*{\prn{ \stocder{x\ind{t, k-1}, z\ind{t, k}}{2} - \grad^2 F 
	\prn{x\ind{t, k-1}}}  \prn{x\ind{t} - x\ind{t-1}} }^2 \lmid 
	{\cG\ind{t}}}
	\nonumber \\& \qquad
	\le 
	\frac{1}{(\numbabysteps\ind{t})^2}\sum_{k=1}^{\numbabysteps\ind{t}} 
	\En \brk*{\nrm*{\stocder{x\ind{t, k-1}, z\ind{t, k}}{2} - \grad^2 
	F\prn{x\ind{t, k-1}} }_{\op}^2 \lmid 
		{\cG\ind{t}}}  \norm[\big]{x\ind{t} - x\ind{t-1}}^2
		\nonumber \\& \qquad
	\le   \sigma_2^2\cdot  
	\frac{\norm{x\ind{t}-x\ind{t-1}}^2}{\numbabysteps\ind{t}} \le 
	\bias\ind{t} 
	\cdot 
	\frac{\epsilon^2}{5},
	\label{eq:hess-vec-variance}
	\end{flalign}
	where the equality $(\star)$ above is due to the fact that 
	$z\ind{t,1},\ldots,z\ind{t,\numbabysteps\ind{t}}$ are i.i.d., as well as 
	$x\ind{t,k}-x\ind{t,k-1}=\frac{1}{\numbabysteps\ind{t}}(x\ind{t}-x\ind{t-1})$.

	Next, we observe that Taylor's theorem and fact that $F$ has $L_2$-Lipschitz 
	Hessian implies that $\norm{\grad F(x')-\grad F(x) - \hess(x)F (x'-x)}\le 
	\frac{L_2}{2}\norm{x'-x}^2$ for all $x,x'\in\R^d$. Therefore, 
	\begin{flalign}
	\norm*{\E \brk*{ \psi\ind{t} ~\Big\vert~ \cG\ind{t}} }
	&
	= 
	\norm[\Bigg]{\sum_{k=1}^{\numbabysteps\ind{t}}
		\grad F(x\ind{t,k})- 
		\grad F(x\ind{t,k-1}) -  
		\hess F (x\ind{t,k-1})\prn*{ x\ind{t,k} - 	x\ind{t,k-1}}
	}
	\nonumber \\& 
	\le \sum_{k=1}^{\numbabysteps\ind{t}}
	\norm*{
		\grad F(x\ind{t,k})- 
		\grad F(x\ind{t,k-1}) -  
		\hess F (x\ind{t,k-1})\prn*{ x\ind{t,k} - 	x\ind{t,k-1}}
	}
	\nonumber \\& 
	\le \numbabysteps\ind{t} \cdot  \frac{L_2}{2} \cdot 
	\prn*{\frac{\norm{x\ind{t}-x\ind{t-1}}}{\numbabysteps\ind{t}}}^2
		\le 
	\bias\ind{t} \cdot \frac{\eps}{50},
	\label{eq:hess-vec-bias}
	\end{flalign}
	where we used \pref{eq:key-scaling} again. 
	
	Substituting back through equations~\pref{eq:hess-vec-bias}, 
	\pref{eq:hess-vec-variance}, \pref{eq:coin0-error}, \pref{eq:coin1-error} 
	and \pref{eq:coin-errors}, we have
	\begin{flalign*}
	\E\norm[\big]{\err\ind{t}}^2 &\le \E \brk*{\bias\ind{t} \cdot 
	\tfrac{\epsilon^2}{5} 
	+ (1-\bias\ind{t}) \prn*{(1+\tfrac{\bias\ind{t}}{2}) 
	\norm[\big]{\err\ind{t-1}}^2
	+( 1 + \tfrac{2}{\bias\ind{t}}) (\tfrac{\bias\ind{t} \epsilon}{50})^2 + 
	\bias\ind{t}\cdot \tfrac{\epsilon^2}{5}}
}
	\\ & \leq 
	\prn*{1-\tfrac{\E\brk{\bias\ind{t}}}{2}}\En{}{ \nrm[\big]{ \gradestc{t-1} - 
			\derF{x\ind{t-1}}{}}^2}  + \tfrac{\E\brk{\bias\ind{t}}}{2} \epsilon^2\leq\eps^{2},
	\end{flalign*}
	as required; the second inequality follows from algebraic manipulation and 
	the fact that 
	$\err\ind{t-1}$ is 
	independent of $\bias\ind{t}$ by 
	assumption.
\end{proof}

The following lemma bounds the number of oracle queries made per call to 
the gradient estimator. 

\begin{lemma} 
\label{lem:gradient_estimator_oracle_complexity_meta_lemma}
The expected number of stochastic oracle queries made by 
$\textnormal{\getgradientfunc}$ when called a single time with arguments ($x$, $\xprev, \gprev$)
and parameters $(\eps,\bias)$ is at most
\[ 6\prn*{1+\frac{\bias \sigma_1^2}{\epsilon^2} + \frac{(\sigma_2^2 + L_2 
\epsilon) \cdot \nrm*{x - \xprev}^2}{b \epsilon^2}}.\]
\end{lemma} 

\begin{proof} Let $m$ denote the number of oracle calls made by the gradient estimator when invoked with arguments ($x$, $\xprev, \gprev$).  For any call to the estimator, there are two cases, either (a) $\coin = 1$, or (b) $\coin = 0$. In the first case, the gradient estimator queries the oracle $n$ times at the point $x$ and returns the empirical average of the returned stochastic estimates (see \pref{line:fresh_start} in \pref{alg:gradient_estimator}). Thus, $m = n$ for this case. In the second case, the estimator queries the oracle once for each point in the set $\prn*{x\ind{k - 1}}_{k=1}^K$, and updates the gradient using a stochastic path integral as in \pref{line:hvp_update_step}. Thus, $m = K$ for this case.

Combining the two cases, using $\coin\sim\bern(\bias)$ and substituting 
in the values of $n$ and $\numbabysteps$, we get \begin{align*}
   \En \brk*{m} &= \Pr \prn*{\coin = 1} \En \brk*{m \mid{} \coin = 1} + \Pr \prn*{\coin = 0} \En \brk*{m \mid{}  \coin = 0} \\
    &= \En \brk*{ \bias \cdot n + \prn*{1 - \bias} \cdot \numbabysteps } \\ 
    & = { \ceil*{\frac{5\bias\sigma_1^2}{\epsilon^2}} + 
    \ceil*{\frac{5(\sigma_2^2 + L_2 \epsilon) \cdot \nrm*{x - \xprev}^2}{b 
    \epsilon^2}}} \\
    &\le 
    6\prn*{\frac{\bias \sigma_1^2}{\epsilon^2} + \frac{(\sigma_2^2 + L_2 
    \epsilon) \cdot \nrm*{x - \xprev}^2}{b \epsilon^2} + 1} ,
\end{align*} 
where the final inequality follows from  $\ceil*{x} \leq x + 1$.
\end{proof}

\section{Supporting technical results}
\subsection{Error bound for empirical Hessian} \label{app:hess_estimator}
In order to find the negative curvature direction at a given point or
to build a cubic regularized sub-model,
\arxiv{\pref{alg:epsilon_cubic} and \pref{alg:eg_cubic} estimate}
\colt{\pref{alg:epsilon_cubic} estimates}
the
Hessian by computing an empirical average of the stochastic Hessian
queries to the oracle. The following lemma is a standard result which bounds the expected
error for the empirical Hessian. 

\begin{lemma} \label{lem:hess_error_gen_lemma}
Given a function  $F \in \cF_2(\Delta, \infty, L_2)$, a stochastic oracle in 
$\stocpOracle[2]$ and a  point $x$, let $H \ldef{} \frac{1}{m} \sum_{i=1}^m 
\stocder{x, z\ind{i}}{2}$ denote the empirical Hessian at the point $x$ 
estimated using $m$ stochastic queries at $x$, where $z\ind{i}\iidsim{}P_z$.  Then \begin{align*} 
\En \brk*{ \nrm*{ \hessest - \grad^2 F(x)}^2_\op} \leq \frac{22 \sigma_2^2 \log(d)}{m}.
                                                     \end{align*}
\end{lemma}  
\begin{proof}This is an immediate consequence of
  \pref{lem:matrix_var} below, using $A_i \ldef{} \stocder{x, z\ind{i}}{2}$
  and $B \ldef{} \derF{x}{2}$.
\end{proof} 

\begin{lemma} 
  \label{lem:matrix_var} 
  Let $(A_i)_{i=1}^{n}$ be a collection of i.i.d. matrices in $\symd$, with $\En\brk*{A_i}=B$ and
  $\En\nrm*{A_i-B}_{\op}^{2}\leq{}\sigma^{2}$. Then it holds that
  \[
    \En\nrm*{\frac{1}{n}\sum_{i=1}^{n}A_i-B}_{\op}^{2} \leq{} \frac{22 \sigma^{2}\log{}d}{n}.
  \]
\end{lemma}
\begin{proof}
  We drop the normalization by $n$ throughout this proof. We first
  symmetrize. Observe that by Jensen's inequality we have
  \begin{align*}
    \En\nrm*{\sum_{i=1}^{n}A_i-B}_{\op}^{2} &\leq{}
                                              \En_{A}\En_{A'}\nrm*{\sum_{i=1}^{n}A_i-A_i'}_{\op}^{2}\\
                                            &=
    \En_{A}\En_{A'}\nrm*{\sum_{i=1}^{n}(A_i-B)-(A_i'-B)}_{\op}^{2}\\
    &=
      \En_{A}\En_{A'}\En_{\eps}\nrm*{\sum_{i=1}^{n}\eps_i((A_i-B)-(A_i'-B))}_{\op}^{2}
      \leq{}
      4\En_{A}\En_{\eps}\nrm*{\sum_{i=1}^{n}\eps_i(A_i-B)}_{\op}^{2},
  \end{align*}
  where $(A')_{i=1}^{n}$ is a sequence of independent copies of
  $(A_i)_{i=1}^{n}$ and  $(\eps_i)_{i=1}^{n}$ are Rademacher random
  variables. Henceforth we condition on $A$. Let $p=\log{}d$, and let
  $\nrm*{\cdot}_{S_p}$ denote the Schatten $p$-norm. In what follows,
  we will use that for any matrix $X$,
  $\nrm*{X}_{\op}\leq{}\nrm*{X}_{S_{2p}}\leq{}e^{1/2}\nrm*{X}_{\op}$. To
  begin, we have
  \[
    \En_{\eps}\nrm*{\sum_{i=1}^{n}\eps_i(A_i-B)}_{\op}^{2}
    \leq{}\En_{\eps}\nrm*{\sum_{i=1}^{n}\eps_i(A_i-B)}_{S_{2p}}^{2}
    \leq{} \prn*{\En_{\eps}\nrm*{\sum_{i=1}^{n}\eps_i(A_i-B)}_{S_{2p}}^{2p}}^{1/p},
  \]
  where the second inequality follows by Jensen. We now apply the
  matrix Khintchine inequality \citep[Corollary
  7.4]{mackey2014matrix}, which implies that
  \begin{align*}
    \prn*{\En_{\eps}\nrm*{\sum_{i=1}^{n}\eps_i(A_i-B)}_{S_{2p}}^{2p}}^{1/p}
    \leq{} (2p-1) \nrm*{\sum_{i=1}^{n}(A_i-B)^{2}}_{S_{2p}}
    &\leq{} (2p-1) \sum_{i=1}^{n}\nrm*{(A_i-B)}_{S_{2p}}^{2}\\
    &\leq{} e(2p-1)\sum_{i=1}^{n}\nrm*{(A_i-B)}_{\op}^{2}.
  \end{align*}
  Putting all the developments so far together and taking expectation
  with respect to $A$, we have
  \begin{align*}
    \En\nrm*{\sum_{i=1}^{n}A_i-B}_{\op}^{2}
    \leq{} 4e(2p-1)\sum_{i=1}^{n}\En_{A_i}\nrm*{(A_i-B)}_{\op}^{2}
    \leq{} 4e(2p-1)n\sigma^{2}.
  \end{align*}
  To obtain the final result we normalize by $n^{2}$.
\end{proof}
\subsection{Descent lemma for stochastic gradient descent} \label{apx:supporting_lemmas}

The following lemma characterizes the effect of gradient descent update step used by \pref{alg:epsilon_hvp} and \pref{alg:epsilon_gamma_hvp}.
 \begin{lemma} 
\label{lem:general_lemma_gradient_descent}
Given a function $F \in \cF_2(\Delta, L_1, \infty)$, a point $x$, and
gradient estimator $\gradest{}$ at x, define \[ y \ldef{} x - \eta \gradest. \] Then, for any $\eta \leq \frac{1}{2L_1}$, the point $y$ satisfies   
\begin{align*}
 F\prn{x} - F\prn{y}  \geq  \frac{\eta}{8} \nrm*{ \nabla F\prn{x}}^2 -  \frac{3\eta}{4} \nrm*{\nabla F\prn{x} - \gradest}^2.
\end{align*}
\end{lemma}
\begin{proof}
\noindent
Since, the gradient of $F$ is $\Lip{1}$-Lipschitz, we have
\begin{align*}
   F\prn{y} &\leq F\prn{x} + \tri*{\nabla F\prn{x}, y - x} + \frac{\Lip{1}}{2} \nrm*{y - x}^2 \\
							&\overeq{\proman{1}} F\prn{x} - \eta \tri*{\nabla F\prn{x}, \gradest} + \frac{\Lip{1} \eta^2}{2} \nrm*{\gradest}^2 \\ 
     						&\overeq{}  F\prn{x} - \eta \tri*{\nabla F\prn{x} - \gradest, \gradest}  - \eta \nrm*{\gradest}^2 + \frac{\Lip{1} \eta^2}{2} \nrm*{\gradest}^2  \\ 
     						&\overleq{\proman{2}}  F\prn{x} + \eta  \nrm*{\nabla F\prn{x} - \gradest}\nrm*{ \gradest}  - \eta \prn*{1 - \frac{L_1 \eta}{2}} \nrm*{\gradest}^2  \\
     						&\overleq{\proman{3}}  F\prn{x} + \frac{\eta}{2}  \nrm*{\nabla F\prn{x} - \gradest}^2 - \eta \prn*{\frac{1}{2} - \frac{L_1 \eta}{2}} \nrm*{\gradest}^2 \\ 
     					    &\overleq{\proman{4}}  F\prn{x} + \frac{\eta}{2}  \nrm*{\nabla F\prn{x} - \gradest}^2 - \frac{\eta}{4} \nrm*{\gradest}^2 \\ 
     					    &\overleq{\proman{5}} F\prn{x} + \frac{3\eta}{4} \nrm*{\nabla F\prn{x} - \gradest}^2 - \frac{\eta}{8} \nrm*{ \nabla F\prn{x}}^2, \numberthis \label{eq:epsilon_hvp_proof1}
\end{align*} 
where  $\proman{1}$ uses that  $ y - x =  \eta\gradest$, $\proman{2}$
is due to the Cauchy-Schwarz inequality, $\proman{3}$  is given by an
application of the AM-GM inequality and $\proman{4}$ holds because $\eta  \leq \frac{1}{2L_1}$. Finally, $\proman{5}$ follows by invoking Jensen's inequality for the function $\nrm*{\cdot}^2$ to upper bound $\nrm*{\nabla F(x)}^2 \leq 2 \prn*{ \nrm*{\nabla F(x - \gradest)}^2 + \nrm*{\gradest}^2 }$. Rearranging the terms in \pref{eq:epsilon_hvp_proof1}, we get,
\begin{align*}
 F\prn{x} - F\prn{y}  \geq  \frac{\eta}{8} \nrm*{ \nabla F\prn{x}}^2 -  \frac{3\eta}{4} \nrm*{\nabla F\prn{x} - \gradest}^2.
\end{align*}
\end{proof}

\subsection{Descent lemma for cubic-regularized trust-region method}
The following lemmas establish properties for the updates step involving constrained minimization of the cubic regularized model in used in \pref{alg:epsilon_cubic}\arxiv{ and \pref{alg:eg_cubic}}. 

 \begin{lemma} 
\label{lem:constrained_cubic_struct_lem1}
Given a function $F \in \cF_2(\Delta, \infty, L_2)$, gradient estimator $\gradest{}\in\bbR^{d}$ and hessian estimator $\hessest{}\in\symd$, define \[m_x(y) = F(x) + \tri*{\gradest{},y-x} + \frac{\hessest{}}{2}\brk*{y-x, y-x} + \frac{M}{6}\nrm*{y-x}^{3},\]  and let $y \in \argmin_{z \in \ball_\eta(x)} m_x(z)$. Then, for any $M \geq 4 L_2$ and $\eta \geq 0$, the point $y$ satisfies   
\begin{equation*}
	F(x) - F(y) \geq  \frac{M}{12} \nrm*{y - x}^3 -  \frac{8}{\sqrt{M}}\nrm*{\grad F(x) - \gradest{}}^{\frac{3}{2}} + \frac{4\eta^\frac{3}{2}}{\sqrt{M}} \nrm*{\grad^2 F(x) - \hessest{}}^{\frac{3}{2}}.   
\end{equation*} 
\end{lemma}
\begin{proof}
Since $\hess F$ is $L_2$-Lipschitz, we have
\begin{align*}
   F(y) - F(x)  &\leq F(x) + \tri*{\grad F(x), y-x} + \frac{1}{2} \grad^2 F(x) \brk*{y-x, y-x} + \frac{L_2}{6}\nrm*{y-x}^3 - F(x) \\ 
   					&\overeq{\proman{1}} m_x(y) + \frac{L_2 - M}{6} \nrm*{y-x}^3 + \tri*{\grad F(x) - \gradest{}, y-x} + \frac{1}{2} \grad^2 F(x)\brk*{y-x, y-x} \\ 
   					& \qquad  - \frac{1}{2} \hessest{}\brk*{y-x, y-x} - m_x(x) \\ 
   					&\overleq{\proman{2}}  - \frac{M}{8} \nrm*{y-x}^3 + \nrm*{\grad F(x) - \gradest{}}\nrm*{y-x}  + \frac{1}{2} \nrm*{\grad^2 F(x) \brk*{y-x, \cdot} - \hessest{} \brk*{y-x, \cdot}} \nrm*{y-x},  \numberthis \label{eq:cubic_partlem1_eq1} 
\end{align*} where $\proman{1}$ follows from the definition of $m_x(\cdot)$ and $\proman{2}$ follows by the fact that $y \in \argmin_{y' \ball_\eta(x)} m_x(y')$, along with an application of the  Cauchy-Schwarz inequality for remainder of the terms, and because $M \geq 4 L_2$. Additionally, using Young's inequality, we have
\begin{align*}
   \nrm*{\grad F(x) - \gradest{}}\nrm*{y-x} &\leq  \frac{8}{\sqrt{M}}\nrm*{\grad F(x) - \gradest{}}^{\frac{3}{2}} + \frac{M}{64}\nrm*{y-x}^3, 
\intertext{and,} 
	\nrm*{\grad^2 F(x) \brk*{y-x, \cdot} - \hessest{} \brk*{y-x, \cdot}} \nrm*{y-x} &\leq  \frac{8}{\sqrt{M}} \nrm*{\grad^2 F(x) \brk*{y-x, \cdot} - \hessest{} \brk*{y-x, \cdot}}^\frac{3}{2}  +  \frac{M}{64} \nrm*{y-x}^3. 
\end{align*}
Plugging these bounds into \pref{eq:cubic_partlem1_eq1}, we have
\begin{align*}
      F(y) - F(x) &\leq - \frac{M}{12} \nrm*{y-x}^3 +  \frac{8}{\sqrt{M}}\nrm*{\grad F(x) - \gradest{}}^{\frac{3}{2}} 
							  + \ \frac{4}{ \sqrt{M}} \nrm*{\grad^2 F(x) \brk*{y-x, \cdot} - \hessest{} \brk*{y-x, \cdot}}^\frac{3}{2}  \\
						&\overleq{\proman{1}} - \frac{M}{12} \nrm*{y-x}^3 +  \frac{8}{\sqrt{M}}\nrm*{\grad F(x) - \gradest{}}^{\frac{3}{2}} 
							  + \ \frac{4}{ \sqrt{M}} \nrm*{\grad^2 F(x) - \hessest{}}_\op^\frac{3}{2} \nrm*{y-x}^\frac{3}{2}  \\
						&\overleq{\proman{2}} - \frac{M}{12} \nrm*{y-x}^3 +  \frac{8}{\sqrt{M}}\nrm*{\grad F(x) - \gradest{}}^{\frac{3}{2}} 
							  + \ \frac{4}{ \sqrt{M}} \nrm*{\grad^2 F(x) - \hessest{}}_\op^\frac{3}{2} \cdot \eta^\frac{3}{2},
\end{align*} where $\proman{1}$ follows by the definition of the operator norm and $\proman{2}$ follows by observing that $\nrm*{y-x} \leq \eta$. Rearranging the terms, we have
\begin{equation*} 
	F(x) - F(y) \geq  \frac{M}{12} \nrm*{y - x}^3 -  \frac{8}{\sqrt{M}}\nrm*{\grad F(x) - \gradest{}}^{\frac{3}{2}} + \frac{4\eta^\frac{3}{2}}{\sqrt{M}} \nrm*{\grad^2 F(x) - \hessest{}}^{\frac{3}{2}}.   
\end{equation*} 
\end{proof}

\begin{lemma}
\label{lem:constrained_cubic_struct_lem2} Under the same setting as \pref{lem:constrained_cubic_struct_lem1}, the point $y$ satisfies 
\begin{align*} 
\indicator{ \nrm*{\nabla F(y)} \geq \frac{M\eta^2}{2}} &\leq   \frac{2}{\eta^2} \nrm*{y-x}^2 + \frac{2}{M\eta^2} \prn*{{\nrm*{\nabla F(x) - \gradest{}} } +  {\eta \nrm*{ \nabla^2 F(x) - \hessest}_\op}}.    
\end{align*}
\end{lemma}
\begin{proof}
There are two scenarios: $\proman{1}$ either $y$ lies on the boundary of $\bbB_\eta(x)$, or $\proman{2}$ $y$ is in the interior of  $\bbB_\eta(x)$. In the first case, $\nrm*{y - x} = \eta$. In the second case, \begin{align*} 
   \nrm*{\nabla F(y)} &\overleq{\proman{1}} \nrm*{ \nabla F(y) - \nabla F(x) - \nabla^2 F(x)\brk*{y-x, \cdot} } + \nrm*{\nabla F(x) + \nabla^2 F(x)\brk*{y - x, \cdot}} \\
  								&\overleq{\proman{2}}  \frac{L_2}{2} \nrm*{y-x}^2  + {\nrm*{\nabla F(x) + \nabla^2 F(x)\brk*{y - x, \cdot}}}  \\ 
   								&\overleq{\proman{3}} \frac{L_2}{2} { \nrm*{y-x}^2} + {\nrm*{\nabla F(x) - \gradest{}} } + { \nrm*{ \nabla^2 F(x)\brk*{y - x, \cdot} - \hessest \brk*{y - x, \cdot}}} + {\nrm*{ \gradest{}+ \hessest{}[y-x, \cdot]}} \\ 
   								&\overleq{\proman{4}} \frac{L_2}{2} { \nrm*{y-x}^2} + {\nrm*{\nabla F(x) - \gradest{}} } + { \nrm*{ \nabla^2 F(x) - \hessest}_\op \cdot \eta }  +  {\nrm*{ \gradest{}+ \hessest{}[y-x, \cdot]}} \\ 
   								&\overleq{\proman{5}} \frac{L_2 + M}{2} { \nrm*{y-x}^2} +  {\nrm*{\nabla F(x) - \gradest{}}} + { \nrm*{ \nabla^2 F(x) - \hessest}_\op \cdot \eta}, \numberthis \label{eq:cubic_partlem2_eq1}
\end{align*} 
where $\proman{1}$ follows by triangle inequality, $\proman{2}$ follows by Taylor expansion of  $\nabla F(y)$ at $x$ and observing that $F$ is $\Lip{2}$-hessian Lipschitz, $\proman{3}$ follows by another application of the triangle inequality, $\proman{4}$ follows from Cauchy-Schwarz inequality and observing that $\nrm*{y-x} \leq \eta$, and $\proman{5}$ follows by using first order optimization conditions for $y \in \argmin_{\ball_{\eta}(x)} m_x(y)$, i.e., 
\begin{align*}
 \nrm*{\nabla \hm_x(y)} = 0, \text{~ or, \quad}  \gradest{}+ \hessest\brk*{y-x, \cdot} + \frac{M}{2} \nrm*{y-x} \prn*{ y -x} = \mb{0}.   
\end{align*}

\noindent
Rearranging the terms in \pref{eq:cubic_partlem2_eq1}, we get, 
\begin{align*}
	\nrm*{y-x}^2 \geq \frac{2}{L_2 + M} \prn*{    \nrm*{\nabla F(y)}  -   {\nrm*{\nabla F(x) - \gradest{}}} -  { \nrm*{ \nabla^2 F(x) - \hessest}_\op \cdot \eta} }.
\end{align*} 

\noindent
Since one of the two cases ($\nrm*{y - x} < \eta$ or $\nrm*{y - x} = \eta$) must hold, we have,  
\begin{align*}
   \nrm*{y - x}^2 &\geq \min \crl*{\eta^2,  \frac{2}{L_2 + M} \prn*{ { \nrm*{\nabla F(y)}}  -  {\nrm*{\nabla F(x) - \gradest{}} } - {\eta \cdot \nrm*{ \nabla^2 F(x) - \hessest}^2_\op}}} \\ 
   						  &\geq \min \crl*{\eta^2,  \frac{2}{L_2 + M} { \nrm*{\nabla F(y)}}}   -  \frac{2}{L_2 + M}{\nrm*{\nabla F(x) - \gradest{}} } - { \frac{2 \eta}{L_2 + M} \nrm*{ \nabla^2 F(x) - \hessest}_\op}. 
\end{align*} 
Rearranging the terms, and using the fact that $M \geq 2L_2$, we have
\begin{equation*}
 \min \crl*{\frac{M\eta^2 }{2},  { \nrm*{\nabla F(y)}}} \leq  M \nrm*{y-x}^2 + {\nrm*{\nabla F(x) - \gradest{}} } +  {\eta \nrm*{ \nabla^2 F(x) - \hessest}_\op}. 
\end{equation*} 
Finally, using the fact that for any $a, b \geq 0$, $\min\crl*{a, b} \leq a \indicator{b \geq a}$, we have
\begin{align*}
 \frac{M\eta^2 }{2} \indicator{ \nrm*{\nabla F(y)} \geq \frac{M\eta^2}{2}} &\leq  M \nrm*{y-x}^2 + {\nrm*{\nabla F(x) - \gradest{}} } +  {\eta \nrm*{ \nabla^2 F(x) - \hessest}_\op}, \intertext{or, equivalently, }
\indicator{ \nrm*{\nabla F(y)} \geq \frac{M\eta^2}{2}} &\leq   \frac{2}{\eta^2} \nrm*{y-x}^2 + \frac{2}{M\eta^2} \prn*{{\nrm*{\nabla F(x) - \gradest{}} } +  {\eta \nrm*{ \nabla^2 F(x) - \hessest}_\op}}.    
\end{align*}
\end{proof}

\begin{lemma} \label{lem:epsilon_cubic_descent_epsilon_step} Consider the same setting as \pref{lem:constrained_cubic_struct_lem1}, but let $H\in\symd$ and $g\in\bbR^{d}$ be random variables. Then the random variable $y$ satisfies
\begin{align*}
      \En \brk*{F\prn{ x} -  F\prn{y}} &\geq  \frac{M \eta^3}{60}  \Pr \Big( \nrm*{\derF{y}{}} \geq \frac{M\eta^2}{2} \Big) - \frac{9}{\sqrt{M}} \cdot { \En \brk*{ \nrm*{\nabla  F\prn{ x} - \gradest}^2}}^\frac{3}{4} \\ 
      & \qquad \qquad \quad - \frac{5 \eta^{\frac{3}{2}}}{\sqrt{M}} \cdot {\En \brk*{   \nrm*{ \nabla^2  F\prn{ x} - \hessest}^2_\op}}^{\frac{3}{4}},
\end{align*}
where $\Pr(\cdot)$ and $\En\brk*{\cdot}$ are taken with respect to the randomness over $H$ and $g$.
\end{lemma}
\begin{proof} For the ease of notation, let $\chi$ and  $\zeta$ denote the error in the gradient estimator $\gradest$ and the hessian estimator $\hessest$ at $x$ respectively, i.e. 
\begin{align*} 
     \chi \ldef{} \nrm*{\nabla  F\prn{ x} - \gradest} \quad \text{and }  \quad \zeta \ldef{} \nrm*{ \nabla^2  F\prn{ x} - \hessest}_\op.   
\end{align*}
We prove the desired statement by combining the following two results. 
\begin{enumerate}[label=$\bullet$] 
\setlength{\itemindent}{-.14in} 
\item	 First, plugging $x = x$, and $z = y$ in to \pref{lem:constrained_cubic_struct_lem1}, we have
\begin{align*}
   F\prn{ x} -  F\prn{y} \geq   \frac{M}{12} \nrm*{y - x}^3 - \frac{8}{\sqrt{M}} \chi^\frac{3}{2}_t - \frac{4}{\sqrt{M}} \prn*{ \eta \zeta }^{\frac{3}{2}}. 
\end{align*}
 
Taking expectations on both the sides, we get, 
\begin{align*}
   \En \brk*{F\prn{ x} -  F\prn{y}} &\geq  \frac{M}{12} \En \brk*{ \nrm*{y - x}^3} - \frac{8}{\sqrt{M}} \En \brk*{ \prn*{\chi}^\frac{3}{2} } - \frac{4}{\sqrt{M}} \En \brk*{  \prn*{ \eta \zeta }^{\frac{3}{2}} } \\
   &\geq  \frac{M}{12} \En \brk*{ \nrm*{y - x}^3} - \frac{8}{\sqrt{M}} \prn*{ \En \brk*{ \chi^2_t}}^\frac{3}{4} - \frac{4}{\sqrt{M}}  \prn*{\eta^2 \En \brk*{   \zeta^2_t }}^{\frac{3}{4}}  \numberthis \label{eq:eg_cubic_sublemma1_1}, 
\end{align*} where the last inequality follows from an application of Jensen's inequality. 

\item Similarly, plugging $x = x$, $z = y$ in \pref{lem:constrained_cubic_struct_lem2}, we get
\begin{align*}
\indicator{ \nrm*{\derF{y}{}} \geq \frac{M\eta^2}{2}} &\leq   \frac{2}{\eta^2} \nrm*{ y -x}^2 + \frac{2}{M\eta^2} \prn*{\chi  +  \eta \zeta }.    
\end{align*}  Raising both the sides with the exponent of $\frac{3}{2}$, we get
\begin{align*}
\indicator{ \nrm*{\derF{y}{}} \geq \frac{M\eta^2}{2}} &\leq  \prn*{ \frac{2}{\eta^2} \nrm*{y-x}^2 + \frac{2}{M\eta^2} \prn*{\chi  +  \eta \zeta }}^{\frac{3}{2}} \\ 
&\leq \frac{5}{\eta^3} \nrm*{y-x}^3 + \frac{5}{M^\frac{3}{2}\eta^3} \prn*{\chi^\frac{3}{2}  +  \prn*{\eta \zeta}^\frac{3}{2}}.    
\end{align*} Taking expectations on both the sides and rearranging the terms implies that     
\begin{align*}
\En \brk*{ \nrm{ x\ind{t +1} - x}^3} &\geq \frac{\eta^3}{5}  \Pr \prn*{ \nrm*{\derF{y}{}} \geq \frac{M\eta^2}{2}}  - \frac{1}{M^\frac{3}{2}} \En \brk*{\chi^\frac{3}{2}  +  \prn*{\eta \zeta}^\frac{3}{2}} \\ 
&\geq \frac{\eta^3}{5}  \Pr \prn*{ \nrm*{\derF{y}{}} \geq \frac{M\eta^2}{2}}  - \frac{1}{M^\frac{3}{2}} \prn*{\prn*{\En \brk*{\chi^2_t}}^\frac{3}{4}  +  \prn*{\eta^2 \En \brk*{ \zeta^2_t}}^\frac{3}{4}},   \numberthis \label{eq:eg_cubic_sublemma1_2}  
\end{align*} where the last inequality follows from an application of the Jensen's inequality. 
\end{enumerate} 
Plugging \pref{eq:eg_cubic_sublemma1_2}  into  \pref{eq:eg_cubic_sublemma1_1}, we get
\begin{align*}
      \En \brk*{F\prn{ x} -  F\prn{y}} &\geq  \frac{M \eta^3}{60}  \Pr \prn*{ \nrm*{\derF{y}{}} \geq \frac{M\eta^2}{2}} - \frac{9}{\sqrt{M}} \prn*{ \En \brk*{ \chi^2_t}}^\frac{3}{4} - \frac{5 \eta^{\frac{3}{2}}}{\sqrt{M}}  \prn*{\En \brk*{   \zeta^2_t }}^{\frac{3}{4}}. 
 \end{align*}
The final statement follows from the above inequality by using the definition of $\chi$ and $\zeta$. 

\end{proof}

\subsection{Stochastic negative curvature search} \label{apx:supporting_lemmas_eg} 
The following lemma establishes properties of the negative curvature search step used in
\pref{alg:epsilon_gamma_hvp} and \pref{alg:eg_cubic}.
\begin{lemma} 
\label{lem:general_lem_curvature_descent}
Let $\gamma > 0$, and $F \in \cF_2(\Delta, \infty, L_2)$ be given.
Let $x\in\bbR^{d}$ be given, and let $H\in\symd$ be a random variable
(representing a stochastic estimator for the Hessian at $x$). Define
$y$ via
\begin{align*}
   y \ldef{} \left\{  \begin{array}{ll} x +  \frac{r \gamma}{L_2} \cdot u, & \text{if} \quad \eigmin(H) \leq - 4\gamma,
   											  			\\  
 x, & \textnormal{otherwise.}  
   											  				\end{array} \right.,
\end{align*} 
where $r$ is an independent Rademacher random variable and $u$ is an arbitrary unit vector such that $H[u, u] \leq - 2\gamma$. Then, the point $y$ satisfies 
\begin{align*} 
  \En  \brk*{ \func{x} -  \func{y}} \geq  \frac{5\gamma^3}{6 L_2^2} \Pr \prn*{\eigmin(H) \leq - 4\gamma} - \frac{\gamma^2}{2L_2^2} \En \brk*{\nrm*{ \derF{x}{2} - H}_\op},
\end{align*}
where $\Pr(\cdot)$ and $\En\brk*{\cdot}$ are taken with respect to the
randomness in $H$ and $r$.
\end{lemma}
\begin{proof} There are two cases: either (a) $\eigmin(H) > -4\gamma$, or, (b) $\eigmin(H) \leq  -4\gamma$. In the first case,  $y = x$, and thus, 
\begin{align*}
   F(y) - F(x) = 0 \leq \frac{\gamma^2}{2L_2^2}  \nrm*{H - \derF{x}{2}}_\op \numberthis \label{eq:eps_gamma_proof_gamma_descent1} 
\end{align*}
In the second case, Taylor expansion for $\func{y}$ at $\func{x}$ implies that   
\begin{align*}
F\prn{y} &\leq F\prn{x} +  \tri*{\grad \func{x}, \tilde{u}} + \frac{1}{2} \grad^2 \func{x} \brk*{\tilde{u}, \tilde{u}}  + \frac{L_2}{6} \nrm*{\tilde{u}}^3, \\ 
\intertext{where $\tilde{u} \ldef \frac{r \gamma }{L_2} \cdot u$. Taking expectations on both the sides with respect to $r$, we get}
		\En_r \brk*{F\prn{y} }	 &\overeq{\proman{1}} F\prn{x}  + \frac{\gamma^2}{2L_2^2}   \grad^2 \func{x} \brk*{u, u} +  \frac{\gamma^3}{6L_2^2} \nrm*{ u}^3 \\ 
		 &\leq F\prn{x}  + \frac{\gamma^2}{2L_2^2}   \prn*{ H\brk*{u, u} + \grad^2 \func{x}  \brk*{u, u} - H\brk*{u, u}} +  \frac{\gamma^3}{6L_2^2} \nrm*{ u}^3 \\ 
		  &\overeq{\proman{2}} F\prn{x}  + \frac{\gamma^2}{2L_2^2}\prn*{ - 2\gamma + \nrm*{ \derF{x}{2} - H}_\op} + \frac{\gamma^3}{6L_2^2} \\
           &\leq  F\prn{x}  - \frac{5\gamma^3}{6 L_2^2} + \frac{\gamma^2}{2L_2^2}  \nrm*{ \derF{x}{2} - H}_\op,  \numberthis \label{eq:eps_gamma_proof_gamma_descent2} 
\end{align*} where $\proman{1}$  is given by the fact that $\En_{r}
\brk*{ \tri*{ \grad \func{x},  r u}} = 0$, and $\proman{2}$ follows from the fact that $u$ is chosen such that $\En \brk*{ \derF{x}{2}[u, u]} \leq -2\gamma$ and $\nrm*{u} = 1$, and the fact that for any matrix $A$ and vector $b$, $\nrm*{Ab} \leq \nrm*{A}_\op \nrm*{b}$.  

Since, one of the two cases ($\eigmin(H) > -4\gamma$ or $\eigmin(H) \leq  -4\gamma$) must hold, combining \pref{eq:eps_gamma_proof_gamma_descent1}  and \pref{eq:eps_gamma_proof_gamma_descent2}, we have
\begin{align*}
  \En_r \brk*{ \func{y}} \leq \func{x} - \frac{5\gamma^3}{6 L_2^2} \indicator{\eigmin(H) \leq -4\gamma} + \frac{\gamma^2}{2L_2^2}  \nrm*{ \derF{x}{2} - H}_\op. 
\end{align*}
Taking expectation on both the sides gives the desired statement:
\begin{align*}
  \En  \brk*{ \func{x} -  \func{y}} \geq  \frac{5\gamma^3}{6 L_2^2} \Pr \prn*{\eigmin(H) \leq - 4\gamma} - \frac{\gamma^2}{2L_2^2} \En \brk*{\nrm*{ \derF{x}{2} - H}_\op}.
\end{align*}  \end{proof}
The following lemma establishes properties of Oja's method ($\ojaF$),
as used in \pref{alg:epsilon_gamma_hvp}. 
 \begin{lemma}[\citet{allen2018natasha}, Lemma 5.3] 
	\label{lem:ojas_algorithm} 
	The procedure $\ojaF$ takes as input a point $x \in \bbR^d$, 
	a stochastic Hessian-vector product oracle  $\estimator[2] \in 
	\stocpOracleas$ for some function $F \in \cF_2(\Delta, L_1, \infty)$, a 
	precision parameter $\gamma > 0$ and a failure probability $\delta 
	\in(0,1)$, and 
	runs 
	outputs $u \in \bbR^d \cup \crl*{\bot}$ such that with probability at 
	least 
	$1 - \delta$, either\footnote{Note that if this event fails, the algorithm still returns either $\bot$ or a unit vector $u$.}
	\begin{enumerate}[label=\alph*)]
		\item  $u = \bot$, and $\grad^2 F(x) \succeq -2\gamma I $.
		\item if $u\neq\bot$, then $\nrm*{u} = 1$ and $\tri{u, \grad^2 \func{x} u} \leq 
		-\gamma$.  
	\end{enumerate}
	Moreover, when invoked as above, the procedure uses at most
	\[O\prn*{\frac{\prn*{\bar{\sigma}_2 + L_1}^2}{4 \gamma^2} 
		\log^2\prn*{\frac{d}{\delta}}}\] queries to the stochastic 
	Hessian-vector product oracle. 
\end{lemma}

\section{Upper bounds for finding $\epsilon$-stationary points} \label{app:upper}
\subsection{Proof of Theorem \ref*{thm:epsilon_hvp}}
\label{app:epsilon_hvp}

\begin{proof}[\pfref{thm:epsilon_hvp}] In the following, we first show that \pref{alg:epsilon_hvp} returns a point $\hx$ such that, $\En \brk*{ \nrm*{ \nabla F(\hx)}}  \leq 32 \epsilon$. We then bound the expected number of oracle queries used throughout the execution.\footnote{In the proof, we show convergence to a $32\epsilon$-stationary point. A simple change of variable, i.e. running \pref{alg:epsilon_hvp} with $\epsilon \leftarrow \frac{\epsilon}{32}$, returns a point $\hx$ that enjoys the guarantee that $\nrm*{\derF{\hat{x}}{}} \leq \epsilon$.}

\noindent
Since, $\eta = \frac{1}{2\sqrt{L_1^2 + \bsigma_2^2 + \tepsilon L_2}} \leq \frac{1}{2L_1}$ and $F$ has $\Lip{1}$-Lipschitz gradient, \pref{lem:general_lemma_gradient_descent} implies that the point $x\ind{t+1}$ computed using the update rule $x\ind{t+1} \leftarrow x\ind{t} - \eta g\ind{t}$ satisfies   
\begin{align*} 
  \frac{\eta}{8} { \nrm*{ \nabla F \prn{x\ind{t}}}^2 } &\leq  F \prn{x\ind{t}} - F \prn{x\ind{t + 1}}  +  \frac{3\eta}{4} { \nrm*{\nabla F \prn{x\ind{t}} - \gradest[t]}^2}. 
\intertext{Telescoping the above from $t$ from $1$ to $T$, this implies}
  \frac{\eta}{8} \sum_{t=1}^{T}\nrm*{\derF{x\ind{t}}{}}^2  &\leq  F\prn{x\ind{0}} - F\prn{x\ind{T + 1}}  + \frac{3\eta}{4} \sum_{t=1}^{T} \nrm*{\nabla F\prn{x\ind{t}} - \gradest[t]}^2 \\
   &\leq  \Delta  + \frac{3\eta}{4} \sum_{t=1}^{T} \nrm*{\nabla F\prn{x\ind{t}} - \gradest[t]}^2,
\end{align*} where the last inequality follows from the fact that  $F\prn{x\ind{0}} - F\prn{x\ind{T + 1}} \leq \Delta$. Next, taking expectation on both the sides (with respect to the stochasticity of the oracle and the algorithm's internal randomization), we get
\begin{align*}
  \frac{\eta}{8} \En \brk*{ \sum_{t=1}^{T}\nrm*{\nabla F\prn{x\ind{t}}}^2 } &\leq \Delta + \frac{3\eta}{4} \sum_{t=1}^{T} \En \brk*{ \nrm*{\nabla F\prn{x\ind{t}} - \gradest[t]}^2}.
 \end{align*}
{Using \pref{lem:gradient_estimator_general_lemma}, we have $\En \brk*{ \nrm*{\nabla F\prn{x\ind{t}} - \gradest[t]}^2} \leq \epsilon^2$ for all $t \geq 1$. Dividing both the sides by $\frac{\eta T}{8}$, and plugging in the value of the parameters $T$ and $\eta$, we get, }   \begin{align*}
\En \brk*{ \frac{1}{T} \sum_{t=1}^{T}\nrm*{\derF{x\ind{t}}{}}^2 } &\leq  \frac{8 \Delta }{\eta T} + \frac{6}{T} \sum_{t=1}^{T} \En \brk*{ \nrm*{\nabla F\prn{x\ind{t}} - \gradest[t]}^2} \leq \frac{8 \Delta }{\eta T} + 6\epsilon^2 \leq 14 \epsilon^2.   \numberthis \label{eq:epsilon_hvp_proof2}	   	
\end{align*} 
Thus, for $\hx$ chosen uniformly at random from the set $\prn*{x\ind{t}}_{t=1}^{T}$, we have
\begin{align*}
   \En \nrm*{ \nabla F(\hx)} &= \frac{1}{T} \sum_{t=1}^{T} \En \nrm*{ \nabla F\prn{x\ind{t}}} \leq  \sqrt{\En \brk*{ \frac{1}{T} \sum_{t=1}^{T} \nrm*{ \nabla F\prn{x\ind{t}}}^2}} \leq 4 \epsilon. 	 
\end{align*}
Finally, Markov's inequality implies that with probability at least $\frac{7}{8}$, 
\begin{align*}
   \nrm*{\derF{\hx}{}} \leq 32 \epsilon. \numberthis \label{eq:gradient_bound_eps_hvp}
\end{align*}

\paragraph{Bound on the number of oracle queries.}
 \pref{alg:epsilon_hvp} queries the stochastic oracle in only when it invokes 
 \mainalg in \pref{line:gradient_estimate_epsilon_hvp} to compute the 
 gradient estimate $g\ind{t}$ at time $t$. Let $M$ denote the total number 
 of oracle calls made up until time  $T$. Invoking 
 \pref{lem:gradient_estimator_oracle_complexity_meta_lemma} to bound 
 the expected number of stochastic oracle calls for each $t \geq 1$, and 
 ignoring all the mutiplicative constants,  we get
\begin{align*}
\En \brk*{M} &\leq   5 \sum_{t=1}^{T} \En \brk*{ \frac{b 
\sigma_1^2}{\epsilon^2} + \frac{ \nrm*{x\ind{t + 1} - x\ind{t}}^2 \cdot 
\prn*{\sigma_2^2 + \epsilon L_2}}{b \epsilon^2} + 1  } \\ 
   		&\overleq{\proman{1}} O\prn*{ \sum_{t=1}^{T} \En \brk*{ \frac{b \sigma_1^2}{\epsilon^2} + \frac{ \nrm*{\eta \gradest[t]}^2 \cdot \prn*{\sigma_2^2 + \epsilon L_2}}{b \epsilon^2} + 1 }} \\ 
   		& \overleq{\proman{2}} O\prn*{\frac{\Delta}{\eta \epsilon^2} \cdot \prn*{  \frac{b \sigma_1^2}{\epsilon^2} + \En \brk*{ \frac{1}{T}  \sum_{t=1}^T \nrm*{\gradest[t]}^2} \cdot \frac{\eta^2 \prn*{\sigma_2^2 + \epsilon L_2}}{b \epsilon^2} + 1 }} \\
   		&\overeq{\proman{3}} O\prn*{\frac{\Delta}{\eta \epsilon^2} \cdot \prn*{  \frac{b \sigma_1^2}{\epsilon^2} + \frac{\eta^2 \prn*{\sigma_2^2 + \epsilon L_2}}{b} + 1 }}, \numberthis \label{eq:eps_gamma_hvp_oc22}
\end{align*} where $\proman{1}$ is given by plugging in the update rule from \pref{line:hvp_update_rule} and by dropping multiplicative constants, $\proman{2}$ is given by rearranging the terms, plugging in the value of $T$ and using that $T \geq 1$ (to simplify the ceiling operator) under the assumption $\epsilon \leq \sqrt{\Delta L_1}$,  and $\proman{3}$ follows by observing that
\begin{align*}
   \En \brk*{\frac{1}{T}  \sum_{t=1}^T \nrm*{\gradest[t]}^2} &\leq 2 \En \brk*{\frac{1}{T}  \sum_{t=1}^T \nrm*{\gradest[t] - \derF{x\ind{t}}{}}^2 + \frac{1}{T}  \sum_{t=1}^T \nrm*{ \derF{x\ind{t} }{}}^2} \leq  30 \epsilon^2,
\end{align*}
as a consequence of \pref{lem:gradient_estimator_general_lemma} and the bound in \pref{eq:epsilon_hvp_proof2}. Next, note that since we assume $\epsilon < \sigma_1$, and since we have $\eta \leq \frac{1}{2 \sqrt{\sigma_2^2 + \epsilon L_2}}$, the parameter $\bias$ is equal to $\frac{\eta \epsilon \sqrt{\sigma_2^2 + \epsilon L_2}}{\sigma_1}$ (as this is smaller than $1$). Thus, plugging the value of $\bias$ and $\eta$ in the bound \pref{eq:eps_gamma_hvp_oc22}, we get, 
 \begin{align*}
     				\En \brk*{m(T)}	  	&= O\prn*{ \frac{\Delta\sigma_1 \sqrt{
     				\sigma_2^2 + \epsilon L_2}}{\epsilon^3} + \frac{\Delta \sqrt{L_1^2 + \sigma_2^2 + \epsilon L_2}}{\epsilon^2}}  \\ 
     				&= O\prn*{ \frac{\Delta \sigma_1 \sigma_2}{\epsilon^3} + \frac{\Delta \sigma_1 \sqrt{L_2}}{\epsilon^{2.5}} + \frac{\Delta \sigma_2}{\epsilon^2} + \frac{\Delta L_1}{\epsilon^2} + \frac{\Delta \sqrt{L_2}}{\epsilon^{1.5}}}.
\end{align*}
Using Markov's inequality, we have that with probability at least $\frac{7}{8}$, 
\begin{align*}
   M \leq  O\prn*{ \frac{\Delta \sigma_1 \sigma_2}{\epsilon^3} + \frac{\Delta \sigma_1 \sqrt{L_2}}{\epsilon^{2.5}} + \frac{\Delta \sigma_2}{\epsilon^2} + \frac{\Delta L_1}{\epsilon^2} + \frac{\Delta \sqrt{L_2}}{\epsilon^{1.5}}}. \numberthis \label{eq:epsilon_hvp_proof3} 
\end{align*}
The final statement follows by taking a union bound with failure probabilities for \pref{eq:gradient_bound_eps_hvp} and \pref{eq:epsilon_hvp_proof3}.  
\end{proof}

\arxiv{\subsection{Proof of Theorem \ref*{thm:epsilon_cubic}}}
\colt{\subsection{Full statement and proof for Algorithm \ref*{alg:epsilon_cubic}}}
\label{app:upper_bounds_e_L2}

\colt{%
\begin{algorithm}[t]
   \caption{Subsampled cubic-regularized trust-region method with 
     \mainalg} 
	\label{alg:epsilon_cubic}
	\begin{algorithmic}[1] 
		\Require 
		 Oracle $(\estimator[2], P_z) \in 
\stocpOracle[2]$ for $F \in \cF_2\prn*{\Delta, \infty, L_2}$.
		 Precision parameter $\epsilon$.
		\State Set $M = 5 \max\crl*{L_2, \frac{\epsilon\sigma_2^2\log (d)}{\sigma_1^2}}$, $\eta =  25 \sqrt{\frac{\epsilon}{M}}$, $T = 
		\ceil*{\frac{5 \Delta}{3 \eta \epsilon}}$ and $n_H = \ceil*{\frac{22 \sigma_2^2 \eta^2 \log(d)}{\epsilon^2}}$.
		\State Set $\bias=\min\crl*{1, \frac{\eta \sqrt{\sigma_2^2 + \epsilon L_2}}{25 \sigma_1}}$. 
		\State Initialize $x\ind{0}, x\ind{1} \leftarrow 0$,~ $\gradest[0] \leftarrow \bot$.  
		\For {$t = 1 ~\text{to}~ T$}
		\State \label{line:Hessian_estimator}Query the oracle $n_H$ times at 
		$x \ind{t}$ and compute 
		\begin{equation*}
		\hessest[t] \leftarrow \frac{1}{n_H} \sum_{j=1}^{n_H} 
		\stocder{x\ind{t}, z\ind{t, j}}{2},\quad\text{where}\quad{}z\ind{t, 
		j}\overset{\mathrm{i.i.d.}}{\sim}P_z.
		\end{equation*}
		\State $\gradest[t] \leftarrow 
		\getgradientfunc_{\newestparams, \bias}\prn*{x\ind{t}, x\ind{t-1}, g\ind{t-1}}$.\label{line:gradient_estimation_cubic}  
		\State\label{line:constrained_cubic_step}Set the next point $x\ind{t+1}$ as 
	\begin{equation*}
			x\ind{t+1} \gets \argmin_{y:\norm{y-x\ind{t}}\le\eta} 
		\tri[\big]{\gradest[t], y - x\ind{t}} + \frac{1}{2}\tri[\big]{y-x\ind{t}, 
		\hessest[t](y - x\ind{t})} + \frac{M}{6} 
		\nrm[\big]{y-x\ind{t}}^{3}.
	\end{equation*}
		\EndFor
		\State \textbf{return} $\hx$ chosen uniformly at random from 
		$\crl*{x\ind{t}}_{t=2}^{T + 1}.$ 
\end{algorithmic} 
\end{algorithm}

}

\begin{proof}[Proof of  \pref{thm:epsilon_cubic}]In the following, we first show that \pref{alg:epsilon_cubic}
  returns a point $\hat{x}$, such that with probability at least $\frac{7}{8}$, $ \nrm*{ \nabla F(\hat{x})}  \leq 350 \epsilon$. We then bound, with probability at least $\frac{7}{8}$, the total number of oracle queries made up until time $T$.

Note that, using \pref{lem:gradient_estimator_general_lemma} 
 and \pref{lem:hess_error_gen_lemma}, we have for all $t \geq 0$, 
\begin{align*} 
 \En \brk*{ \nrm{\nabla  F\prn*{ x \ind{t}} - \gradest[t]} } \leq \epsilon^2,  \quad \text{and }  \quad \En \brk*{ \nrm{ \nabla^2  F\prn*{ x \ind{t}} - \hessest[t]}_\op } \leq \frac{\epsilon^2}{\eta^2}.  \numberthis \label{eq:epsilon_cubic_t_var_bound}   
\end{align*} 
Thus, for each $t\geq{}1$, invoking  \pref{lem:epsilon_cubic_descent_epsilon_step} and plugging in the bounds from \pref{eq:epsilon_cubic_t_var_bound}, and using the value of $\eta$, we get
\begin{align*}
      \En \brk*{F\prn{x\ind{t}} -  F\prn{x\ind{t+1}}} &\geq  \frac{M \eta^3}{60}  \Pr \Big( \nrm*{\derF{x\ind{t+1}}{}} \geq \frac{M\eta^2}{2} \Big) - \frac{14 \epsilon^{\frac{3}{2}}}{\sqrt{M}} \\ 
      &\geq  \frac{240 \epsilon^\frac{3}{2}}{\sqrt{M}}  \prn*{ \Pr \Big( \nrm*{\derF{x\ind{t+1}}{}} \geq 350 \epsilon \Big) - \frac{1}{16}}.
 \end{align*} 
Telescoping this inequality from $t=1$ to $T$, we have that
\begin{align*}
      \En \brk*{F\prn{x\ind{1}} -  F\prn{x\ind{T+1}}} &\geq \frac{240  \epsilon^\frac{3}{2}}{\sqrt{M}}  \cdot T \cdot \prn*{\frac{1}{T} \sum_{t=1}^T \Pr \Big( \nrm*{\derF{x\ind{t+1}}{}} \geq 350 \epsilon \Big) - \frac{1}{16}} \\
      &= \frac{240  \epsilon^\frac{3}{2}}{\sqrt{M}}  \cdot T \cdot \prn*{ \Pr \prn*{  \nrm*{\derF{\hx}{}} \geq 350 \epsilon}  - \frac{1}{16}}, 
 \end{align*} where the equality follows because $\hx$ is sampled uniformly at  random from the set $\crl*{x\ind{t}}_{t=2}^{T+1}$.  Next, using the fact that, $\func{ x \ind{t}}{} -  F\prn*{x\ind{T+1}} \leq  \Delta$, rearranging the terms, and plugging in the value of $T$, we get
 \begin{align*}
    \Pr \prn*{  \nrm*{\derF{\hx}{}} \geq 350 \epsilon} \leq \frac{\Delta \sqrt{M}}{240 \epsilon^\frac{3}{2} T} + \frac{1}{16} \leq \frac{1}{8}.
\end{align*}
Thus, with probability at least $\frac{7}{8}$, 
\begin{align*}
   \nrm*{\derF{\hx}{}} \leq 350 \epsilon. \numberthis \label{eq:gradient_bound_eps_cubic}
\end{align*}

\paragraph{Bound on the number of oracle queries.}\pref{alg:epsilon_cubic} 
queries the stochastic oracle in  \pref{line:Hessian_estimator} and 
\pref{line:gradient_estimation_cubic} only to compute the respective 
Hessian and gradient estimates.  Let $M_h$ and $M_g$ denote the total 
number of stochastic oracle queries made by  \pref{line:Hessian_estimator} 
and  \pref{line:gradient_estimation_cubic} till time $T$ respectively.  
Further, Let $M=M_h + M_g$ denote the total number of oracle queries made till time 
$T$.

In what follows, we first bound $\En \brk*{M_h}$ and $\En \brk*{M_g}$. Then, we invoke Markov's inequality to deduce that the desired bound on $M$ holds with probability at least $\frac{7}{8}$. 
\begin{enumerate} 
\setlength{\itemindent}{-.13in} 
\item \textbf{Bound on $\En \brk*{M_h}$.} Since the algorithm queries the stochastic Hessian oracle $n_H$ times per iteration, $M_h = T \cdot n_H$. Plugging the values of $T$, $n_H$ and $M$ as specified in \pref{alg:epsilon_cubic}, and ignoring multiplicative constant, we get, 
\begin{align*} 
   \En \brk*{M_h} &= \ceil*{ \frac{5 \Delta \sqrt{M}}{3 \epsilon^{1.5}}} \cdot \ceil*{\frac{22 \sigma_2^2 \eta^2 \log(d)}{\epsilon^2}} \\ 
   &\leq O\prn*{\frac{\Delta \sqrt{M}}{\epsilon^{1.5}} + \frac{\Delta \sigma_2^2 \log(d)}{\epsilon^{2.5} \sqrt{M}}}  \\ &\leq O\prn*{\frac{\Delta \sqrt{L_2}}{\epsilon^{1.5}} + \frac{\Delta \sigma_2}{\epsilon^{2}} + \frac{\Delta \sigma_1 \sigma_2 \sqrt{\log(d)}}{\epsilon^{3}} },
       \numberthis \label{eq:cubic_compexity_bound2}
\end{align*} 
 where the first inequality above follows from the fact that $\frac{\Delta \sqrt{M}}{\epsilon^{1.5}} \geq 1$  under the natural choice for the precision parameter $\epsilon \leq \Delta^{\frac{2}{3}}  M^{\frac{1}{3}}$ and using the identity $\ceil*{x} \leq x + 1$ for $x \geq 0$. 
 
 \item \textbf{Bound on $\En \brk*{M_g}$.} Invoking \pref{lem:gradient_estimator_oracle_complexity_meta_lemma} for each $t \geq 1$, we get
\begin{align*} 
\En \brk*{M_g} &= 6 \sum_{t=1}^T \En \brk*{ \frac{\bias \sigma_1^2}{\epsilon^2} + \frac{(\sigma_2^2 + L_2 \epsilon) \cdot \nrm*{x\ind{t} - x\ind{t - 1}}^2}{b \epsilon^2} + 1 } \\
						&\overeq {\proman{1}} O\prn*{ T \cdot  \prn*{ \frac{\bias \sigma_1^2}{\epsilon^2} + \frac{(\sigma_2^2 + L_2 \epsilon) \cdot \eta^2}{b \epsilon^2} + 1 }} \\ 
						&\overeq{\proman{2}} O\prn*{ \frac{\Delta}{\eta \epsilon} \cdot  \prn*{ \frac{\bias \sigma_1^2}{\epsilon^2} + \frac{(\sigma_2^2 + L_2 \epsilon) \cdot \eta^2}{b \epsilon^2} + 1 }} \numberthis \label{eq:cubic_compexity_bound1_eta}
\end{align*} where $\proman{1}$ follows by observing $\nrm*{x\ind{t} - x\ind{t-1}} \leq \eta$ due to the update rule in \pref{line:constrained_cubic_step} and $\proman{2}$ is given by plugging in the value of $T  \leq O(\frac{\Delta}{\eta \epsilon})$ for the natural choice of parameter $\epsilon = O(\Delta^{\frac{2}{3}} M^{\frac{1}{3}})$. Next, note that since $M > L_2$, and since we assume $\epsilon < \sigma_1$, the parameter $b$ is equal to $\frac{\eta \sqrt{\sigma_2^2 + \epsilon L_2}}{25 \sigma_1}$ (which is smaller than $1$). Thus, plugging the value of $b$ and $\eta$ in the bound \pref{eq:cubic_compexity_bound1_eta}, we get
\begin{align*}
   	\En \brk*{M_g} &= O\prn*{ \frac{\Delta \sigma_1 \sqrt{\sigma_2^2 + \epsilon L_2}}{\epsilon^3} + \frac{\Delta \sqrt{M}}{\epsilon^{1.5}}} \\ 
			            &= O\prn*{ \frac{\Delta \sigma_1 \sigma_2}{\epsilon^3} + \frac{\Delta \sigma_1 \sqrt{L_2}}{\epsilon^{2.5}} +  \frac{\Delta \sigma_2}{\epsilon^2}\sqrt{\log(d)} + \frac{\Delta \sqrt{L_2}}{\epsilon^{1.5}}}, \numberthis \label{eq:cubic_compexity_bound1}
\end{align*}
where the second equality follows by using that $\eps\leq{}\sigma_1$ to simplify the term $\frac{\Delta \sqrt{M}}{\epsilon^{1.5}}$.
\end{enumerate} 

\noindent
Adding \pref{eq:cubic_compexity_bound1}  and \pref{eq:cubic_compexity_bound2}, the total number of oracle queries made by \pref{alg:epsilon_cubic} till time $T$ is bounded, in expectation, by
\begin{align*}
  \En \brk*{M} = \En \brk*{M_g + M_h} =  O\prn*{\frac{\Delta \sigma_1 \sigma_2}{\epsilon^3} \sqrt{\log(d)} + \frac{\Delta \sigma_1 \sqrt{L_2}}{\epsilon^{2.5}} +  \frac{\Delta \sigma_2 }{\epsilon^{2}}\sqrt{\log(d)}  +  \frac{\Delta \sqrt{L_2}}{\epsilon^{1.5}}}. 
\end{align*} 
Using Markov's inequality, we get that, with probability at least $\frac{7}{8}$, 
\begin{align*}
   M \leq  O\prn*{\frac{\Delta \sigma_1 \sigma_2}{\epsilon^3} \sqrt{\log(d)} + \frac{\Delta \sigma_1 \sqrt{L_2}}{\epsilon^{2.5}} +  \frac{\Delta \sigma_2 }{\epsilon^{2}}\sqrt{\log(d)} +  \frac{\Delta \sqrt{L_2}}{\epsilon^{1.5}}}.  \numberthis \label{eq:oracle_complexity_bound_eps_cubic}
\end{align*}
The final statement follows by taking a union bound for the failure probability of \pref{eq:gradient_bound_eps_cubic} and \pref{eq:oracle_complexity_bound_eps_cubic}.  

\end{proof}

 \section{Upper bounds for finding $(\epsilon, 
 \gamma)$-second-order-stationary points} \label{app:upper_eg}
 
 \newcommand{\biasg}{\bias_g}
 \newcommand{\biash}{\bias_H}
 
\subsection{Full statement and proof for Algorithm \ref*{alg:epsilon_gamma_hvp}}
 \label{app:epsilon_gamma_hvp} 
 \begin{algorithm}[H]
	\caption{Stochastic gradient descent with negative curvature search and \mainalg} 
	\label{alg:epsilon_gamma_hvp}
	\begin{algorithmic}[1]
		\Require Oracle $(\estimator[2], P_z) \in \stocpOracleas$ for $F \in 
		\cF_2\prn*{\Delta, L_1, L_2}$. Precision parameters $\epsilon, 
		\gamma$.		
		\State \label{line:parameters_eps_gamma_hvp}  Set $\eta = \min\crl*{\frac{\gamma}{\epsilon L_2}, 
		\frac{1}{2 \sqrt{ L_1^2 + \bsigma_2^2 + \epsilon L_2}}}$,    $T 
		= \ceil*{\frac{20 \Delta L_2^2}{\gamma^3 } + \frac{2\Delta}{\eta 
		\epsilon^2}}$, $p = \frac{\gamma^3}{\gamma^3 + 10 \Delta L_2^2 
		\eta \epsilon^2}$, $\delta=\frac{\gamma}{40^2L_2}$.
		\State Set $\biasg = \min\{1, 
		\frac{\eta \epsilon \sqrt{\bsigma_2^2 + 
		\epsilon L_2}}{\sigma_1 }\}$ and $\biash = \min\{1,  
		\frac{\gamma \sqrt{\bsigma_2^2 + 
		\epsilon L_2}}{\sigma_1 
		L_2}\}$.
		\State Initialize $x\ind{0}, x\ind{1} \leftarrow 0$, $g\ind{1}\leftarrow 
		\getgradientfunc_{\newestparams,\biasg}(x\ind{1}, x\ind{0}, \bot)$. 
		\For {$t = 1$ to $T$} 
		\State Sample $\Q_t\sim$ Bernoulli($p$).
		\If {$\Q_t = 1$}{} 
				\State $x \ind{t+1} \leftarrow x \ind{t} - \eta \cdot \gradest[t]$. 
				\label{line:gradient_step_hvp_alg_eps_gamma} 
				\State $\gradest[t+1] \leftarrow 
				\getgradientfunc_{\newestparams,\biasg}(x \ind{t+1}, x\ind{t}, 
				\gradest[t])$.  \label{line:gradient_step_update_gradest}  
		\Else
		\State $u\ind{t} \leftarrow \ojaF\prn*{x\ind{t}, \estimator[2], 
		2\gamma, \delta}$. \label{line:negative_curvature_search}   
              \hfill\algcomment{Oja's algorithm (\pref{lem:ojas_algorithm}).}
		\If{$ u\ind{t} \equiv \bot $} 
		\State $x\ind{t+1} \leftarrow x\ind{t}$.    \label{line:u_t_bot} 
		\State $\gradest[t+1]\gets \gradest[t]$.
		\Else  \label{line:if_statement_eps_gamma_hvp} 
			\State Sample $r\ind{t} \sim \text{Uniform}\prn*{\crl*{-1, 1}}$.
			\State  $x\ind{t+1} \leftarrow x\ind{t} + 
			\frac{\gamma}{L_2} \cdot r\ind{t} \cdot  u\ind{t}$. 
			\label{line:negative_curvature_update}
			\State $\gradest[t+1] \leftarrow 
			\getgradientfunc_{\newestparams,\biash}(x \ind{t+1 }, x\ind{t}, 
			\gradest[t])$.  \label{line:negative_curvature_update_gradest}  
		\EndIf 
		\EndIf 
		\EndFor
		\State \textbf{return} $\hx$ chosen uniformly at random from $\prn*{x\ind{t}}_{t = 1}^{T}$. 
\end{algorithmic} 
\end{algorithm} 
\begin{proof}[Proof of \pref{thm:epsilon_gamma_hvp}]
We first show that \pref{alg:epsilon_gamma_hvp} returns a point $\hx$
such that, $\En \brk*{ \nrm*{ \nabla F(\hx)}}  \leq 8 \epsilon$ and
$\eigmin\prn*{\derF{\hx}{2}} \geq -4\gamma$. We then bound the
expected number of oracle queries used throughout the execution.

To begin, note that, for any $t \geq 1$, there are two scenarios: (a) either $\Q_t = 1$ and $x\ind{t+1}$ is set using the update rule in \pref{line:gradient_step_hvp_alg_eps_gamma}, or, (b) $\Q_t = 0$ and we set $x\ind{t+1}$ using \pref{line:negative_curvature_search}, respectively. We analyze the two cases separately below.  
\paragraph{Case 1: $\Q_t = 1$.}  Since, $\eta \leq \frac{1}{2\sqrt{L_1^2 + \bsigma_2^2 + \tepsilon L_2}} \leq \frac{1}{2L_1}$ and $F$ has $\Lip{1}$-Lipschitz gradient, using \pref{lem:general_lemma_gradient_descent}, we have
\begin{align*} 
 F \prn{x\ind{t}} - F \prn{x\ind{t + 1}}  &\geq  \frac{\eta}{8} { \nrm*{ \nabla F \prn{x\ind{t}}}^2 } -  \frac{3\eta}{4} { \nrm*{\nabla F \prn{x\ind{t}} - \gradest[t]}^2}.
\end{align*}
Taking expectation on both the sides, while conditioning on the event that $\Q_t = 1$, we get
\begin{align*}
  \En \brk*{ \func{x\ind{t}} - \func{x\ind{t+1}} \mid{} \Q_t = 1} &\geq \frac{\eta}{8} \En \brk*{\nrm*{ \nabla F \prn{x\ind{t}}}^2}-  \frac{3\eta}{4} \En \brk*{ \nrm*{\nabla F \prn{x\ind{t}} - \gradest[t]}^2}   \\
  &\geq \frac{\eta}{8} \En \brk*{\nrm*{ \nabla F \prn{x\ind{t}}}^2} -  \frac{3\eta \epsilon^2}{4},  
  \numberthis \label{eq:eps_gamma_ncd_q1}
\end{align*}
where the last inequality follows using \pref{lem:gradient_estimator_general_lemma}. 
\paragraph{Case 2: $\Q_t = 0$.} Let $\Eoja(t)$ denote the event that $\oja$ succeeds at time $t$, in the sense that the event in \pref{lem:ojas_algorithm} holds: $(i)$ if $u \ind{t} = \bot$ then $\grad^2 F(x\ind{t}) \succeq -2\gamma I $, and $(ii)$ otherwise, $u\ind{t}$ satisfies $\tri{u \ind{t}, \grad^2 \func{x \ind{t}} u \ind{t}} \leq  -\gamma$. 

Then, using \pref{lem:oja_nc_descent}, we are guaranteed that
  \[
        \En\brk*{\func{x\ind{t}} - \func{x\ind{t+1}}\mid{}Q_t=0}
        \geq   \frac{5\gamma^3}{6L_2^2}\prn*{ \Pr\prn*{\eigmin(\nabla^2 F(x\ind{t})) \leq - 4\gamma} -\frac{2L_1}{\gamma}\Pr\prn*{\lnot\Eoja(t)\mid{}Q_t=0}}.
      \]
      In particular, we are guaranteed by \pref{lem:ojas_algorithm} that
      \begin{equation}
        \En\brk*{\func{x\ind{t}} - \func{x\ind{t+1}}\mid{}Q_t=0}
        \geq   \frac{5\gamma^3}{6L_2^2}\prn*{ \Pr\prn*{\eigmin(\nabla^2 F(x\ind{t})) \leq - 4\gamma} -\frac{2L_1}{\gamma}\delta}.
        \label{eq:eps_gamma_ncd_q0}
      \end{equation}

\noindent{}Combining the two cases ($\Q_t = 0$ and $\Q_t = 1$) from \pref{eq:eps_gamma_ncd_q1} and \pref{eq:eps_gamma_ncd_q0} above, we get
\begin{align}
&   \En \brk*{\func{x\ind{t}} - \func{x \ind{t+1}}}  \\&=  \sum_{q \in \crl*{0, 1}} \Pr \prn{\Q_t = q} \En \brk*{\func{x\ind{t}} - \func{x \ind{t+1}} \mid{} \Q_t = q} \notag \\ 
   		&\geq \frac{5(1-p)\gamma^3}{6L_2^2} \prn*{\Pr\prn*{\eigmin(\nabla^2 F(x\ind{t})) \leq - 4 \gamma}-\frac{2L_1}{\gamma}\delta}   + p \prn*{   		 \frac{\eta}{8} \En \brk*{\nrm*{ \nabla F \prn{x\ind{t}}}^2} -  \frac{3\eta \epsilon^2}{4}}    \label{eq:eps_gamma_ncd_q3}. 
\end{align}
Using that $ \En \brk*{\nrm*{ \nabla F \prn{x\ind{t}}}^2} \geq (8\epsilon)^2 \cdot \Pr \prn*{\nrm*{ \nabla F \prn{x\ind{t}}} \geq 8 \epsilon}$ and that $\delta\leq{}\frac{\gamma}{1600L_1}$, we have
\begin{align*}
&  \En \brk*{\func{x\ind{t}} - \func{x \ind{t+1}}}	\\&\geq \frac{5 (1 - p) \gamma^3}{6L_2^2} \prn*{\Pr\prn*{\eigmin(\nabla^2 F(x\ind{t})) \leq - 4 \gamma}-\frac{1}{800}} + 8 p \eta \epsilon^2  \prn*{ \Pr \prn*{\nrm*{ \nabla F \prn{x\ind{t}}} \geq 8 \epsilon}  -  \frac{3}{32}}.  
\end{align*}
Telescoping this inequality for $t$ from $1$ to $T$ and using the bound $\En \brk*{\func{x\ind{1}} - \func{x\ind{T + 1}}} \leq \Delta$, we get
\begin{align*}
  \Delta &\geq \En \brk*{\func{x\ind{1}} - \func{x\ind{T + 1}}} \\ 
         &\geq  \frac{5 T (1- p) \gamma^3}{6 L_2^2} \Big( \frac{1}{T} \sum_{t=0}^{T-1}  \Pr\prn*{\eigmin(\nabla^2 F(x\ind{t})) \leq - 4 \gamma} -\frac{1}{800}\Big)  \\
  &~~~~+ 8Tp\eta \epsilon^2 \Big( \frac{1}{T} \sum_{t=0}^{T-1} \Pr \prn*{\nrm*{\nabla \func{x \ind{t}}} \geq 8 \epsilon} - \frac{3}{32} \Big) \\ 
 &\overgeq{\proman{1}}  \frac{5 T (1- p) \gamma^3}{6 L_2^2}   \prn*{\Pr\prn*{\eigmin(\nabla^2 F(\hx)) \leq - 4 \gamma}-\frac{1}{800}}  + 8Tp\eta \epsilon^2 \Big(\Pr \prn*{\nrm*{\nabla \func{\hx}} \geq 8 \epsilon} - \frac{3}{32} \Big) \\ 
  &\overgeq{\proman{2}} 16 \Delta \prn*{ \Pr\prn*{\eigmin(\nabla^2 F(\hx)) \leq - 4 \gamma}  +  \Pr \prn*{\nrm*{\nabla \func{\hx}} \geq 8 \epsilon}   - \frac{1}{4}},   \numberthis \label{eq:eps_gamma_ncd_q2} 
\end{align*} where $\proman{1}$ follows because $\hx$ is sampled uniformly at random from $\prn*{x\ind{t}}_{t=1}^{T}$  and $\proman{2}$ follows from \pref{lem:TP_eg_hvp}. Rearranging the terms, we get 
\begin{align*}
    \Pr\prn*{\eigmin(\nabla^2 F(\hx)) \leq - 4 \gamma}  +  \Pr \prn*{\nrm*{\nabla \func{\hx}} \geq 8 \epsilon}  \leq \frac{5}{16}, 
\intertext{which further implies that} 
    \Pr\prn*{\eigmin(\nabla^2 F(\hx)) \geq - 4 \gamma \wedge \nrm*{\nabla \func{\hx}} \leq 8 \epsilon}  \geq \frac{11}{16}. \numberthis \label{eq:eps_gamma_hvp@high_probbaility_guarantee2} 
\end{align*} 

\paragraph{Bound on the number of oracle queries.}  At every iteration, 
\pref{alg:epsilon_gamma_hvp} queries the stochastic oracle in either 
\pref{line:gradient_step_update_gradest} or
\pref{line:negative_curvature_update_gradest} (to compute the stochastic 
gradient estimator and to execute Oja's algorithm, respectively), and possibly
\pref{line:negative_curvature_search} (to update the gradient
estimator after a negative curvature step).  Let $m_g(t)$ denote 
the total number of stochastic oracle queries made by  
\pref{line:gradient_step_update_gradest} or  
\pref{line:negative_curvature_update_gradest} at time $t$, and let 
$M_{\mathrm{g}}=\sum_{t=1}^{T} m_g(t)$. Further, let $M_{\text{nc}}$ denote the 
total 
number of oracle calls made by \pref{line:negative_curvature_search}, and 
further let $M = M_{\mathrm{g}} + M_{\text{nc}}$ be the total number of oracle queries 
made up 
until time $T$. 

In what follows, we first bound $\En \brk*{M_{\mathrm{g}}}$ and $\En 
\brk*{M_{\mathrm{nc}}}$. Then, we invoke Markov's inequality to bound $M$ 
with probability at least $\frac{19}{20}$.

\paragraph{Bound on $M_{\mathrm{g}}$.}  For any $t > 0$, there are two
scenarios, either (a) $\Q_{t} = 1$ and we go through
\pref{line:gradient_step_hvp_alg_eps_gamma}, or (b) $\Q_{t} = 0$ and
\pref{line:negative_curvature_update_gradest} is executed.  Thus, 
\begin{align*}
 \En \brk*{M_{\mathrm{g}}} &= {\sum_{t=1}^{T} \Pr \prn*{\Q_t = 0} \En \brk*{m_g(t) \mid \Q_t = 0}}
   + {\sum_{t=1}^{T} \Pr \prn*{\Q_t = 1} \En \brk*{m_g(t) \mid \Q_t = 1}} \numberthis \label{eq:eps_gamma_hvp_oc0}
\end{align*}
We denote the two terms on the right hand side above by $\markedterm{A}$ and $\markedterm{B}$, respectively. We bound them separately as follows.
\begin{enumerate}[label=$\bullet$]
\setlength{\itemindent}{-.13in}
\item \textbf{Bound on $\markedterm{A}$.} Using \pref{lem:gradient_estimator_oracle_complexity_meta_lemma} with the fact that $\Pr \prn*{\Q_{t} = 0}  = 1 - p$, we get 
\begin{align*}
\markedterm{A}  &
\overeq{}  O(1) \sum_{t=1}^{T} (1 - p) \cdot \En \brk*{ 
	\biash \frac{\sigma_1^2}{\epsilon^2} + {\nrm[\big]{x\ind{t+1} - 
			x\ind{t}}^2 \cdot \frac{  
			{\bsigma_2^2 + \epsilon L_2}}{\biash\epsilon^2}} + 1
	   \lmid \Q_{t} = 0} \\ 
					   &\overeq{\proman{1}} O\prn*{ T \cdot (1 - p) \cdot \prn*{ 
					   \frac{ \gamma \sigma_1 \sqrt{\bsigma_2^2 + \epsilon 
					   L_2}}{L_2 \epsilon^2} + \frac{\gamma^2}{\epsilon^2} \cdot 
					   \frac{\bsigma^2 + \epsilon L_2}{L_2^2}+1   }} \\
					   &\overleq{\proman{2}} O\prn*{ \frac{\Delta L_2 \sigma_1 
					   \sqrt{\bsigma^2_2 + \epsilon L_2}}{\gamma^2 \epsilon^2}  
					   + \frac{\Delta (\bsigma_2^2 + \epsilon L_2)}{\gamma 
					   \epsilon^2} +  \frac{\Delta L_2^2}{\gamma^3}},  
					   \numberthis \label{eq:eps_gamma_hvp_oc1}
\end{align*} where $\proman{1}$ is given by plugging in 
$\norm{x\ind{t}-x\ind{t-1}}=\gamma/L_2$. The inequality $\proman{2}$ follows by 
using the bound on $T \cdot (1 - p)$ from \pref{lem:TP_eg_hvp}. 

\item \textbf{Bound on $\markedterm{B}$.} Using \pref{lem:gradient_estimator_oracle_complexity_meta_lemma} with the fact that $\Pr \prn*{\Q_{t} = 1}  = p$, we get 
\begin{align*}
   \markedterm{B} &\overeq{}  O(1) \sum_{t=1}^{T} p \cdot \En \brk*{ 
   		\biasg \frac{\sigma_1^2}{\epsilon^2} + {\nrm[\big]{x\ind{t+1} - 
   			x\ind{t}}^2 \cdot \frac{  
   			{\bsigma_2^2 + \epsilon L_2}}{\biasg\epsilon^2}} + 1
   	\lmid \Q_{t} = 1} \\ 
   		&\overeq{\proman{1}} O(1)\sum_{t=1}^{T} p \cdot \En \brk*{ 
   		\biasg \frac{\sigma_1^2}{\epsilon^2} +
   		{\nrm[\big]{\eta \gradest[t]}^2 \cdot \frac{  
   				{\bsigma_2^2 + \epsilon L_2}}{\biasg\epsilon^2}} + 1
   			\lmid \Q_{t} = 1}  \\
   		& \overleq{\proman{2}} O\Bigg(\frac{\Delta}{\eta \epsilon^2} \cdot 
   		\Big( \En \brk*{ \frac{1}{T}  \sum_{t=1}^{T} \nrm*{ \gradest[t]}^2}   
   		\cdot \frac{  
   			{\eta^2(\bsigma_2^2 + \epsilon L_2)}}{\biasg\epsilon^2} + \biasg 
   			\frac{\sigma_1^2}{\epsilon^2} + 1\Big) \Bigg) \\
   		&\overeq{\proman{3}}  O \Big( \frac{  \Delta \sigma_1
   		\sqrt{\bsigma_2^2 + 
   		\epsilon L_2}}{\epsilon^3} + \frac{\Delta}{\eta \epsilon^2} \Big), 
   		\numberthis \label{eq:eps_gamma_hvp_oc2}
\end{align*}
where $\proman{1}$ follows by plugging in the update rule from 
\pref{line:gradient_step_hvp_alg_eps_gamma} (when $\Q_t = 1$), 
$\proman{2}$ follows by rearranging the terms and using the bound on $T 
\cdot p$ from \pref{lem:TP_eg_hvp}, and $\proman{3}$ is follows 
from the choices of $\biasg$ (in particular, our assumption that
$\eps\leq{}\sigma_1$ implies that $\biasg=\frac{\eta \epsilon \sqrt{\bsigma_2^2 + 
		\epsilon L_2}}{\sigma_1 }$) and $\eta$, as well as the following bound for
 $\En \brk*{\frac{1}{T}  
\sum_{t=1}^T \nrm*{\gradest[t]}^2}$:
\begin{align*}
   \En \brk*{\frac{1}{T}  \sum_{t=1}^T \nrm{\gradest[t]}^2} &\leq \En 
   \brk*{\frac{2}{T}  \sum_{t=1}^T \nrm*{\gradest[t] - \derF{x\ind{t}}{}}^2 + 
   \frac{2}{T}  \sum_{t=1}^T \nrm*{ \derF{x\ind{t} }{} }^2}  \\  
   	  &\leq  O\prn*{\epsilon^2 + \norm{\grad F(\hx{})}^2} \le O(\epsilon^2), 
\end{align*} 
where the last inequality is uses 
\pref{lem:gradient_estimator_general_lemma} and 
\pref{lem:techincal_lemma_on_Hx_hvp_eg}. 
\end{enumerate} 
Combining the bounds from \pref{eq:eps_gamma_hvp_oc1} and \pref{eq:eps_gamma_hvp_oc2} in \pref{eq:eps_gamma_hvp_oc0}, we have
\begin{align*}
   \En \brk*{M_{\mathrm{g}}} \leq O\prn*{ \frac{\Delta L_2 \sigma_1 
   \sqrt{\bsigma^2_2 + \epsilon L_2}}{\gamma^2 \epsilon^2} + 
   	\frac{\Delta (\bsigma_2^2 + \epsilon L_2)}{\gamma \epsilon^2}  +
   	 \frac{\Delta L_2^2}{\gamma^3} + \frac{  \Delta \sigma_1 
   	 \sqrt{\bsigma_2^2 + \epsilon L_2}}{\epsilon^3} + \frac{\Delta}{\eta 
   	 \epsilon^2}} \numberthis \label{eq:eps_gamma_hvp_oc3}. 
\end{align*}

 \paragraph{Bound on $M_{\mathrm{nc}}$.} Using the law of total probability 
 with the observation that \pref{alg:epsilon_gamma_hvp} enters 
 \pref{line:negative_curvature_search} only if $\Q_t = 0$,  we get 
\begin{align*} 
   \En \brk*{\sum_{t=1}^T m_{\mathrm{nc}}(t)} &=  \sum_{t=1}^T \sum_{q \in 
   \crl*{0, 1}}\Pr \prn*{\Q_t = q} \En \brk*{ m_{\mathrm{nc}}(t) \mid \Q_t = q}  
\\ 
 &= \sum_{t=1}^T \Pr \prn*{\Q_t = 0} \En 
 \brk*{ m_{\mathrm{nc}}(t) \mid \Q_t = 0}  \\ 
 &=  T \cdot (1 - p) \cdot n_H \leq O\prn*{ \frac{\Delta L_2^2}{\gamma^3}  
 \cdot n_H},  \numberthis \label{eq:eps_gamma_hvp_oc4}
\end{align*} where $n_H$ denotes the number of oracle queries made by 
$\ojaF$, the last inequality follows by bounding $T \cdot (1-p)$ as in 
\pref{eq:eps_gamma_ncd_q2}. Note that \pref{lem:ojas_algorithm} implies 
that for $\delta = \frac{\gamma}{1600L_1}$,
\begin{align*}
n_H \leq O\prn[\Bigg]{\frac{\prn*{\bar{\sigma}_2 + L_1}^2}{\gamma^2} \log^2\prn*{\frac{L_1}{\gamma}d}}.    \numberthis \label{eq:eps_gamma_hvp_oc5}
\end{align*}
Combining the above bounds for $M_{\mathrm{g}}$ and $M_{\mathrm{nc}}$ (in 
\pref{eq:eps_gamma_hvp_oc3} and \pref{eq:eps_gamma_hvp_oc4} 
respectively), we get 
\begin{align*}
   \En \brk*{M}  &\leq 20 \En \brk*{ M_{\mathrm{g}} + M_{\mathrm{nc}} } \\
   									&= O\prn*{ \frac{\Delta L_2 \sigma_1 
   									\sqrt{\bsigma^2_2 + \epsilon L_2}}{\gamma^2 
   									\epsilon^2} +  
   										\frac{\Delta (\bsigma_2^2 + \epsilon 
   										L_2)}{\gamma \epsilon^2}  +
   										\frac{\Delta L_2^2}{\gamma^3} + \frac{\Delta 
   										\sigma_1 \sqrt{\bsigma_2^2 + \epsilon 
   										L_2}}{\epsilon^3} + \frac{\Delta}{\eta \epsilon^2} 
   										+ \frac{\Delta L_2^2}{\gamma^3} \cdot n_H}.
 \end{align*} 
Plugging in the value of $\eta$ from \pref{alg:epsilon_gamma_hvp} and $n_H$ from \pref{eq:eps_gamma_hvp_oc5}, and using Markov's inequality, we get that, with probability at least $\frac{15}{16}$,  
\begin{align*}
M  &= O\Bigg(\frac{  \Delta \sigma_1 \sqrt{\bsigma_2^2 + \epsilon 
L_2}}{\epsilon^3}  +  \frac{\Delta L_2 \prn*{\sigma_1\bsigma_2 + 
\sqrt{\epsilon L_2} + \gamma \bsigma_2^2 /L_2 +\gamma 
\epsilon}}{\gamma^2 \epsilon^2} +  \frac{\Delta L^2_2}{\gamma^3} 
\prn*{\frac{\prn*{\bar{\sigma}_2 + L_1}^2}{\gamma^2} 
\log^2\prn*{\frac{L_1}{\gamma}d}} \\
   &~~~~+  O\Bigg(\frac{\Delta L_2^2}{\gamma^3} + \frac{\Delta \sqrt{L_1^2 + \bsigma_2^2 + \epsilon L_2}}{\epsilon^2} \Bigg). \numberthis \label{eq:eps_gamm_full_lc_bound} 
\end{align*}    	
Ignoring the lower-order terms, we have 
\begin{align*}
M  &= \wt{O} \prn[\Bigg]{ \frac{  \Delta \sigma_1 \bsigma_2}{\epsilon^3}  +  \frac{\Delta L_2 \sigma_1\bsigma_2}{\gamma^2 \epsilon^2} +  \frac{\Delta L^2_2 \prn*{\bsigma_2 + L_1}^2}{\gamma^5} }. 
\end{align*} 
The final statement follows by taking a union bound for the failure probability of the claims in \pref{eq:eps_gamma_hvp@high_probbaility_guarantee2} and  \pref{eq:eps_gamm_full_lc_bound}.
\end{proof}

\begin{lemma}
  \label{lem:oja_nc_descent}
  Under the setting of \pref{thm:epsilon_gamma_hvp}, we are guaranteed that
   \colt{ \begin{align*}
      &\En\brk*{\func{x\ind{t}} - \func{x\ind{t+1}}\mid{}Q_t=0}\\
      &\geq   \frac{5\gamma^3}{6L_2^2}\prn*{ \Pr\prn*{\eigmin(\nabla^2 F(x\ind{t})) \leq - 4\gamma} -\frac{2L_1}{\gamma}\Pr\prn*{\lnot\Eoja(t)\mid{}Q_t=0}}.
  \end{align*}} 
  \arxiv{ \begin{align*}
      \En\brk*{\func{x\ind{t}} - \func{x\ind{t+1}}\mid{}Q_t=0} &\geq   \frac{5\gamma^3}{6L_2^2}\prn*{ \Pr\prn*{\eigmin(\nabla^2 F(x\ind{t})) \leq - 4\gamma} -\frac{2L_1}{\gamma}\Pr\prn*{\lnot\Eoja(t)\mid{}Q_t=0}}.
  \end{align*}}
  \end{lemma}
\begin{proof}
  Recall that \pref{alg:epsilon_gamma_hvp} calls $\ojaF$ with the precision parameter $2\gamma$. To begin, suppose that $\Eoja(t)$ holds. Then if $\ojaF$ returns $\bot$, then $\eigmin\prn*{\derF{x\ind{t}}{2}} \geq -4\gamma$, otherwise $\ojaF$ returns a unit vector  $u\ind{t}$ such that $\derF{x\ind{t}}{2}[u\ind{t},u\ind{t}] \leq -2\gamma$.  Thus, using \pref{lem:general_lem_curvature_descent} with $H = \derF{x\ind{t}}{2}$ and $u\ind{t}$, we conclude that---conditioned on the history up to time $t$, and on $Q_t=0$---we have
  \begin{align*}
    \indic\crl{\Eoja(t)}\prn{\func{x\ind{t}} - \func{x\ind{t+1}}} 
    &\geq   \frac{5\gamma^3}{6L_2^2} \indic\crl{\eigmin(\nabla^2 F(x\ind{t})) \leq - 4 \gamma\wedge{}\Eoja(t)}.
  \end{align*}
  In particular, this implies that
  \begin{align*}
    &\func{x\ind{t}} - \func{x\ind{t+1}} \\
    &\geq   \frac{5\gamma^3}{6L_2^2}\prn*{ \indic\crl{\eigmin(\nabla^2 F(x\ind{t})) \leq - 4\gamma} -\indic\crl{\lnot\Eoja(t)}}
      - \indic\crl{\lnot\Eoja(t)}\prn{\func{x\ind{t}} - \func{x\ind{t+1}}}.
  \end{align*}
  Taking conditional expectations, this further implies that
  \begin{align*}
    \En\brk*{\func{x\ind{t}} - \func{x\ind{t+1}}\mid{}Q_t=0} 
    &\geq   \frac{5\gamma^3}{6L_2^2}\prn*{ \Pr\prn*{\eigmin(\nabla^2 F(x\ind{t})) \leq - 4\gamma} -\Pr\prn*{\lnot\Eoja(t)\mid{}Q_t=0}}\\
    &~~~~- \En\brk*{\indic\crl{\lnot\Eoja(t)}\prn{\func{x\ind{t}} - \func{x\ind{t+1}}}\mid{}Q_t=0}.
  \end{align*}
  Now, consider the term
  \begin{align*}
    &    \En\brk*{\indic\crl{\lnot\Eoja(t)}\prn{\func{x\ind{t}} - \func{x\ind{t+1}}}\mid{}Q_t=0}\\
    &= \Pr(\lnot\Eoja(t)\mid{}Q_t=0)\cdot{}\En\brk*{\func{x\ind{t}} - \func{x\ind{t+1}}\mid{}Q_t=0,\lnot\Eoja(t)}.
  \end{align*}
  Given that $\oja{}$ fails, there are two cases two consider: The first case is where it returns $\bot$ (even though we may not have $\eigmin\prn*{\derF{x\ind{t}}{2}} \geq -4\gamma$), which we denote by $P_t=0$, and the second case is that it returns some vector $u\ind{t}$ (which may not actually satisfy $\grad^{2}F(x\ind{t})[u\ind{t},u\ind{t}]\leq-2\gamma$), which we denote $P_t=1$. If $P_t=0$, we have $x\ind{t+1}-x\ind{t}$, so \[\En\brk*{\func{x\ind{t}} - \func{x\ind{t+1}}\mid{}Q_t=0,\lnot\Eoja(t),P_t=0}=0.\] Otherwise, using a third-order Taylor expansion, and following the same reasoning as the proof of \pref{lem:general_lem_curvature_descent}, we have
  \begin{align*}
    &    \En\brk*{\func{x\ind{t}} - \func{x\ind{t+1}}\mid{}Q_t=0,\lnot\Eoja(t),P_t=1}\\
    &\leq{}    \En\brk*{
      \frac{\gamma^2}{2L_2^2}   \grad^2 \func{x} \brk{u\ind{t}, u\ind{t}} +  \frac{\gamma^3}{6L_2^2} \nrm{u\ind{t}}^3
      \mid{}Q_t=0,\lnot\Eoja(t),P_t=1}\\
        &\leq{}    
          \frac{\gamma^2}{2L_2^2}L_1 +  \frac{\gamma^3}{6L_2^2}
    \leq{}    
      \frac{2}{3}\frac{\gamma^2L_1}{L_2^2}.
  \end{align*}
Combining this bound with the earlier inequalities (and being rather loose with constants), we conclude that
  \begin{align*}
    &\En\brk*{\func{x\ind{t}} - \func{x\ind{t+1}}\mid{}Q_t=0} \\
    &\geq   \frac{5\gamma^3}{6L_2^2}\prn*{ \Pr\prn*{\eigmin(\nabla^2 F(x\ind{t})) \leq - 4\gamma} -\prn*{1+\frac{L_1}{\gamma}}\Pr\prn*{\lnot\Eoja(t)\mid{}Q_t=0}}\\
    &\geq   \frac{5\gamma^3}{6L_2^2}\prn*{ \Pr\prn*{\eigmin(\nabla^2 F(x\ind{t})) \leq - 4\gamma} -\frac{2L_1}{\gamma}\Pr\prn*{\lnot\Eoja(t)\mid{}Q_t=0}}.
  \end{align*}\end{proof}
  
\begin{lemma}
\label{lem:techincal_lemma_on_Hx_hvp_eg} Under the same setting as 
\pref{thm:epsilon_gamma_hvp}, the point $\hx$ returned by 
\pref{alg:epsilon_gamma_hvp} satisfies
\begin{align*}
   \En \brk*{\nrm{ \nabla F \prn{x\ind{t}}}^2} \leq 17 \epsilon^2.
\end{align*}
\end{lemma}
\begin{proof}
Starting from \pref{eq:eps_gamma_ncd_q3} in the proof of \pref{thm:epsilon_gamma_hvp}, we have
\begin{align*}
&  \En \brk*{\func{x\ind{t}} - \func{x \ind{t+1}}}  \\&\geq
\frac{5(1-p)\gamma^3}{6L_2^2} \prn*{\Pr\prn*{\eigmin(\nabla^2 F(x\ind{t})) \leq - 4 \gamma}-\frac{2L_1}{\gamma}\delta}   + p \prn*{   		 \frac{\eta}{8} \En \brk*{\nrm{ \nabla F \prn{x\ind{t}}}^2} -  \frac{3\eta \epsilon^2}{4}}.	\end{align*}
Ignoring the positive term $\Pr\prn*{\eigmin(\nabla^2 F(x\ind{t})) \leq - 4 \gamma}$ on the right hand side in the above, we get 
\begin{align*}
   \En \brk*{\func{x\ind{t}} - \func{x \ind{t+1}}}  &\geq \frac{p \eta}{8} \prn*{  \En \brk*{\nrm{ \nabla F \prn{x\ind{t}}}^2} -  6\epsilon^2} - \frac{5(1-p)\gamma^3}{3L_2^2} \frac{L_1}{\gamma}\delta.
 \end{align*}
 Telescoping this inequality for $t$ from $1$ to $T$ and using that $F(x\ind{1}) - F(x\ind{T+1}) \leq \Delta$, we get 
\begin{align*}
  \Delta &\geq \frac{T p \eta }{8} \prn*{  \En \brk*{\nrm{ \nabla F \prn{\xhat}}^2} -  6\epsilon^2} - T\frac{5(1-p)\gamma^3}{3L_2^2} \frac{L_1}{\gamma}\delta\geq \frac{\Delta}{4\epsilon^2} \prn*{  \En \brk*{\nrm{ \nabla F \prn{\xhat}}^2} -  12 \epsilon^2}
           - 70\Delta\frac{L_1}{\gamma}\delta,
\end{align*}
where the last inequality follows from \pref{lem:TP_eg_hvp}. Rearranging the terms, we get
   \[
     \En \brk*{\nrm{ \nabla F \prn{\xhat}}^2}  \leq 16 \epsilon^2 + 280\eps^{2}\cdot{}\frac{L_1}{\gamma}\delta
     \leq{} 17\eps^{2},
     \]
     where the last inequality uses that 
     $\delta\leq{}\frac{\gamma}{1600L_1}$. 
\end{proof}

\begin{lemma} 
\label{lem:TP_eg_hvp}
For the values of the parameters $T$ and $p$ specified in \pref{alg:epsilon_gamma_hvp},
\begin{align*}
   \frac{2\Delta}{\eta \epsilon^2} \leq T p \leq \frac{4\Delta}{\eta \epsilon^2}, \quad \text{and,} \quad  \frac{20 \Delta L_2^2}{\gamma^3} \leq T (1 - p) \leq  \frac{40 \Delta L_2^2}{\gamma^3}.                                                                                                                                                                                                                                            \end{align*}
\end{lemma}
\begin{proof} 
Since, $\eta \leq \frac{1}{2\sqrt{L_1^2 + \bsigma_2^2 + \epsilon L_2}} \leq \frac{1}{2L_1}$ and $\epsilon \leq \sqrt{\Delta L_1}$, we have that 
\begin{align*}
  T \geq \frac{2\Delta}{\eta \epsilon^2} &\geq \frac{4 \Delta L_1}{\epsilon^2} \geq 4. 
\end{align*}
Thus, using the fact that $x \leq \ceil*{x} \leq 2x$ for all $x \geq 1$, we get
\begin{align} 
  \frac{20 \Delta L_2^2}{\gamma^3 } + \frac{2\Delta}{\eta 
		\epsilon^2}  \leq T \leq \frac{40 \Delta L_2^2}{\gamma^3 } + \frac{4\Delta}{\eta 
		\epsilon^2} \label{eq:ceil_bound_eg_hvp}. 
\end{align} Consequently, by plugging in the values of $T$ and $p$, we have 
\begin{align*}
   T (1 - p) &= \ceil*{\frac{20 \Delta L_2^2}{\gamma^3 } + \frac{2\Delta}{\eta 
		\epsilon^2}} \cdot \prn*{ 1 - \frac{\gamma^3}{\gamma^3 + 10 \Delta L_2^2 
		\eta \epsilon^2}} \\ 
				&\leq  \prn*{\frac{40 \Delta L_2^2}{\gamma^3 } + \frac{4\Delta}{\eta 
		\epsilon^2}} \cdot \prn*{ \frac{10 \Delta L_2^2 \eta \epsilon^2}{\gamma^3 + 10 \Delta L_2^2 
		\eta \epsilon^2} }  = \frac{40 \Delta L_2^2}{\gamma^3}, 
\end{align*} where the first inequality is due to \pref{eq:ceil_bound_eg_hvp}. Similarly, we have that 
\begin{align*}
   T (1 - p) &\geq  \prn*{\frac{20 \Delta L_2^2}{\gamma^3 } + \frac{2\Delta}{\eta 
		\epsilon^2}} \cdot \prn*{ \frac{10 \Delta L_2^2 \eta \epsilon^2}{\gamma^3 + 10 \Delta L_2^2 
		\eta \epsilon^2} }  = \frac{20 \Delta L_2^2}{\gamma^3}. 
\end{align*} Together, the above two bounds imply that 
\begin{align*}
   \frac{20 \Delta L_2^2}{\gamma^3} \leq T (1 - p) \leq  \frac{40 \Delta L_2^2}{\gamma^3}. 
\end{align*} The bound on $T \cdot p$ follows similarly. 
\end{proof} 

{\subsection{Full statement and proof for Algorithm \ref*{alg:eg_cubic}}} \label{app:eg_cubic_upper}
\begin{algorithm}[H] 
   \caption{Subsampled cubic-regularized trust-region method with \mainalg} 
	\label{alg:eg_cubic}
	\begin{algorithmic}[1] 
		\Require 
		\Statex~~~~Stochastic second-order oracle $(\estimator[2], P_z) \in \stocpOracle[2]$, where $F \in \cF_2\prn*{\Delta, \infty, L_2}$.
		\Statex~~~~Precision parameter $\epsilon$. 
		\State Set $M = 4\max \crl*{L_2, \frac{\sigma_2^2 \epsilon \log(d)}{\sigma_1^2}}$, $\eta = 30 \sqrt{\frac{\epsilon}{M}}$, $T = \ceil*{\frac{18 \Delta L_2^2}{\gamma^3} + \frac{\Delta\sqrt{M}}{30 \epsilon^{{3/2}}}}$, $p = \frac{\sqrt{M} \gamma^{3/2} }{\sqrt{M} \gamma^{3/2} + 540 L_2^2\epsilon^{3/2}}$. 
		\State\label{line:parameter_setting_eg_cubic1}Set $\hessn{1} = \ceil*{ \frac{2 \cdot 10^4 \cdot \sigma_2^2 \log(d)}{\epsilon M}}$, $\hessn{2} = \ceil*{\frac{440\sigma_2^2 \log(d)}{\gamma^2}}$.
		\State Set $\biasg = \min\crl*{1, \frac{\eta \sqrt{\sigma_2^2 + \epsilon L_2}}{30 \sigma_1}}$ and $\biash = \min\crl*{1, \frac{\gamma \sqrt{\sigma_2^2 + \epsilon L_2}}{\sigma_1 L_2}}$.
		\State Initialize $x\ind{0}, x\ind{1} \leftarrow 0$, $\gradest[1] \leftarrow \getgradientfunc_{\newestparams,\biasg}\prn*{x\ind{1}, x\ind{0}, \bot}$.
		\For {$t = 1 ~\text{to}~ T$} 
		\State Sample $\Q_t\sim$ Bernoulli($p$) with bias $p$. 
		\If {$\Q_t = 1$}{} 
		\State \label{line:Hessian_estimator_eg_cubic1}  Query the oracle 
		$\hessn{1}$ times at $x \ind{t}$ and compute 
		\begin{equation*}
			\hessest[t]_1 \leftarrow \frac{1}{\hessn{1}} \sum_{j=1}^{\hessn{1}} \stocder{x\ind{t}, z\ind{t, j}}{2},\quad\text{where}\quad{}z\ind{t, j}\overset{\mathrm{i.i.d.}}{\sim}P_z. 
		\end{equation*} 
		\State \label{line:constrained_step_eg_cubic} Set the next point $x\ind{t+1}$ as \begin{equation*}  
		\hspace{0.5in}x\ind{t+1} \gets \argmin_{\nrm{y -x\ind{t}}\leq\eta} 
		\tri[\big]{\gradest[t], y - x\ind{t}} + \frac{1}{2}\tri[\big]{y-x\ind{t}, 
		\hessest[t]_1(y - x\ind{t})} + \frac{M}{6} 
		\nrm{y-x\ind{t}}^{3}. \end{equation*}
		\State $\gradest[t +1] \leftarrow \getgradientfunc_{\newestparams, \biasg}\prn*{x\ind{t +1}, x\ind{t}, g\ind{t}}$.\label{line:gradient_estimation_eg_cubic}
 		\Else 
 		\State \label{line:Hessian_estimator_eg_cubic2}  Query the oracle 
 		$\hessn{2}$ times at $x \ind{t}$ and compute 
		\begin{equation*}
			\hessest[t]_2 \leftarrow \frac{1}{\hessn{2}} \sum_{j=1}^{\hessn{2}} \stocder{x\ind{t}, z\ind{t, j}}{2},\quad\text{where}\quad{}z\ind{t, j}\overset{\mathrm{i.i.d.}}{\sim}P_z. 
		\end{equation*}
		\If {$\eigmin \prn*{H\ind{t}_2} \leq - 4\gamma$} 
  			\State Find a unit vector $u\ind{t}$ such that $\hessest[t]_2\brk*{u\ind{t}, u\ind{t}} \leq - 2\gamma$.
  			\State $x\ind{t+1} \leftarrow x\ind{t} + \frac{\gamma}{L_2} \cdot{}r\ind{t}\cdot u\ind{t}$, where $r\ind{t} \sim \text{Uniform}\prn*{\crl*{-1, 1}}$.   \label{line:negative_curvature_update_eg_cubic}
			\State $\gradest[t +1] \leftarrow \getgradientfunc_{\newestparams, \biash}\prn*{x\ind{t +1}, x\ind{t}, g\ind{t}}$.\label{line:gradient_estimation_eg_cubic2}
		\Else
			\State $x\ind{t+1} \leftarrow x\ind{t}$.\label{line:negative_curvature_update_eg_cubic_equality}
			\State $\gradest[t+1] \leftarrow \gradest[t]$.
		\EndIf
		\EndIf 
		\EndFor 
		\State \textbf{return} $\hx$ chosen uniformly at random from 
		$\crl*{x\ind{t}}_{t=1}^{T-1}$. 
\end{algorithmic} 
\end{algorithm} 

\colt{\begin{theorem}
\label{thm:eg_cubic_upper} For any function $F \in \cF_2(\Delta, \infty, 
L_2)$, stochastic second order oracle in $\cO_{2}(F, \sigma_1, \sigma_2)$, 
$\epsilon \leq \sigma_1$, and $\gamma \leq \min\crl{\sigma_2, 
\sqrt{\epsilon{}L_2}, \Delta^{\frac{1}{3}} L_2^{\frac{2}{3}}}$, with probability at least $\frac{3}{5}$, \pref{alg:eg_cubic} returns a point $\xhat$ such that  
   \[
\nrm*{\grad \func{\xhat}} \leq \epsilon \quad\text{ 
   and} \quad\lambda_{\min{}} \prn*{\grad^2 \func{\xhat}} \geq -\gamma,
\]
 and performs at most
 \begin{align*}
\widetilde{O} \prn*{\frac{  \Delta \sigma_1 \sigma_2}{\epsilon^3}  +  
\frac{\Delta L_2 \sigma_1\sigma_2}{\gamma^2 \epsilon^2} +  \frac{\Delta 
L^2_2 \sigma_2^2}{\gamma^5}}
\end{align*}
 stochastic gradient and Hessian queries.
\end{theorem}} 
\arxiv{\newpage \begin{proof}[Proof of  \pref{thm:eg_cubic_upper}]}
\colt{\begin{proof}}
We first show that \pref{alg:eg_cubic} returns a point $\hx$ such that, $\nrm*{ \nabla F(\hx)}  \leq 450 \epsilon$ and $\eigmin\prn*{\derF{\hx}{2}} \geq -4\gamma$. We then bound the expected number of oracle queries used throughout the execution.

\noindent
Before we delve into the proof, first note that using \pref{lem:gradient_estimator_general_lemma}, we have for all $t \geq 1$, 
\begin{align*} 
 \En \brk*{ \nrm*{\nabla  F\prn{ x \ind{t}} - \gradest[t]}^2 } \leq \epsilon^2. 
\end{align*}
Further, using  \pref{lem:hess_error_gen_lemma} with our choice of $\hessn{1}$ and $\hessn{2}$, we have, for all $t \geq 1$,  
\begin{align*} 
\En \brk*{ \nrm*{ \nabla^2  F\prn{ x \ind{t}} - \hessest[t]_1}^2_\op } \leq \frac{\epsilon M}{900}, \quad \text{and,} 
 \quad \En \brk*{ \nrm*{ \nabla^2  F\prn{ x \ind{t}} - \hessest[t]_2}^2_\op } \leq \frac{\gamma^2}{20}.  \numberthis \label{eq:epsilon_cubic_hess_var_bound}   
\end{align*} 

To begin the proof, we observe that for any $t \geq 0$, there are two scenarios: (a) either $\Q_t = 1$ and the algorithm goes through \pref{line:constrained_step_eg_cubic}, or, (b) $\Q_t = 0$ and the algorithm goes through \pref{line:negative_curvature_update_eg_cubic}. We analyze the two cases separately below. 
\begin{enumerate}[label=(\alph*)]
\setlength{\itemindent}{-.12in} 
\item \textbf{Case 1: $\Q_t = 1$.} In this case, we set $x\ind{t+1}$ using the update rule in \pref{line:constrained_step_eg_cubic}. Invoking \pref{lem:epsilon_cubic_descent_epsilon_step} with the bound in \pref{eq:epsilon_cubic_hess_var_bound} and $\eta = 30 \sqrt{\frac{\epsilon}{M}}$, we get
\begin{align}
      \En \brk*{F\prn{x\ind{t}} -  F\prn{x\ind{t+1}} \lmid \Q_t = 1} 
      &\geq  \frac{450 \epsilon^{3/2}}{\sqrt{M}}  \prn*{ \Pr \Big( \nrm*{\derF{x\ind{t+1}}{}} \geq 450 \epsilon \Big) - \frac{1}{32}}.  \label{eq:eg_cubic_proof3}
 \end{align} 

\item \textbf{Case 2: $\Q_t = 0$.} In this case, either $\eigmin\prn*{H\ind{t}_2} > -4\gamma$, in which case we set $x\ind{t+1} = x\ind{t}$, or we compute $x\ind{t+1}$ using the update rule in \pref{line:negative_curvature_update_eg_cubic} in \pref{alg:eg_cubic}. Thus, using  \pref{lem:general_lem_curvature_descent} with \pref{eq:epsilon_cubic_hess_var_bound}, we get 
\begin{align} 
		 \En \brk*{ \func{x\ind{t}} - \func{x\ind{t+1}} \lmid \Q_t = 0} \geq \frac{5\gamma^3}{6L_2^2} \prn*{ \Pr \prn*{\eigmin\prn*{H\ind{t}_2} \leq  \gamma} - \frac{1}{32}}. \label{eq:eg_cubic_proof4}
\end{align} 
\end{enumerate}
Combining the two cases ($\Q_t = 0$ or $\Q_t = 1$)  from \pref{eq:eg_cubic_proof3} and \pref{eq:eg_cubic_proof4} above, we get 
\begin{align*}
   \En \brk*{\func{x\ind{t}} - \func{x \ind{t+1}}}  &=  \sum_{q \in \crl*{0, 1}} \Pr \prn{\Q_t = q} \En \brk*{\func{x\ind{t}} - \func{x \ind{t+1}} \mid{} \Q_t = q} \\ 
   		&\overgeq{} (1 - p) \cdot \frac{5\gamma^3}{6L_2^2} \prn*{ \Pr\prn*{\eigmin(\nabla^2 F(x\ind{t})) \leq -4\gamma} - \frac{1}{32}} \\
   		&\qquad \qquad + p \cdot \frac{450 \epsilon^{3/2}}{\sqrt{M}} \prn*{  \Pr \prn*{ \nrm*{\derF{x\ind{t+1}}{}} \geq 450 \epsilon}  - \frac{1}{32}}. \end{align*} 
Telescoping the inequality above for $t$ from 0 to $T-1$, and using the bound $\En \brk*{\func{x\ind{0}} - \func{x \ind{T}}} \leq \Delta$, we get  
\begin{align*}
\Delta  &\overgeq{}    \En \brk*{\func{x\ind{0}} - \func{x \ind{T}}} \\ 
			&\geq \frac{5 T  (1 - p) \gamma^3}{6L_2^2}  \prn*{ \frac{1}{T} \sum_{t=0}^{T-1} \Pr\prn*{\eigmin(\derF{x\ind{t}}{2}) \leq -4\gamma} - \frac{1}{32}} \\
			& \qquad \qquad\qquad  + \frac{450 Tp\epsilon^{3/2}}{\sqrt{M}} \prn*{ \frac{1}{T} \sum_{t=1}^{T} \Pr \prn*{ \nrm*{\derF{x\ind{t}}{}} \geq 450 \epsilon}  - \frac{1}{32}} \\ 
   		&\overgeq{\proman{1}}15 \Delta \prn*{  \frac{1}{T}  { \sum_{t=0}^{T-1} \Pr\prn*{\eigmin(\derF{x\ind{t}}{2}) \leq -4\gamma} +  \frac{1}{T} \sum_{t=1}^{T} \Pr \prn*{ \nrm*{\derF{x\ind{t}}{}} \geq 450 \epsilon}}  - \frac{1}{8}}  \\	
   		&\overgeq{\proman{2}}15 \Delta \prn*{  \frac{5}{6(T-1)} \sum_{t=1}^{T-1} \prn*{  \Pr\prn*{\eigmin(\derF{x\ind{t}}{2}) \leq -4\gamma} + \Pr \prn*{ \nrm*{\derF{x\ind{t}}{}} \geq 450 \epsilon}}  - \frac{1}{8}}  \\ 
   		&\overgeq{\proman{3}}15 \Delta \prn*{  \frac{5}{6} \prn*{  \Pr\prn*{\eigmin(\derF{\hx}{2}) \leq -4\gamma} + \Pr \prn*{ \nrm*{\derF{\hx}{}} \geq 450 \epsilon}}  - \frac{1}{8}}, \numberthis \label{eq:eg_cubic_proof5}
\end{align*} where the inequality in $\proman{1}$ follows from \pref{lem:TP_eg_cubic}. The inequality in $\proman{2}$ is given by ignoring the (non-negative) terms  $\Pr \prn*{\derF{x\ind{0}}{2} \leq -4 \gamma}$ and $ \Pr \prn*{ \nrm*{\derF{x\ind{T}}{}} \geq 450 \epsilon}$ on the right-hand side and using the fact that $T \geq 6$. Finally, $\proman{3}$ follows by recalling the definition of $\hx$ as samples uniformly at random from the set $\prn{x\ind{t}}_{t=1}^{T-1}$. Rearranging the terms, we get
\begin{align*}
\Pr\prn*{\eigmin(\derF{\hx}{2}) \leq -4\gamma} + \Pr \prn*{ \nrm*{\derF{\hx}{}} \geq 450 \epsilon} \leq \frac{1}{4}, \intertext{which further implies that the returned point $\hx$ satisfies}
    \Pr \prn*{\eigmin(\nabla^2 F(\hx)) \geq - \gamma \wedge \nrm*{\nabla F(\hx)} \leq 450 \epsilon} &\geq  \frac{3}{4}. \numberthis \label{eq:eg_cubic_hp_guarantee}
\end{align*} 

\paragraph{Bound on the number of oracle queries.}
Let us first introduce some notation to count the number of oracle calls made in each iteration of the algorithm.
\begin{itemize}
  \item On \pref{line:gradient_estimation_eg_cubic} and \pref{line:gradient_estimation_eg_cubic2}, \pref{alg:eg_cubic} queries the stochastic oracle through the subroutine $\getgradientfunc$. Let $m_g(t)$ denote the 
    total number of oracle queries resulting from either line at iteration $t$.
    \item Let 
$m_{h, 1}(t)$ and $m_{h, 2}(t)$ denote the total number of oracle calls made 
by \pref{line:Hessian_estimator_eg_cubic1} and  
\pref{line:Hessian_estimator_eg_cubic2} at iteration $t$ to compute $H\ind{t}_1$ and 
$H\ind{t}_2$ respectively.
\end{itemize}
Define $M_g$, $M_{h, 1}$ and $M_{h, 2}$ 
by $\sum_{t=1}^{T} m_g(t)$, $\sum_{t=1}^{T} m_{h, 1}(t)$ and 
$\sum_{t=1}^{T} m_{h, 2}(t)$ respectively. In what follows, we give separate bounds for $\En \brk*{M_g}$, $\En \brk*{M_{h, 1}}$ and $\En \brk*{M_{h, 2}}$. The final statement on the total number of oracle calls follows by an application of Markov's inequality. 

\paragraph{Bound on $\En \brk*{M_g}$.}  For any $t > 0$, there are two scenarios, either  (a) $\Q_{t} = 1$ and we update $x\ind{t+1}$ through \pref{line:constrained_step_eg_cubic}, or (b) $\Q_{t} = 0$ and we update $x\ind{t+1}$ through \pref{line:negative_curvature_update_eg_cubic} or\pref{line:negative_curvature_update_eg_cubic_equality}.  Thus, using the law of total expectation
\begin{align*}
 \En \brk*{M_g} &= {\sum_{t=0}^{T-1} \Pr \prn*{\Q_t = 0} \En \brk*{m_g(t) \mid \Q_{t} = 0}}
   +  {\sum_{t=0}^{T-1} \Pr \prn*{\Q_t = 1} \En \brk*{m_g(t) \mid \Q_t = 1}}. \numberthis \label{eq:eg_cubic_part_complexity0}
\end{align*} 
We denote the two terms on the right hand side above by $\markedterm{A}$ and $\markedterm{B}$, respectively. We bound them separately in as follows.
\begin{enumerate}[label=(\alph*)]
\setlength{\itemindent}{-.13in}
\item \textbf{Bound on $\markedterm{A}$.} Using \pref{lem:gradient_estimator_oracle_complexity_meta_lemma} with the fact that $\Pr \prn*{\Q_t = 0}  = 1 - p$, we get 
\begin{align*}
\markedterm{A}  &\overeq{}  6 \sum_{t =1}^{T} (1 - p) \cdot \En \brk*{ \frac{\biash \sigma_1^2}{\epsilon^2} + \frac{(\sigma_2^2 + L_2 \epsilon) \cdot \nrm*{x\ind{t +1} - x\ind{t}}^2}{\biash \epsilon^2} + 1  \lmid \Q_t = 0} \\ 
					   &\overeq{\proman{1}}  6 T (1 - p) \cdot \prn*{ \frac{\biash \sigma_1^2}{\epsilon^2} + \frac{(\sigma_2^2 + L_2 \epsilon) \cdot \gamma^2}{\biash \epsilon^2 L_2^2} + 1 } \\ 
					   &\overeq{\proman{2}} O\prn*{ \frac{\Delta L_2^2}{\gamma^3}  \cdot \prn*{\frac{\biash \sigma_1^2}{\epsilon^2} + \frac{(\sigma_2^2 + L_2 \epsilon) \cdot \gamma^2}{\biash \epsilon^2 L_2^2} + 1}} \\
					   &\overeq{\proman{3}} O\prn*{ \frac{\Delta L_2 \sigma_1 \sqrt{\sigma^2_2 + \epsilon L_2}}{\gamma^2 \epsilon^2} + \frac{\Delta \prn*{\sigma_2^2 + \epsilon L_2}}{\gamma \epsilon^2} + \frac{\Delta L_2^2}{\gamma^3}}, \numberthis \label{eq:eg_cubic_part_complexity1}
\end{align*} where $\proman{1}$ holds because when $Q_t=0$, we either have $\nrm*{x\ind{t} - x\ind{t-1}} \leq \frac{\gamma}{L_2}$ (if we follow the update rule in \pref{line:negative_curvature_update_eg_cubic}) or $\nrm*{x\ind{t} - x\ind{t-1}}=0$ (if we follow \pref{line:negative_curvature_update_eg_cubic_equality}). The inequality $\proman{2}$ uses the bound on $T \cdot (1 - p)$ from \pref{lem:TP_eg_cubic} and $\proman{3}$ follows from plugging in the value of $\biash$.

\item \textbf{Bound on $\markedterm{B}$.} Using \pref{lem:gradient_estimator_oracle_complexity_meta_lemma} with the definition $\Pr \prn*{\Q_t = 1}  = p$, we get 
\begin{align*}
   \markedterm{B} &\overeq{}  6 \sum_{t=1}^{T} p \cdot \En \brk*{ {\frac{\biasg \sigma_1^2}{\epsilon^2} + \frac{(\sigma_2^2 + L_2 \epsilon) \cdot \nrm*{x\ind{t + 1} - x\ind{t}}^2}{\biasg \epsilon^2} + 1} \lmid \Q_t = 1} \\ 
   		&\overeq{\proman{1}} 6  T p \cdot \prn*{ {\frac{\biasg \sigma_1^2}{\epsilon^2} + \frac{(\sigma_2^2 + L_2 \epsilon) \cdot \eta^2}{\biasg \epsilon^2} + 1}} \\
   		& \overeq{\proman{2}}  O\prn*{ \frac{\Delta \sqrt{M}}{\epsilon^{1.5}} \cdot \prn*{ {\frac{\biasg \sigma_1^2}{\epsilon^2} + \frac{(\sigma_2^2 + L_2 \epsilon) \cdot \eta^2}{\biasg \epsilon^2} + 1}}} \\ 
   		& \overeq{\proman{3}} O\prn*{\frac{\Delta \sigma_1 \sqrt{\sigma_2^2 + \epsilon L_2}}{\epsilon^3} +  \frac{\Delta \sqrt{M}}{\epsilon^{1.5}}} \\ 
   		&\overeq{\proman{4}} O\prn*{\frac{\Delta \sigma_1 \sqrt{\sigma_2^2 + \epsilon L_2}}{\epsilon^3} +  \frac{\Delta \sqrt{L_2}}{\epsilon^{1.5}} + \frac{\Delta \sigma_2 \sqrt{\log(d)}}{\epsilon^2}}, \numberthis \label{eq:eg_cubic_part_complexity2}
\end{align*} where $\proman{1}$ is given by the update rule from \pref{line:constrained_step_eg_cubic} and the fact that $\getgradientfunc$ uses parameter $\biasg$ in this case, and $\proman{2}$ follows by using the bound on $T \cdot p$ from \pref{lem:TP_eg_cubic}. The inequality $\proman{3}$ follows because for the choice of parameters $\eta$ and $M$ and the assumed range of $\eps$ in the theorem statement, $\biasg = \frac{\eta \sqrt{\sigma_2^2 + \epsilon L_2}}{\sigma_1} < 1$. Finally, the inequality $\proman{4}$ is given by plugging in the value of $M$ and using that $\eps\leq{}\sigma_1$. 
\end{enumerate}

Plugging the bound in \pref{eq:eg_cubic_part_complexity1} and \pref{eq:eg_cubic_part_complexity2} back in \pref{eq:eg_cubic_part_complexity0}, we get 
\begin{align*}
   \En \brk*{M_g} &= O\prn*{\frac{\Delta L_2 \sigma_1 \sqrt{\sigma^2_2 + \epsilon L_2}}{\gamma^2 \epsilon^2} + \frac{\Delta \prn*{\sigma_2^2 + \epsilon L_2}}{\gamma \epsilon^2} +  \frac{\Delta L_2^2}{\gamma^3}} \\ 
   & \qquad \qquad + O\prn*{\frac{\Delta \sigma_1 \sqrt{\sigma_2^2 + \epsilon L_2}}{\epsilon^3} +  \frac{\Delta \sqrt{L_2}}{\epsilon^{1.5}} + \frac{\Delta \sigma_2 \sqrt{\log(d)}}{\epsilon^2}}. \numberthis \label{eq:eg_cubic_part_complexity3}
\end{align*}

\paragraph{Bound on $\En \brk*{M_{H, 1}}$.}  For each $t \geq 0$,  
\pref{alg:eg_cubic} samples an independent Bernoulli $\Q_t$ with bias 
$\En \brk*{Q_t} = p$ and executes \pref{line:Hessian_estimator_eg_cubic1} 
if $\Q_t = 1$. For every such pass through 
\pref{line:Hessian_estimator_eg_cubic1}, the algorithm queries the 
stochastic Hessian oracle $\hessn{1}$ times. Thus, \begin{align*} 
   \En \brk*{M_H} &=  \En \brk*{ \sum_{t=0}^{T-1} \indicator{\Q_t = 1} \cdot  \hessn{1}} = T \cdot p \cdot \hessn{1} \\ &\overeq{\proman{1}} O\prn*{  \frac{\Delta \sqrt{M}}{\epsilon^{1.5}} \cdot \ceil*{\frac{900 \sigma_2^2 \log(d)}{\epsilon M}}} = O\prn*{\frac{\Delta \sqrt{L_2}}{\epsilon^{1.5}} + \frac{\Delta \sigma_1 \sigma_2 \sqrt{\log(d)}}{\epsilon^{3}}}, \numberthis \label{eq:eg_cubic_part_complexity4}
   \end{align*} 
where $\proman{1}$ follows by plugging in the values of $\hessn{1}$ and $M$ as specified in \pref{alg:eg_cubic} (using that $\eps\leq\sigma_1$ to simplify), and using the bound on  $T \cdot p$ from  \pref{lem:TP_eg_cubic} . 

\paragraph{Bound on $\En \brk*{M_{H, 2}}$.} The algorithm executes \pref{line:Hessian_estimator_eg_cubic2} 
only if $\Q_t = 0$, which happens with probability $1-p$. For every such pass through 
\pref{line:Hessian_estimator_eg_cubic2}, the algorithm queries the 
stochastic Hessian oracle $\hessn{2}$ times. Consequently,
\begin{align*} 
   \En \brk*{M_H} &=  \En \brk*{ \sum_{t=0}^{T-1} \indicator{\Q_t = 0} \cdot \hessn{1}} = T \cdot (1 - p) \cdot \hessn{1} \\ &\overeq{\proman{1}} O\prn*{  \frac{\Delta L_2^2}{\gamma^3} \cdot \ceil*{\frac{20 \sigma_2^2 \log(d)}{\gamma^2}}} = O\prn*{  \frac{\Delta L_2^2 \sigma_2^2 \log(d)}{\gamma^5} + \frac{\Delta L_2^2}{\gamma^3}} , \numberthis \label{eq:eg_cubic_part_complexity5}
   \end{align*} 
where $\proman{1}$ follows by plugging in the values of $\hessn{1}$ as specified in \pref{alg:eg_cubic}, and using the bound on  $T \cdot p$ from  \pref{lem:TP_eg_cubic}. 

Adding together all the bounds above (from \pref{eq:eg_cubic_part_complexity3}, \pref{eq:eg_cubic_part_complexity4}, and \pref{eq:eg_cubic_part_complexity5}), we have that the total number of oracle queries by \pref{alg:eg_cubic} till time $T$ is bounded in expectation by
\begin{align*}
\begin{aligned}
   \En \brk*{M} &= \En \brk*{M_g + M_{H, 1} + M_{H, 2}} \\
   						&= O\Bigg(\frac{\Delta L_2^2 \sigma_2^2 \log(d)}{\gamma^5} + \frac{\Delta L_2 \sigma_1 \sqrt{\sigma^2_2 + \epsilon L_2}}{\gamma^2 \epsilon^2}  + \frac{\Delta \sigma_1 \sigma_2  \sqrt{ \log(d)}}{\epsilon^3} + \frac{\Delta \sigma_1 \sqrt{L_2}}{\epsilon^{2.5}}\Bigg) \\ 
   						&~~~~ \qquad + O\Bigg( \frac{\Delta \prn*{\sigma_2^2 + \epsilon L_2}}{\gamma \epsilon^2} + \frac{\Delta L_2^2}{\gamma^3} +  \frac{\Delta \sigma_2 \sqrt{\log(d)}}{\epsilon^{2}} +  \frac{\Delta \sqrt{L_2}}{\epsilon^{1.5}} \Bigg).
\end{aligned}				
\end{align*}
Using Markov's inequality, this implies that with probability at least $\frac{7}{8}$, 
\begin{align*}
  M  &= O\Bigg(\frac{\Delta L_2^2 \sigma_2^2 \log(d)}{\gamma^5} + \frac{\Delta L_2 \sigma_1 \sqrt{\sigma^2_2 + \epsilon L_2}}{\gamma^2 \epsilon^2}  + \frac{\Delta \sigma_1 \sigma_2  \sqrt{ \log(d)}}{\epsilon^3} + \frac{\Delta \sigma_1 \sqrt{L_2}}{\epsilon^{2.5}}\Bigg) \\ 
   						&~~~~ \qquad + O\Bigg( \frac{\Delta \prn*{\sigma_2^2 + \epsilon L_2}}{\gamma \epsilon^2} + \frac{\Delta L_2^2}{\gamma^3} +  \frac{\Delta \sigma_2 \sqrt{\log(d)}}{\epsilon^{2}} +  \frac{\Delta \sqrt{L_2}}{\epsilon^{1.5}} \Bigg). 
\end{align*}Ignoring the lower order terms, we have 
\begin{align*}
  M  &= \widetilde{O}\Bigg(\frac{\Delta L_2^2 \sigma_2^2}{\gamma^5} + \frac{\Delta L_2 \sigma_1 \sigma_2}{\gamma^2 \epsilon^2}  + \frac{\Delta \sigma_1 \sigma_2}{\epsilon^3} \Bigg). \numberthis \label{eq:eg_cubic_full_bound}
\end{align*} The final statement follows by union bound, using the failure probabilities for \pref{eq:eg_cubic_hp_guarantee} and \pref{eq:eg_cubic_full_bound}. 
\end{proof} 

\begin{lemma} 
\label{lem:TP_eg_cubic}
For the values of the parameters $T$ and $p$ specified in \pref{alg:eg_cubic}, we have
 \begin{align*}
   \frac{\Delta \sqrt{M}}{30 \epsilon^\frac{3}{2}} \leq T p \leq     \frac{2\Delta \sqrt{M}}{30 \epsilon^\frac{3}{2}} \quad \text{and} \quad  \frac{18 \Delta L_2^2}{\gamma^3} \leq T (1 - p) \leq  \frac{36 \Delta L_2^2}{\gamma^3}.
\end{align*} \end{lemma} \begin{proof}
Under the assumption that  $\gamma \leq \Delta^{\frac{1}{3}} L_2^{\frac{2}{3}}$, we have that 
\begin{align*}
  T \geq \frac{18 \Delta L_2^2}{ \gamma^3} &\geq 18. 
\end{align*}
Thus, using the fact that $x \leq \ceil*{x} \leq 2x$ for any $x \geq 1$, we get
\begin{align}
\frac{18 \Delta L_2^2}{\gamma^3} + \frac{\Delta\sqrt{M}}{30 \epsilon^{\frac{3}{2}}} \leq T \leq \frac{36 \Delta L_2^2}{\gamma^3} + \frac{\Delta\sqrt{M}}{15 \epsilon^{\frac{3}{2}}} \label{eq:ceil_bound_eg_hvp2}. 
\end{align} Thus, plugging in the value of $T$ and $p$, we get
\begin{align*}
   T (1 - p) &= \ceil*{ \frac{18 \Delta L_2^2}{\gamma^3} + \frac{\Delta\sqrt{M}}{30 \epsilon^{\frac{3}{2}}}} \cdot \prn*{ 1 - \frac{\sqrt{M} \gamma^\frac{3}{2} }{\sqrt{M} \gamma^\frac{3}{2} + 540 L_2^2\epsilon^\frac{3}{2}} } \\ 
				&\leq  \prn*{\frac{36 \Delta L_2^2}{\gamma^3} + \frac{\Delta\sqrt{M}}{15 \epsilon^{\frac{3}{2}}}} \cdot { \frac{540 L_2^2\epsilon^\frac{3}{2}}{\sqrt{M} \gamma^\frac{3}{2} + 540 L_2^2\epsilon^\frac{3}{2}} }   = \frac{36 \Delta L_2^2}{\gamma^3}, 
\end{align*} where the first inequality is due to \pref{eq:ceil_bound_eg_hvp2}. Similarly, we have that 
\begin{align*}
   T (1 - p) &\geq \ceil*{\frac{18 \Delta L_2^2}{\gamma^3} + \frac{\Delta\sqrt{M}}{30 \epsilon^{\frac{3}{2}}}} \cdot { \frac{540 L_2^2\epsilon^\frac{3}{2}}{\sqrt{M} \gamma^\frac{3}{2} + 540 L_2^2\epsilon^\frac{3}{2}}}  = \frac{18 \Delta L_2^2}{\gamma^3}. 
\end{align*} Together, the above two bounds imply that
\begin{align*}
   \frac{18 \Delta L_2^2}{\gamma^3} \leq T (1 - p) \leq  \frac{36 \Delta L_2^2}{\gamma^3}
\end{align*} The bound on $T \cdot p$ follows similarly. 
\end{proof}

\section{Lower bounds}
\label{app:lower}

\subsection{Proof of Theorem \ref*{thm:eps_so_zero_respecting}}
\label{app:epsilon_lower}

In this section, we prove \pref{thm:eps_so_zero_respecting}. We begin by generalizing the
lower bound framework of \cite{arjevani2019lower}---which centers
around the notion of zero-respecting algorithms and stochastic
gradient estimators called \emph{probabilistic zero-chains}---to higher-order derivatives.
	Given a $q$th-order tensor $\tensT\in\R^{\otimes^q d }$, we define
	$\support{\{\tensT\}}\defeq \{i\in[d]~|~\tensT_i\neq 0\}$, 
	where $\tensT_i$ is the $(q-1)$-order subtensor defined by 		
	$[\tensT_i]_{j_1,\dots,j_{q-1}}=T_{i,j_1,\dots,j_{q-1}}$. Given a 
	tuple of tensors $\cT = \prn*{\tensT^{(1)}, \tensT^{(2)},\dots}$, we let 
	$\support{\{\cT\}}\defeq \bigcup_{i} \support\{\tensT^{(i)}\}$ be the 
	union 
	of the supports of $\tensT^{(i)}$. Lastly, given an algorithm $\alg$ and a 
	an oracle $\oracle_{F}^p$, we let $x\ind{t}_{\alg\brk{\oracle_{F}^{p}}}$ denote 
	the (possibly randomized) $t$th query point generated by $\alg$ when fed 
	by 
	information from $\oracle$ (i.e.,
        $x\ind{t}_{\alg\brk{\oracle_{F}^{p}}}$ is a measurable
        function of $\crl*{\oracle_{F}^p(x\ind{i},z\ind{i})}_{i=1}^{t-1}$,
        and possibly a random seed $r\ind{t}$).
	\begin{definition}
		A stochastic $p$th-order algorithm $\alg$ is \emph{zero-respecting} 
		if for any function $F$ and any $p$th-order oracle $\oracle_F^p$, the 
		iterates  
		$\{x\ind{t}\}_{t\in\N}$ produced by $\alg$ by querying
		$\oracle_F^p$ satisfy
		\begin{equation}\label{eq:zero_respecting}
		\support\prn[\big]{x\ind{t}}\subseteq
		\bigcup_{i < t} 
		\support\prn[\big]{\oracle_F^p(x\ind{i},z\ind{i})},~  
		\text{ for all } t\in\N,
		\end{equation}
		with probability one with respect the randomness of the algorithm and the 
		realizations of $\{z\ind{t}\}_{t\in\N}$.
	\end{definition}
	Given $x\in\R^d$, we define
	\begin{equation}
	\label{eq:progress}
	\prog_{\alpha}(x) \defeq \max\crl*{i\ge 0 \mid \abs{x_i} > 
		\alpha}~~\mbox{(where we set $x_0\defeq 1$)},
	\end{equation}
	which represents the highest index of $x$ whose entry is $\alpha$-far from zero, for 
	some threshold $\alpha\in[0,1)$. To lighten notation, we further let 
	$\prog\defeq \prog_0$. For a 
	tensor $\tensT$, we let $\prog(\tensT)\defeq \max\crl{ \support{\{\tensT\}}}$ 
	denote the highest index in $\support{\{\tensT\}}$ (where 
	$\prog(\tensT)\defeq0$ if $\support\{\tensT\}=\emptyset$), and let 
	$\prog(\cT)\defeq \max_{i}\prog(\tensT^{(i)})$ be the overall maximal index 
	of $\prog(T^{(i)})$ for a tuple of tensors $\cT=\prn*{\tensT^{(1)}, 
	\tensT^{(2)},\dots}$. 
	\begin{definition}
		\label{def:prob_p_zero_chain}
		A collection of derivative estimators 
		$\pest{1}(x,z),\dots,\pest{p}(x,z)$ for a function $F$
		forms a probability-$\pchain$ zero-chain if 
		\begin{align*}
		\Pr \prn*{ \exists x~\mid~ \prog\prn{\pest{1}(x, z), \dots, 
		\pest{p}(x, 	z)}
		= \progf{1}{4}(x) + 1} &\leq \pchain \intertext{and} \Pr \prn*{ \exists
		x~\mid~ \prog\prn{\pest{1}(x, z), \dots, \pest{p}(x,	z)} = 	
		\progf{1}{4}(x) +
		i} &=0, ~i>1.
		\end{align*}
		No constraint is imposed for  $i \leq 	\progf{1}{4}(x)$.
	\end{definition}
	We note that the constant $1/4$ is used here for compatibility with the 
	analysis in \citet[Section 3]{arjevani2019lower}. Any non-negative constant 
	less than $1/2$ would suffice in its place. The next lemma formalizes the 
	idea that 
	any zero-respecting algorithm interacting with a probabilistic zero-chain 
	must wait many rounds to activate all the coordinates. 
	
	\begin{lemma}
		\label{lem:prob-zero-chain}
		Let $\pest{1}(x,z),\dots,\pest{p}(x,z)$ be a 
		collection of probability-$\pchain$ zero-chain derivative estimators 
		for $F:\bbR^{T}\to\bbR$, and let $\stocOhigh_F$ be an oracle 
		with $\stocOhigh_F(x,z)=(\pest{q}(x,z))_{q\in[p]}$. Let 
		$\crl[\big]{x\ind{t}_{\alg[\oracle_F]}}$ be a sequence of queries 
		produced by $\alg\in\AlgZR(K)$ interacting with $\stocOhigh_F$. Then, 
		with probability at least $1-\delta$, 
		\begin{equation*}
		\prog\prn*{	x\ind{t}_{\alg[\stocOhigh_F]}
		} < T, \quad\text{for all } t \leq{} 
		\frac{T-\log(1/\delta)}{2\pchain}.
		\end{equation*}
	\end{lemma}
	The proof of \pref{lem:prob-zero-chain} is a simple adaptation of the proof 
	of Lemma 1 of \cite{arjevani2019lower} to high-order 
	zero-respecting 
	methods---we provide it here for completeness. The proof idea is that any 
	zero-respecting algorithm must activate coordinates in sequence, and must 
	wait on average at least $\Omega({1}/{\rho})$ rounds between activations, 
	leading to a total wait time of $\Omega({T}/{\rho})$ rounds. 
	\newcommand{\gammac}{\delta}
	\begin{proof}
		Let $\{\pest{q}(x^{(i)},z^{(i)})\}_{q\in[p]}$ denote the oracle responses 
		for the $i$th query made at the point $x^{(i)}$, and let $\cG^{(i)}$ be 
		the natural filtration for the algorithm's iterates, the oracle randomness, and the 
		oracle answers up to time $i$. We measure the progress of the algorithm 
		through two quantities: 
		\begin{align*}
		\pi^{(t)} &\ldef{} \max_{i \leq t} \prog 
		\prn*{x^{(i)}} = 	\max \crl*{j \leq d \mid{} x_j^{(i)} \neq 0 \text{ 
		for some } 	i \leq t  }, \\
		\gammac^{(t)} &\ldef{}  \max_{i \leq t}  \prog
		\prn*{\grad^{q} F(x^{(i)}, z^{(i)}) }\\ 
		&= \max \crl*{ j \leq d \mid 			\grad^{q} 
		f(x^{(i)}, \zi)_j 
			\neq 0 \text{ for some } i 
			\leq t  \text{ and } q\in[p]} .
		\end{align*}
		
		Note that $\pi^{(t)}$ is the largest non-zero coordinate in 
		$\support\{\prn{x^{(i)}}_{i \leq t}\}$, and that $\pi^{(0)}=0$ and 
		$\gammac^{(0)}=0$. Thus, for any zero-respecting 
		algorithm 
		\begin{equation}
		\pi^{(t)} \leq \gammac^{(t-1)}, 
		\label{eq:gamma_pi_1}  
		\end{equation}
		for all $t$. Moreover, observe that with probability one,
		\begin{equation}
		\prog\prn*{\grad^{q} F(x^{(t)}, z^{(t)}) } \leq 
		1 + \prog_\frac{1}{4}(x^{(t)}) \leq 
		1 + \prog(x^{(t)}) \leq 1 + \pi^{(t)} \leq 1 + 
		\gammac^{(t-1)},  
		\label{eq:gamma_pi_2} 
		\end{equation}
		where the first inequality follows by the zero-chain property. Further, 
		using the $\pchain$-zero chain property, it follows that conditioned 
		on $\cG^{(i)}$, with probability at least $1- \pchain$, 
		\begin{equation}
		\prog\prn*{\nabla^{q} F(x^{(t)}, z^{(t)}) } \leq 
		\prog_\frac{1}{4}(x^{(t)}) \leq 
		\prog(x^{(t)}) \leq  \pi^{(t)} \leq  
		\gammac^{(t-1)}.
		\label{eq:gamma_pi_3}
		\end{equation}
		Combining \pref{eq:gamma_pi_2} and \pref{eq:gamma_pi_3}, we 
		have that conditioned on $\cG^{(i-1)}$,
		\begin{align*}
		\gammac^{(t-1)} \leq  \gammac^{(t)} \leq 
		\gammac^{(t-1)} 
		+ 1 
		\qquad 
		\text{and} \qquad \Pr \brk*{\gammac^{(t)} = 
		\gammac^{(t-1)} + 
		1 } \leq \rho. 
		\end{align*}
		Thus, denoting the increments $\iota^{(t)} \ldef{} 
		\gammac^\bt - 
		\gammac^{(t-1)}$, we 
		have 	via the Chernoff method,  
		\begin{align*}
		\Pr \brk*{ \gammac^\bt \geq T} 
		&= \Pr \brk*{ \sum_{j=1}^t \iota^{(j)} \geq T} 
		\leq \frac{\En \brk*{\exp\prn*{\sum_{j=1}^t 
		\iota^{(j)}}}}{\exp(T)} 
		= e^{-T} \En \brk*{ \prod_{i=1}^{t} \En 
		\brk*{\exp\prn*{ \iota^{(i)}} \mid \cG^{(i-1)}}}\\
		&\leq e^{-T} \prn*{ 1 - \rho+ \rho\cdot e}^t \leq e^{2\rho{}t - 
		T}.
		\end{align*}
		Thus, $ \Pr \brk*{ \gammac^\bt \geq T}  \leq 
		\delta$ for all $t \leq 
		\frac{T - \log({1}/{\delta})}{2\rho}$; combined with 		
		\pref{eq:gamma_pi_1}, this yields the desired result. 
	\end{proof} 
	
	In light of \pref{lem:prob-zero-chain}, our lower bound
        strategy is as follows. We construct a function $F\in\HOSF$ that both admits probability-$\pchain$ 
	zero-chain derivative estimators and has large gradients for all 
	$x\in\bbR^T$ with $\prog\prn*{x\ind{i}} < T$. Together with 
	\pref{lem:prob-zero-chain}, this ensures that any zero-respecting algorithm 
	interacting with a $p$th-order oracle must perform $\Omega(T/\pchain)$ steps to make the gradient of $F$ 
	small. We make this approach concrete by adopting the construction 
	used in 
		\cite{arjevani2019lower}, and adjusting it so as to be consistent 
		with the additional high-order Lipschitz and variance parameters.  For 
		each $T\in\bbN$, we define
		\begin{equation} \label{eq:hard_function_eps_LB}
		\Funscaled(x) \defeq -\Psi(1)\Phi(x_1) +
		\sum_{i=2}^{T}\brk*{\Psi(-x_{i-1})\Phi(-x_i) -
			\Psi(x_{i-1})\Phi(x_i)},
		\end{equation}
		where the component functions $\Psi$ and $\Phi$ are
		\begin{equation}
		\Psi(x) = \left\{
		\begin{array}{ll}
		0,\quad&x\leq{}1/2,\\
		\exp\prn*{1-\frac{1}{(2x-1)^{2}}},\quad&x>1/2
		\end{array}
		\right.\quad\quad\text{and}\quad\quad\Phi(x) =
		\sqrt{e}\int_{-\infty}^{x}e^{-\frac{1}{2}t^{2}}dt.\label{eq:psi_phi}
		\end{equation}
	We start by collecting some 
	relevant properties of $\Funscaled$.
	\begin{lemma}[\citet{carmon2019lower_i}]
		\label{lem:deterministic-construction} The function
		$\Funscaled$ satisfies:
		\begin{enumerate}
			\item \label{item:val} $\Funscaled(0) -
			\inf_{x}\Funscaled(x) \leq \Delta_0\cdot T$, where 
			$\Delta_0 = 12$.
			
			\item \label{item:lip} For $p\ge1$, the $p$th 
			order derivatives of  $\Funscaled$ are  $\lip{p}$-Lipschitz 
			continuous, where $\lip{p}\le e^{\frac{5}{2}p \log p + 	cp}$ for a 
			numerical constant	$c<\infty$. 

			\item \label{item:grad} For all $x\in\bbR^{T}$, $p\in\N$
			and $i\in[T]$, we have $\nrm*{\grad^p_i \Funscaled(x)}_{\op}\le 
			\ell_{p-1}$.
 
 			\item \label{item:zero-chain} For all $x\in\R^T$ 
			and $p\in\N$, 			
			$\prog\prn*{\grad^{p}\Funscaled(x)}\leq{}\prog_{\frac{1}{2}}(x)+1$.
			 
			\item \label{item:large-grad} For all $x\in\R^T$, 
			if $\prog_1(x)<T$ then $\nrm*{\grad\Funscaled(x)} \ge 
			|\grad_{\prog_{1}(x)+1}\Funscaled(x)| > 1$.
		\end{enumerate}
		
	\end{lemma}
	\begin{proof}
		Parts~\ref{item:val} and \ref{item:lip} follow 
		from Lemma 3 in \cite{carmon2019lower_i} and its proof; 
		Part~\ref{item:grad} is proven in \pref{sec:proof_ell_infty_bound}; 
		Part~\ref{item:zero-chain} follows from Observation 3 in \cite{carmon2019lower_i} and Part \ref{item:large-grad} is the same as Lemma 
		2 in \cite{carmon2019lower_i}. 
			\end{proof}
	The derivative estimators we use are defined as
	\begin{align} \label{eq:basic-construction}
	\brk*{\pestunscaled{q}(x,z)}_i \defeq 
	\prn*{1+\indicator{i > 
	\prog_{\frac{1}{4}}(x)}\prn*{\frac{z}{\rho}-1}}\cdot
	\grad_i^q \Funscaled(x),
	\end{align}
	where $z\sim  \mathrm{Bernoulli}(\pchain)$. 
	\begin{lemma}\label{lem:pzc-pair}
		The estimators $\pestunscaled{q}$ form a probability-$\pchain$ 
		zero-chain, are unbiased for $\grad^{q} \Funscaled$, and satisfy
		\begin{equation}\label{eq:pzc-var-mss}
		\E \,\norm{\pestunscaled{q}(x,z) - \grad^q 
		\Funscaled(x)}^2 
		\le 
		\frac{\ell_{q-1}^2(1-\pchain)}{\pchain},\quad \text{ 
		for all } x\in \R^T.
		\end{equation}
		
	\end{lemma}

	\begin{proof}
		First, we observe that $\E\brk*{\pestunscaled{q}(x,z)} =  		
		\grad^{q} \Funscaled(x)$ for all $x\in\R^T$, as $\E\brk{{z}/{\pchain}} 
		= 1$. Second, we argue that the probability-$\pchain$ zero-chain 
		property holds. Recall that $\prog_{\alpha}(x)$ is non-increasing in 
		$\alpha$ (in particular, $\progf{1}{4}(x) \ge \progf{1}{2}(x)$). 
		Therefore, by 
		\pref{lem:deterministic-construction}.\ref{item:zero-chain},
		 $\brk{\pestunscaled{q}(x,z)}_i = \grad_i \Funscaled(x) =0$ 
		for all $i>\progf{1}{4}(x)+1$, all $x\in\R^T$ and all $z\in\{0,1\}$.
		In addition, since $z\sim \mathrm{Bernoulli}(\pchain)$, we have $\Pr 
		\prn*{ \exists x~\mid~ \prog\prn{\pestunscaled{1}(x, z), \dots, 		
		\pestunscaled{p}(x, z)} = \progf{1}{4}(x) + 1} \leq \pchain$, 
		establishing that the oracle is a probability-$\pchain$ zero-chain. \\
		\\
		To bound the variance of the derivative estimators, we observe that 
		$\pestunscaled{q}(x,z) - \grad^{q} \Funscaled(x)$ has at most one 
		nonzero $(q-1)$-subtensor in the coordinate $i_x = \progf{1}{4}(x)+1$. 
		Therefore,
		\begin{equation*}
		\E \norm{\pestunscaled{q}(x,z) - \grad^q 
		\Funscaled(x)}^2
		= \nrm*{\grad^q_{i_x} \Funscaled(x)}^2 
		\E\prn*{\frac{z}{\pchain}-1}^2
		= \nrm*{\grad^q_{i_x} \Funscaled(x)}^2 
		\frac{1-\pchain}{\pchain}
		\le \frac{(1-\pchain)\lip{{q-1}}^2}{\pchain},
		\end{equation*}
		where the final inequality is due to 		
		\pref{lem:deterministic-construction}.\ref{item:grad}, 
		establishing the variance bound in~\pref{eq:pzc-var-mss}.
	\end{proof}

	\begin{proof}[\pfref{thm:eps_so_zero_respecting}]
We now prove the \pref{thm:eps_so_zero_respecting} by scaling 
the construction $\Funscaled$ appropriately. Let ${\Delta}_0$ and $\lip{2}$ be 
	the numerical 
	constants in \pref{lem:deterministic-construction}. Let the accuracy 
	parameter $\epsilon$, initial suboptimality $\Delta$, derivative order 
	$p\in\N$, smoothness parameters $\Lip{1},\dots,\Lip{p}$, and variance 
	parameters $\sigma_1,\dots,\sigma_p$ be fixed. We set 
	\begin{align*}
	\Fscaled(x)=\alpha 
	\Funscaled\left(\beta x\right),
	\end{align*}
	for some scalars $\alpha$ and $\beta$ to be determined. The 
	relevant properties of $\Fscaled$ scale as follows
	\begin{align}
	\Fscaled(0) - \inf_{x} \Fscaled(x) &= 
	\alpha \prn[\big]{\Funscaled\left(0\right) - 
		\inf_{x}\Funscaled\left(\alpha x\right)}
	\le	\alpha{\Delta}_0 T, \\
	\nrm*{\grad^{q+1}\Fscaled(x)} &=  
	\alpha\beta^{q+1}\nrm*{\grad^{q+1}\Funscaled\left(\beta 
	x\right)}\le 
	{\alpha\beta^{q+1}}\lip{q},	\\
	\norm{\grad \Fscaled(x)}
	&\ge \alpha \beta \nrm{\grad \Funscaled(x)}\ge 
	\alpha\beta,~ \forall x \text{ s.t., } \prog_{1}(x)<T.
	\end{align}
	The corresponding scaled derivative estimators 
	$\pestscaled{q}(x,z)= \alpha\beta^q\pestunscaled{q}(\beta 
	x,z)$ clearly form a probability-$\pchain$ zero-chain. Therefore, by
	\pref{lem:prob-zero-chain}, we have that for every zero respecting 
	algorithm $\alg$ interacting with $\stocOhigh_{\Fscaled}$, with 
	probability at least $1/2$, $\prog\prn*{x\ind{t}_{\alg[\oracle^p_F]}} < T$ 
	for all $t\le(T-1)/2\pchain$. Hence, since $\prog_{1}(x)\le \prog(x)$ for 
	any $x\in\R^T$, we have by \pref{lem:deterministic-construction},
	\begin{align}\label{eq:eps_gen_rate}
	\E \nrm{ \grad 
	\Fscaled\prn[\big]{x\ind{t}_{\alg[\oracle^p_F]}}}
	&= 
	\alpha \beta \E \nrm{ \grad \Funscaled\prn[\big]{\beta 
			x\ind{t}_{\alg[\oracle^p_F]}}}\ge \frac{\alpha\beta 
			}{2}, 
			\quad \forall 
	t\le(T-1)/2\pchain.
	\end{align}
	We bound the variance of the scaled derivative estimators as 
	\begin{align*}
	\E \norm{\pestscaled{q}(x,z) - \grad^q \Fscaled(x)}^2 &= 
	\alpha^2\beta^{2q}
	\E \left\|{\pestunscaled{q}\left(\beta{x},z\right) - 
	\grad^q
		\Funscaled\left({\beta x}\right)}\right\|^2 
	\le \frac{ \alpha^2\beta^{2q}  
	\lip{q-1}^2(1-\pchain)}{\pchain},
	\end{align*}
	where the last inequality follows by \pref{lem:pzc-pair}. Our goal 
	now is to meet the following set of constraints:
	\begin{itemize}
		\item $\Delta\text{-constraint}\!:\quad\alpha{\Delta}_0 T \le 
		\Delta$
		\item 
		$L_{q}\text{-constraint}\!:\quad{\alpha\beta^{q+1}}\lip{q}\le 
		L_{q},$ for $q\in[p]$
		\item 
		$\eps\text{-constraint}\!:\quad\frac{\alpha\beta}{2}\ge\eps$
		\item $\sigma_q\text{-constraint}\!:\quad\frac{ 
		\alpha^2\beta^{2q}  
		\lip{q-1}^2(1-\pchain)}{\pchain}\le
		\sigma_q^2,$ for $q\in[p]$
	\end{itemize}
	Generically, since there are more inequalities to satisfy than
        the number of degrees of freedom ($\alpha,\beta,T$ and $\rho$) in our construction, not all 
	inequalities can be activated (that is, met by equality) simultaneously. 
	Different compromises will yield  different rates. \\
	\\
	First, to have a tight dependence in terms of $\eps$, we activate the 
	$\eps$-constraint by setting $\alpha = 2\epsilon/\beta$.	
	Next, we activate the $\sigma_1$-constraint, by setting $\rho=\min\{ 
	(\alpha\beta\lip{0}/\sigma_1)^2,1\}= \min \{ (2\eps \lip{0}/\sigma_1)^2 
	,1\}$. The bound on the variance of the q$th$-order derivative now reads 
	\begin{align*}
	\frac{ \alpha^2\beta^{2q}  
	\lip{q-1}^2(1-\pchain)}{\pchain}\le 
	\frac{\sigma_1^2 \alpha^2\beta^{2q}  
		\lip{q-1}^2}{(\alpha\beta\lip{0})^2}=
	\frac{\lip{q-1}^2 \beta^{2(q-1)}  
		\sigma_1^2}{\lip{0}^2}, \quad q=2,\dots,p.
	\end{align*}
	Since $\beta$ is the only degree of freedom which can be tuned to meet 
	though (not necessarily activate) the $\sigma_q$-constraint for 
	$q=2,\dots,p$ and the $\Lip{q}$-constraints for $q=1,\dots,p$, we are 
	forced to set
	\begin{align}
	\beta = 
	\min_{\substack{q=2,\dots,p\\q'=1,\dots, p}}  \min\crl*{ 
		\prn*{\frac{\ell_0\sigma_q}{\ell_{q-1}\sigma_1}}^{\frac{1}{q-1}},
		\prn*{\frac{L_{q'}}{2\eps \lip{q'}}}^{1/q'}}.
	\end{align}
	Lastly, we activate the $\Delta$-constraint by setting
	\begin{align*}
	T = \floor*{\frac{\Delta}{\alpha{\Delta}_0}}
	= \floor*{\frac{\Delta\beta }{2{\Delta}_0\epsilon} }.
	\end{align*}
	Assuming $(2\eps \lip{0}/\sigma_1)^2 \le 1$ and $T\ge 3$, 
	we have by \pref{eq:eps_gen_rate} that the number of 
	oracle queries required to obtain an $\eps$-stationary 
	point for $\barFscaled$ is bounded from below by 
	\begin{align}\label{eq:eps_bound_der}
	\frac{T-1}{2\pchain}\nonumber
	&=
	\frac{1}{2\pchain}\prn*{\floor*{\frac{\Delta\beta 
	}{2{\Delta}_0\epsilon}
		}-1}\nonumber\\ 
	&\stackrel{(\star) }{\ge} \frac{1}{2\pchain}\cdot \frac{\Delta\beta
	}{4{\Delta}_0\epsilon}\nonumber \\ 
		&\ge
	\frac{\sigma_1^2}{2(2\lip{0}\epsilon)^2}\cdot
	\frac{\Delta}{4{\Delta}_0\epsilon}	\cdot 
	\min_{\substack{q=2,\dots,p\\q'=1,\dots, p}}  \min\crl*{ 
		\prn*{\frac{\ell_0\sigma_q}{\ell_{q-1}\sigma_1}}^{\frac{1}{q-1}},
		\prn*{\frac{L_{q'}}{2\eps \lip{q'}}}^{1/q'}}\nonumber\\
	&\ge
	\frac{\Delta\sigma_1^2}{2^5{\Delta}_0\lip{0}^2\epsilon^3}\cdot 
	\min_{\substack{q=2,\dots,p\\q'=1,\dots, p}}  \min\crl*{ 
		\prn*{\frac{\ell_0\sigma_q}{\ell_{q-1}\sigma_1}}^{\frac{1}{q-1}},
		\prn*{\frac{L_{q'}}{2\eps \lip{q'}}}^{1/q'}},
	\end{align}
	where $(\star)$ uses $\floor{\xi}-1\geq{}\xi/2$ whenever 
	$\xi\geq{}3$, implying the desired bound. Lastly, we note that one can
	obtain tight lower complexity bounds for deterministic oracles by setting 
	$\rho=1$. Following the same chain of inequalities as in 
	\pref{eq:eps_bound_der}, in this case we get a lower oracle-complexity bound 
	of
	\begin{align}
	\frac{\Delta}{8{\Delta}_0\epsilon}	 
	\min_{q=1,\dots, p}
	\prn*{\frac{L_{q}}{2\eps \lip{q}}}^{1/q}.
	\end{align}

        \end{proof}
	
        	\subsubsection{Bounding the operator norm of $\nabla^{p}_i\Funscaled$}\label{sec:proof_ell_infty_bound}
          In this subsection we complete the proof of
          \pref{lem:deterministic-construction} by proving Part
          \pref{item:grad}.
	Our proof follows along the lines of the proof of Lemma 3 of \cite{carmon2019lower_i}. Let $x\in\bbR^{T}$ and 	
	$i_1,\dots,i_p\in[T]$, and note that by the chain-like structure of 
	$\Funscaled$, $\partial_{i_1}\cdots \partial_{i_p} 
	\Funscaled(x)$ is non-zero if and only if $|i_j-i_k|\le 1$ 
	for any $j,k\in[p]$. A straightforward calculation yields
	\begin{align} \label{ineq:high_der_pointwise_bound}
	\abs{\partial_{i_1}\cdots \partial_{i_p} 
		\Funscaled(x)} &\le 
	\max_{i\in[T]}\max_{\delta\in\{0,1\}^{p-1}\cup\{0,-1\}^{p-1}}
	\abs{\partial_{i+ \delta_1}\cdots \partial_{i 
	+\delta_{p-1}}\partial_{i} 	
		\Funscaled(x)}\\
	&\le 
	\max_{k\in[p]}\crl*{2\sup_{\xi\in\R}\abs*{\Psi^{k}(\xi)}~\sup_{\xi'\in\R}\abs*{\Phi^{p-k}(\xi')}}
	\le \exp(2.5p\log p+4p +9) \le 
	\frac{\ell_{p-1}}{2^{p+1}},\nonumber
	\end{align} 
	where the penultimate inequality is due to Lemma 1 of \cite{carmon2019lower_i}. Therefore, for a fixed $i\in[T]$, 
	we have
	\begin{align*}
	\nrm{ \grad^p_i \Funscaled(x)}_{\op}
	&\stackrel{(a)}{=} \sup_{\nrm{v}=1} \abs{\tri{ 
			\grad^p_i 
			\Funscaled(x),v}}\\ 
	&= \sup_{\nrm{v}=1} \abs*{\sum_{i_1,\dots,i_{p-1}\in 
	[T]} 	
		\partial_{i_1}\cdots \partial_{i_{p-1}}	
		\partial_{i} 	
		\Funscaled(x)v_{i_1}\cdots v_{i_{p-1}}} \\
	&\stackrel{(b)}{\leq} 
	\sum_{\delta\in\{0,1\}^{p-1}\cup\{0,-1\}^{p-1}} 	
	\abs{\partial_{i+ \delta_1}\cdots \partial_{i 
			+\delta_{p-1}}\partial_{i} 	\Funscaled(x)}\\
	&\stackrel{(c)}{\le} 
	(2^{p}-1)\frac{\ell_{p-1}}{2^{p+1}}\le 
	\ell_{p-1},
	\end{align*}	
	where $(a)$ follows from the definition of the operator norm, $(b)$ follows by the chain-like 
	structure of $\Funscaled$, and $(c)$ follows from 
	\pref{ineq:high_der_pointwise_bound},  concluding the proof.

 \subsection{Proof of Theorem \ref*{thm:gamma_so_zero_respecting}}
\label{app:gamma_lower}
In this section we prove \pref{thm:gamma_so_zero_respecting} following
the schema outlined in \pref{sec:epsilon_gamma_lower}.
We start by collecting all the relevant properties of 
$\Psi$ and $\Phibar$ from the construction in \pref{eq:hard_function_gamma_LB}.
\begin{lemma}\label{lem:second_order_LB_psiphi_properties}
	The functions $\Psi$ and $\Phibar$ satisfy the following properties: 
	\label{lem:gamma_second_order_LB_psiphi_properties}
	\begin{enumerate}

		\item For all $x \leq 1/2$ and for all $k \in \N\cup \{0\}$, 
		$\Psi^{(k)}(x) = 
		0$. 
		\item \label{item:psi_bound} The function $\Psi$ is non-negative and its first- and 
		second-order derivatives are bounded by
		\[0 \leq \Psi \leq e, \qquad 0 \leq \Psi' \leq \sqrt{54/e}, \qquad 
		-40 
		\leq \Psi'' \leq 40.\] 
		\item \label{item:ups_bound} The function $\Phibar$ and its first- and second-order 
		derivatives are bounded by
		\[ -8 \leq \Phibar \leq 0, \qquad -6 \leq \Phibar'  \leq 6, \qquad  
		-8 
		\leq \Phibar'' 
		\leq 4. \]
		\item Both $\Psi$ and $\Phibar$ are infinitely differentiable, and 
		for 
		all $k 
		\in \bbN$, we have
		\[ \sup_x \abs*{\Psi^{(k)}(x)} \leq \exp \prn*{\frac{5k}{2} 
			\log(4k)} \quad 
		\text{and} \quad \sup_x \abs*{\Phibar^{(k)}(x)} \leq  
		\frac{8}{\sqrt{e}}\cdot \exp  	\prn*{\frac{3(k + 
				1)}{2} \log\prn*{\frac{3(k+1)}{2}}}. 
		\] \label{item:moments_bounds}
	\end{enumerate}
\end{lemma}
\begin{proof}
	Parts 1-4 are immediate. Part 5 follows from Lemma 1 of \cite{carmon2019lower_i} and by noting that 
	\begin{align*}
	\sup_{x}\abs*{\Phibar^{(k)}(x)} = 
	\frac{8}{\sqrt{e}}\sup_{x}\abs*{\Phi^{(k+1)}(x)}
	\le \frac{8}{\sqrt{e}}\cdot 
	\exp  \prn*{\frac{3(k + 1)}{2} \log\prn*{\frac{3(k+1)}{2}}}.
	\end{align*}		
\end{proof} 
Using these basic properties of $\Psi$ and $\Phibar$, we establish the 
following properties of the construction $\barF$ (analogous to \pref{lem:deterministic-construction}).
\begin{lemma}	\label{lem:gamma_rescaling_second_order_LB}	
	The function $\barF$ satisfies the following properties: 
	\begin{enumerate}
		\item $\barF(0) - \inf_x(\barF(x)) \leq \overline{\Delta}_0T$, 
		with $\overline{\Delta}_0 = 40$.
		\item For $p\ge1$, the $p$th order derivatives 
		of  $\barF$ are  $\lipBar{p}$-Lipschitz continuous, where 
		$\lipBar{p}\le 
		e^{cp \log p +c'p}$ for a numerical 
		constant $c,c'<\infty$. 
		\item For all $x\in\bbR^{T}$, and $i\in[T]$, we have 
	$\nrm*{\grad^p_i \barF(x)}_{\op}\le \lipBar{p}$.\label{item:bargrad} 
		\item  For all $x\in\R^T$ and $q\in[p]$,
		$\prog\prn*{\grad^{(q)}\barF(x)}\leq{}\progf{1}{2}(x)+1$.\label{item:progBar}
		\item For all $x\in\R^T$, if $\progf{9}{10}(x)<T-1$ 
		then
		$\lambdamin \prn*{\nabla^2 \barF(x)} \leq -0.5$, and  
		$\lambdamin 
		\prn*{\nabla^2 \barF(x)} \leq 700$ otherwise.
		\label{item:second_order_LB_large_eigenvalue}	
	\end{enumerate} 
\end{lemma}
\begin{proof} We prove the individual parts of the lemma one by one:
	\begin{enumerate}
		\item Since $\Psi(0)=\Phibar(0)=0$, we have
		\begin{align*}
		\barF(0) = 
		\Psi(1)\Phibar(0) + \sum_{i=2}^T ~ \brk*{ 
			\Psi(0) 
			\Phibar(0) + \Psi(0) \Phibar(0)) } = -\Psi(1)\Phibar(0)=0.
		\end{align*}
		On the other hand, 
		\begin{align*}
		\barF(x) &= \Psi(1)\Phibar(x_1) + \sum_{i=2}^T ~ \brk*{ 
			\Psi(-x_{i-1}) 
			\Phibar(-x_i) + \Psi(x_{i-1}) \Phibar(x_i)) }  \\
		&\geq - 8 e T \qquad (\text{by 
			\pref{lem:second_order_LB_psiphi_properties}.\ref{item:psi_bound}.
			and 
			\pref{lem:second_order_LB_psiphi_properties}.\ref{item:ups_bound}})\\
		&\geq -40 T.
		\end{align*}
		
		\item The proof follows along the same lines of Lemma
                  3 of \cite{carmon2019lower_i} together with the derivative 
		bounds stated in 		
		\pref{lem:second_order_LB_psiphi_properties}.\ref{item:moments_bounds}.
		
		\item The claim follows using the same calculation as
                  in \pref{sec:proof_ell_infty_bound},
		with the derivative bounds replaced by those in
		\pref{lem:second_order_LB_psiphi_properties}.\ref{item:moments_bounds},
		mutatis mutandis.

		\item The claim follows 
	Observation 
		3 in \cite{carmon2019lower_i}, mutatis mutandis.
				
		\item We have
		\begin{align}
		\frac{\del \barFunscaled }{\del x_j} = 
		-\Psi(-x_{j-1})\Phibar'(-x_j) + \Psi(x_{j-1})\Phibar'(x_j)
		-\Psi'(-x_{j})\Phibar(-x_{j+1}) + 
		\Psi'(x_{j})\Phibar(x_{j+1}).
		\end{align}
		Therefore, for any $x\in\R^d$, 
		$\nabla^2 
		\barF(x)$ is a tridiagonal 
		matrix 
		specified as follows.
		\begin{align*}
		\nabla^2 \barF(x)_{i,j} &= 
		\begin{cases}
		\Psi(-x_{i-1}) \Phibar''(-x_i) + \Psi(x_{i-1}) \Phibar''(x_i) \\
		\quad \quad 
		+\Psi''(-x_i)\Phibar(-x_{i+1})+\Psi''(x_i)\Phibar(x_{i+1})
		& 
		\text{if~~} i = j, \\ 
		\Psi'(-x_j) \Phibar'(-x_i) + \Psi'(x_j) \Phibar'(x_i) & \text{if } 
		j = 
		i -1, \\ 
		\Psi'(-x_i) \Phibar'(-x_j) + \Psi'(x_i) \Phibar'(x_j) & \text{if } 
		j 
		= i + 1 ,\\ 
		0 & \text{otherwise}.
		\end{cases}
		\end{align*}
		The following facts can be 
		verified by a straightforward 
		calculation:
		\begin{enumerate}[label=(\roman*)]
			\item $\Psi(x)\ge 0.5$ for all $x\ge9/10$.
			\item $\Psi''(x)\ge 0$ for all  $\abs{x}<9/10$.
			\item $\Phibar''(x)\le -1$ for all $\abs{x}<9/10$.
		\end{enumerate}
		Next, assuming 	$k\defeq \progf{9}{10}(x)+1< T$, we have, by 
		definition, that 
		$\abs{x_{k+1}},\abs{x_k}<\frac{9}{10}\le 
		\abs{x_{k-1}}$, implying, 
				\begin{align*}
		\lambda_{\min} (\nabla^2 \barF(x)) &= \min_{y \in 
			\bbR^n} 
		\frac{y^T \nabla^2 \barF(x) y}{y^T y} & \text{(Rayleigh 
			quotient)}\\  
		&\leq \frac{e_k^T \nabla^2 \barF(x) e_k}{e_k^T e_k} \\
		&= \nabla^2 \barF(x)_{k,k} \\
		&=
		\Psi(-x_{k-1}) \Phibar''(-x_k) + \Psi(x_{k-1}) \Phibar''(x_k) \\
		&~~~~+\Psi''(-x_k)\Phibar(-x_{k+1})+\Psi''(x_k)\Phibar(x_{k+1}) \\
		&\le  \Psi(-x_{k-1}) \Phibar''(-x_k) + \Psi(x_{k-1}) 
		\Phibar''(x_k) & 
		(\text{(ii) and } \Phibar\le0)\\
		&=  \Psi(\abs{x_{k-1}}) \Phibar''(\operatorname{sign}\{x_{k-1}\} 
		x_k) & (\Psi(x)= 0, ~	\forall 
		x<0)\\
		&\leq -1\cdot 0.5=-0.5. & (\text{(i) and (iii)})
		\end{align*}
		
		Otherwise, if nothing is assumed on $x$, then the same 
		chain of inequalities, using 
		$k=2$, can be used to bound 
		the minimal value 
		of $\grad^2 
		\barFunscaled(x)$.
		\begin{align*}
		\lambda_{\min} (\nabla^2 \barF(x)) &= \min_{y \in 
			\bbR^n} 
		\frac{y^T \nabla^2 \barF(x) y}{y^T y} & \text{(Rayleigh 
			quotient)}\\  
		&\leq \frac{e_k^T \nabla^2 \barF(x) e_k}{e_k^T e_k} \\
		&= \nabla^2 \barF(x)_{k,k} \\
		&=
		\Psi(-x_{k-1}) \Phibar''(-x_k) + \Psi(x_{k-1}) \Phibar''(x_k) \\
		&+\Psi''(-x_k)\Phibar(-x_{k+1})+\Psi''(x_k)\Phibar(x_{k+1}) \\
		&\le  2\prn*{4e + 320}\le 700,
		\end{align*}
		thus giving the desired bound.
	\end{enumerate}
\end{proof}
We employ similar derivative estimators to the proof of 
\pref{thm:eps_so_zero_respecting}, only this 
time we provide a noiseless estimate for the gradient. Formally, we 
set
\begin{align} 
\brk*{\pestunscaledBar{q}(x,z)}_i \defeq 
\begin{cases}
\grad_i \barF(x) & q = 1,\\
\prn*{1+\indicator{i > 
\prog_{\frac{1}{4}}(x)}\prn*{\frac{z}{p}-1}}\cdot
\grad_i^q \barF(x) & q \ge 2,
\end{cases}
\end{align}
where $z\sim  \mathrm{Bernoulli}(\pchain)$. The dynamics of 
zero-respecting methods can be now characterized in an analogous 
way to the proof of \pref{thm:eps_so_zero_respecting}. The only 
 difference is that here, since $\Lambda'(0)=\Psi'(0)=0$, it 
follows that $\prog_0\prn{\grad \barF\prn{x}} = \prog_0(x)$.
Therefore, the collection of estimators defined above is a 
$\pchain$-probability zero-chain---with respect to $\prog_0$ (rather than 
$\prog_{\frac{1}{4}}$ as in 
\pref{def:prob_p_zero_chain})\footnote{Using $\prog_{0}$, rather 
	than $\prog_\frac{1}{4}$, carries one major disadvantage: our bounds for 
	finding $\gamma$-weakly convex points cannot be directly extended to 
	arbitrary randomized algorithm using the technique presented in 
	Section 
	3.4 of \cite{carmon2019lower_i} as is (at least, not without
        the degrading the dependence on problem parameters). We 
	defer such an extension to future 
	work.}---in 
	which 
the variance of the gradient estimator is $0$; a key property that shall be 
used soon.  Following the proof of \pref{lem:prob-zero-chain}, mutatis 
mutandis, gives us the same bound on the number of non-zero entries acquired 
over time. That is, we have that with 
probability at least $1-\delta$,
\begin{equation}\label{eq:gamma_dyn_sto}
\prog\prn*{	x\ind{t}_{\alg[\oracle^{p}_F]}
} < T, \quad\text{for all } t \leq{} 
\frac{T-\log(1/\delta)}{2\pchain},
\end{equation}
where we employ the same notation as in \pref{lem:prob-zero-chain}.
The proof now proceeds along the same lines of the proof of 
\pref{thm:eps_so_zero_respecting}. The estimators have variance
bounded as
\begin{equation}\label{eq:gamma_vara_bounds}
\E \,\norm{\pestunscaledBar{q}(x,z) - \grad^q \barF(x)}^2 
\le \begin{cases}
0 &q=1,\\
\frac{\lipBar{q-1}^2(1-\pchain)}{\pchain},\quad \text{ for all } 
x\in \R^T & q \ge 2,
\end{cases}
\end{equation}
which can established the same fashion as \pref{lem:pzc-pair} by invoking 
\pref{lem:gamma_rescaling_second_order_LB}.\ref{item:bargrad}
and 
\pref{lem:gamma_rescaling_second_order_LB}.\ref{item:progBar}.\\
\begin{proof}[\pfref{thm:gamma_so_zero_respecting}]
We now complete the proof of \pref{thm:gamma_so_zero_respecting} for $p\ge2$ by
scaling $\barF$ appropriately. Let ${\Delta}_0$ and 
$\lipBar{p}$ be the numerical 
constants in \pref{lem:gamma_rescaling_second_order_LB}. 
Let the accuracy parameter $\gamma$, initial suboptimality $\Delta$, 
derivative order $p\in\N$, smoothness 
parameter $\Lip{1},\dots,\Lip{p}$, and 
variance parameter $\sigma_1,\sigma_2,\dots,\sigma_p$
 be fixed. We let 
\begin{align*}
\barFscaled(x)\ldef\alpha 
\barFunscaled\left(\beta x\right),
\end{align*}
for scalars $\alpha$ and $\beta$ to be determined. The relevant 
properties of $\barFscaled$ are as follows:
\begin{align}
\barFscaled(0) - \inf_{x} \barFscaled(x) &= 
\alpha \prn[\big]{\barFunscaled\left(0\right) - 
	\inf_{x}\barFunscaled\left(\alpha x\right)}
\le	\alpha{\barzDelta} T, \\
\nrm*{\grad^{q+1}\barFscaled(x)} &=  
\alpha\beta^{q+1}\nrm*{\grad^{q+1}\barFunscaled\left(\beta x\right)}\le 
{\alpha\beta^{q+1}}\lipBar{q},	\\
\lambdamin\prn{ \grad^2 \barFscaled\prn[\big]{x}}
&=\alpha\beta^2 \lambdamin\prn{ \grad^2 \barFunscaled\prn[\big]{x}}\le 
-\frac{\alpha\beta^2}{2},\quad\quad\forall x \text{ s.t., } \prog_{9/10}(x)<T.
\end{align}
The corresponding scaled derivative estimators 
$\pestscaledBar{q}(x,z)= \alpha\beta^q\pestunscaledBar{q}(\beta x,z)$
clearly form a probability-$\pchain$ zero-chain, thus by
\pref{eq:gamma_dyn_sto}, we have that for every zero respecting algorithm 
$\alg$ 	interacting with $\stocOhigh_{\barFscaled}$, with probability at 
least 
$1-1/(4\cdot700)$, $\prog\prn*{x\ind{t}_{\alg[\oracle^{p}_F]}} < T-1$ for 
all 
$t\le(T-2)/2\pchain$. 
Therefore, since $\prog_{9/10}(x)\le 
\prog(x)$ for any $x\in\R^T$, we have by 
\pref{lem:gamma_rescaling_second_order_LB}.\ref{item:second_order_LB_large_eigenvalue},
\begin{align}\label{eq:gamma_exp_bound}
\E\brk*{ \lambdamin\prn{ \grad^2 
	\barFscaled\prn[\big]{x\ind{t}_{\alg[\oracle^{p}_{F}]}}}}
&= 
\alpha\beta^2 \lambdamin\prn{ \grad^2 
	\barFunscaled\prn[\big]{\beta x\ind{t}_{\alg[\oracle^{p}_F]}}}\nonumber\\
&\le 	\alpha\beta^2\prn*{ -0.5\cdot \prn{1-\frac{1}{4\cdot 700}} + 
	700\cdot\frac{1}{4\cdot 700}  }\nonumber\\
&\le \frac{-\alpha\beta^2}{5},
\end{align}
for any $t\le(T-2)/2\pchain$. The variance of the scaled derivative 
estimators can be bounded as
\begin{align*}
\E \norm{\pestscaledBar{q}(x,z) - \grad^q \barFscaled(x)}^2 &= 
\alpha^2\beta^{2q}
\E \left\|{\pestunscaledBar{q}\left(\beta{x},z\right) - \grad^q
	\barFunscaled\left({\beta x}\right)}\right\|^2 
\le \frac{ \alpha^2\beta^{2q}  \lipBar{q-1}^2(1-\pchain)}{\pchain},
\end{align*}
where the last inequality is by \pref{eq:gamma_vara_bounds}. Our goal now is to meet the 
following set of 
constraints:
\begin{itemize}
	\item $\Delta\text{-constraint}\!:\quad 
	\alpha{\barzDelta} 
	T \le \Delta\quad $.
	\item $L_{q}\text{-constraint}\!:\quad 
	{\alpha\beta^{q+1}}\lipBar{q}\le \Lip{q}~$  for $q=1,\dots,p$.
	\item $\gamma\text{-constraint}\!:\quad 
	-\frac{\alpha\beta^2}{5}\le-\gamma$.
	\item $\sigma_q\text{-constraint}\!:\quad \frac{ 
	\alpha^2\beta^{2q}  \lipBar{q-1}^2(1-\pchain)}{\pchain}\le
	\sigma_q^2~$ for $q=1,\dots,p$.
\end{itemize}
As there are more inequalities to satisfy than 
the four degrees of freedom
($\alpha,\beta,T$ and $\rho$) in our 
construction, generically, not all inequalities can be activated (that 
is, met by equality) simultaneously. 
Different compromises may yield different bounds. First, to have a tight 
dependence in terms of $\gamma$,  we activate the 
$\gamma$-constraint by setting $\alpha =5\gamma/\beta^2$.	Next, we 
activate the $\sigma_2$-constraint, by 
setting $\rho=\min\{ (\alpha\beta^2\lipBar{1}/\sigma_2)^2,1\}=\min\{ 
(5\lipBar{1}\gamma/\sigma_2)^2,1\}$. The bound on the variance of the 
$q$th derivative for $q=3,\dots,p,$ now reads
\begin{align*}
\frac{ \alpha^2\beta^{2q}  \lipBar{q-1}^2(1-\pchain)}{\pchain}
\le
\frac{\sigma_2^2 \alpha^2\beta^{2q}
	\lipBar{q-1}^2}{(\alpha\beta^2\lipBar{1})^2}=
\frac{\lipBar{q-1}^2 \beta^{2(q-2)}  
	\sigma_2^2}{\lipBar{1}^2}, \quad q=3,\dots,p.
\end{align*}
Since $\beta$ is the only degree of freedom which can be tuned to meet 
(though not necessarily activate) the 
$\sigma_q$-constraints for 
$q=3,\dots,p$, and the $L_{q'}$-constraint for 
$q'=2,\dots,p$, we are forced to have
\begin{align}
\beta = 
\min_{\substack{q=3,\dots,p\\q'=2,\dots,p}} \min\crl*{ 
	\prn*{\frac{\lipBar{1}\sigma_q}{\lipBar{q-1}\sigma_2}}^{\frac{1}{q-2}},
	\prn*{\frac{\Lip{q'}}{5\lipBar{q'}\gamma}}^{\frac{1}{q'-1}}}.
\end{align}
Note that, by definition, the $\sigma_1$-constraint always holds 
(as the variance of the gradient estimator is zero, see 
\pref{eq:gamma_vara_bounds}). To satisfy the $L_1$-constraint, 
i.e., 
$\alpha\beta^2 \lipBar{{1}}\le\Lip{1}$, we must have
\begin{align}\label{eq:gamma_L1_const}
\gamma\le\Lip{1}/5\lipBar{{1}}.
\end{align}
This constraint holds w.l.o.g. 
as $L_1$ also bounds the 
absolute value of the Hessian eigenvalues (in other words, any 
point $x$ is trivially $O(L_1)$-weakly convex). Lastly, we activate 
the $\Delta$-constraint, by setting
\begin{align*}
T = \floor*{\frac{\Delta}{\alpha{\barzDelta}}}
= \floor*{\frac{\Delta\beta^2 }{5{\barzDelta}\gamma} }.
\end{align*}
Assuming $(5\lipBar{1}\gamma/ \sigma_2)^2\le 1$ (i.e., $\gamma = 
O(\sigma_2)$) and $T\ge 3$, we 
have by \pref{eq:gamma_exp_bound} that the number of oracle queries 
required to obtain a point $x$ such that $	\lambdamin\prn{ \grad^2 
\barFscaled\prn[\big]{x}}\le \lambda$, is bounded from below by 
\begin{align}\label{eq:gamma_final}
\frac{T-2}{2\pchain}\nonumber
&=
\frac{1}{2\pchain}\prn*{
	\floor*{\frac{\Delta\beta^2 }{5{\barzDelta}\gamma} }
	-2}\nonumber\\ 
&\stackrel{(\star)}{\ge}
\frac{1}{2\pchain}
\frac{\Delta\beta^2 }{5^2{\barzDelta}\gamma}\nonumber\\ 
&\ge 
\frac{\sigma_2^2}{(5\lipBar{1}\gamma)^2}\cdot\frac{\Delta\beta^2 
}{5^2{\barzDelta}\gamma}\nonumber\\
&=	\frac{\sigma_2^2}{(5\lipBar{1}\gamma)^2}\cdot\frac{\Delta 
}{5^2{\barzDelta}\gamma}
\min_{\substack{q=3,\dots,p\\q'=2,\dots,p}}\min\crl*{  
	\prn*{\frac{\lipBar1\sigma_q}{\lipBar{q-1}\sigma_2}}^{\frac{2}{q-2}},
	\prn*{\frac{\Lip{q'}}{5\lipBar{q'}\gamma}}^{\frac{2}{q'-1}}}\nonumber\\
&=
\frac{1}{5^4\lipBar{1}^2{\barzDelta}}\cdot\frac{\Delta \sigma_2^2
}{ \gamma^3 }
	\min_{\substack{q=3,\dots,p\\q'=2,\dots,p}} \min
	\crl*{ 	
	\prn*{\frac{\lipBar1\sigma_q}{\lipBar{q-1}\sigma_2}}^{\frac{2}{q-2}},
	\prn*{\frac{\Lip{q'}}{5\lipBar{q'}\gamma}}^{\frac{2}{q'-1}}},
\end{align}
where $(\star)$ uses that $\floor{\xi}-2\geq{}\xi/5$ whenever 
$\xi\geq{}3$, implying the desired result (note that this bound 
does not depend on $\Lip{1}$ and $\sigma_1$.).

If $\sigma_1=\cdots=\sigma_p = 0$, we obtain the following lower 
complexity bound for noiseless oracles (where $\pchain$ is 
effectively set to one), assuming $\gamma = O(L_1)$ (this holds 
without loss of generality, as we discuss above). As before, we set
$\alpha =5\gamma/\beta^2$. The $L_1$-constraint is satisfied under 
the same condition stated in \pref{eq:gamma_L1_const}. Thus, 
letting 
\begin{align*}
\beta = \min_{q=2,\dots,p} 
	\crl*{\prn*{\frac{\Lip{q}}{5\lipBar{q}\gamma}}^{\frac{1}{q-1}}},
\end{align*} 
it follows that our construction is $L_q$-Lipschitz for any $q=1,\dots,p$. 
Following the same chain of inequalities as in \pref{eq:gamma_final} yields 
an oracle complexity lower bound of
\begin{align}
\frac{\Delta\beta^2 }{5^3{\barzDelta}\gamma}\nonumber
=
\frac{\Delta }{5^3{\barzDelta}\gamma}
\min_{q=2,\dots,p} 
\crl*{\prn*{\frac{\Lip{q}}{5\lipBar{q}\gamma}}^{\frac{2}{q-1}}}.
\end{align}
Note that this bound does not depend on $L_1$.

\end{proof}

 \end{document}